\documentclass{article}

\usepackage{mathrsfs}
\usepackage{amsthm}
\usepackage{amsmath}
\usepackage{amsfonts}
\usepackage{amssymb}
\usepackage{commath}
\usepackage{graphicx,color}
\usepackage{url}
\usepackage{latexsym,enumerate}
\usepackage{booktabs}
\usepackage{framed}
\usepackage{subfigure}
\usepackage{multirow} 
\usepackage{color}
\usepackage{colortbl}

\usepackage{algorithm}
\usepackage{algorithmic}

\definecolor{gray1}{rgb}{0.8,0.8,0.8}
\definecolor{gray2}{rgb}{0.95,0.95,0.95}

\newtheorem{Th}{Theorem}[subsection]
\newtheorem{Lem}[Th]{Lemma}

\newcommand{\RE}{\,{\rm Re}}

\newcommand{\argmin}{\mathop{\rm argmin}}

\newcommand{\Shrink}{\mathop{\rm Shrink}}

\newcommand{\CST}{\mathop{\rm CST}}

\newcommand{\cC}{{\mathcal C}}

\newcommand{\cF}{{\mathcal F}}
\newcommand{\cL}{{\mathcal L}}

\newcommand{\cH}{{\mathcal H}}
\newcommand{\cT}{{\mathcal T}}

\newcommand{\Bh}{{\mathbf h}}
\newcommand{\Bb}{{\mathbf b}}
\newcommand{\Bf}{{\mathbf f}}
\newcommand{\Bu}{{\mathbf u}}

\newcommand{\Br}{{\mathbf r}}
\newcommand{\Bv}{{\mathbf v}}

\newcommand{\Bp}{{\mathbf p}}
\newcommand{\Bw}{{\mathbf w}}

\newcommand{\Bg}{{\mathbf g}}

\newcommand{\Bq}{{\mathbf q}}

\newcommand{\Beps}{{\boldsymbol \epsilon}}

\newcommand{\Bome}{{\boldsymbol \omega}}

\newcommand{\Bk}{{\boldsymbol k}}
\newcommand{\Bx}{{\boldsymbol x}}

\newcommand{\BDm}{{\mathbf{D_1}}}            
\newcommand{\BDmT}{{\mathbf{D}_{\mathbf 1}^{\text T}}}
\newcommand{\BDn}{{\mathbf{D_2}}}
\newcommand{\BDnT}{{\mathbf{D}_{\mathbf 2}^{\text T}}}

\newcommand{\Bz}{{\boldsymbol z}}

\newtheorem{thm}{Theorem}[section]
\newtheorem{prop}[thm]{Proposition}

\begin{document}

\title{Multiphase Segmentation For Simultaneously Homogeneous and Textural Images}

\author{Duy Hoang Thai \thanks{Statistical and Applied Mathematical Science Institute (SAMSI), USA.} \thanks{Corresponding Author:  Email: dhthai@samsi.info}
\ and 
Lucas Mentch\footnotemark[1]
}
\maketitle

 \begin{abstract}
 Segmentation remains an important problem in image processing.  For homogeneous (piecewise smooth) images, a number of important models have been developed and refined over the past several decades.  However, these models often fail when applied to the substantially larger class of natural images that simultaneously contain regions of both texture and homogeneity.  This work introduces a bi-level constrained minimization model for simultaneous multiphase segmentation of images containing both homogeneous and textural regions.  
 We develop novel norms defined in different functional Banach spaces for the segmentation which results in a non-convex minimization.
 Finally, we develop a generalized notion of segmentation delving into approximation theory and demonstrating that a more refined decomposition of these images results in multiple meaningful components.  Both theoretical results and demonstrations on natural images are provided.
 \end{abstract}

\hspace{4.5mm} \emph{Keywords:  Image Decomposition, Variational Calculus, Image Denoising,} \\
\hspace*{25.25mm} \emph{Feature Extraction, Image Segmentation}

\section{Introduction}
\label{sec:Introduction}

Image segmentation remains at the forefront of issues in computer vision and image processing and an abundance of approaches have been developed to solve a wide range of problems; see, for example, \cite{DigitalImageProcessing2002,Szeliski2011}.  The goal of image segmentation is to decompose the image domain into a montage of meaningful components.  This has lead to breakthroughs in a number of research areas such as medical imaging \cite{IntroductiontoBiomedicalImaging2003}, astronomical imaging \cite{AstronomicalImageandDataAnalysis2006}, and biometric recognition \cite{IntroductiontoBiometrics2011,ThaiHuckemannGottschlich2016,ThaiGottschlich2016G3PD}.

Segmentation methods can be broadly characterized by the class of target images for which they are intended, either homogeneous (piecewise smooth) or textural.  Many such methods have been suggested, including approaches based on the intensity of pixels \cite{Otsu1979, SahooWilkinsYeager1997, AlbuquerqueEsquefMelloAlbuquerque2004}
and others based on curve evolution\cite{ChanVese2001,BressonEsedogluVandergheynstThiranOsher2007,ChanEsedogluNikolova2012,LieLysakerTai2006}.  For homogeneous images in particular, the classical approach to segmentation is based on active contours \cite{KassWitkinTerzopoulos1988}.  Given an image $f$ on a bounded domain $\Omega \subset \mathbb{R}^2$, contours are driven to object boundaries by internal and external forces in the functional
\begin{equation} \label{eq:KassWitkinTerzopoulos}
 \inf_{C(s)} \left\{ \alpha \int_{0}^1 \abs{C'(s)}^2 ds + \beta \int_{0}^1 \abs{C''(s)} ds - \lambda \int_{0}^1 \abs{\nabla f(C(s))}^2 ds \right\}
\end{equation}
with a curve $C(s) : [0, 1] \rightarrow \mathbb R^2$
and positive parameters $\alpha \,, \beta$ and $\lambda$.

Under the classical model $f(\Bx) = u(\Bx) + \epsilon(\Bx)$ with $\Bx \in \Omega$, Mumford and Shah \cite{MumfordShah1989} proposed a solution by minimizing the energy functional
\begin{equation} \label{eq:PiecewiseSmoothMS}
 \inf_{u, C} \left\{ \int_\Omega \big( f(\Bx) - u(\Bx) \big)^2 d\Bx ~+~ \nu \int_{\Omega \backslash C} \abs{\nabla u(\Bx)}^2 d\Bx + \mu \abs{C}  \right\}.
\end{equation}
However, this piecewise smooth Mumford-Shah model is NP-hard due to the Hausdorff 1-dimensional measure
$\cH^1(\mathbb R^2)$. 
A simplified version for image segmentation when $f$ is assumed to be piecewise constant can be written as 
\begin{equation} \label{eq:PiecewiseConstantMS:indicator}
 \inf_{[c_n]_{n=1}^N, [\Omega_n]_{n=1}^N} \left\{ 
 \sum_{n=1}^N \int_{\Omega} \big( f(\Bx) - c_n \big)^2 \mathbf 1_{\Omega_n}(\Bx) d\Bx ~+~ 
 \frac{\mu}{2} \sum_{n=1}^N \int_\Omega \abs{\nabla \mathbf 1_{\Omega_n}(\Bx)} d \Bx \right\} \,.
\end{equation}
which closely resembles the Potts model \cite{Potts1952} developed decades earlier.
Rudin et al.\ \cite{RudinOsherFatemi1992} proposed an alternative, more computationally efficient version of the model in (\ref{eq:PiecewiseSmoothMS}) that preserves sharp edges in the restored image.  These advantages led to numerous extensions including examination in different functional spaces, 
\cite{AujolGilboaChanOsher2006, AujolAubertFeraudChambolle2005, AujolGilboa2006, BuadesLeMorelVese2010, VeseOsher2003, AubertVese1997}, versions involving higher-order derivatives \cite{ChanMarquinaMulet2000, LysakerLundervoldTai, RahmanTaiOsher2007, HahnWuTai2012}, 
mean curvature \cite{ZhuChan2012}, Euler's elastica \cite{TaiHahnChung2011, ZhuTaiChan2013}, total variation of the first and second order derivatives \cite{PapafitsorosSchoenlieb2014}, and higher-order PDEs for diffusion solved by directional operator splitting schemes \cite{CalatroniDueringSchoenlieb2014}.  Various techniques have been proposed for solving the convex optimization including Chambolle's projection \cite{Chambolle2004}, the splitting Bregman method \cite{GoldsteinOsher2009}, and
iterative shrinkage/thresholding (IST) algorithms \cite{DaubechiesDefriseMol2004, BeckTeboulle2009, DiasFigueiredo2007}. 
In 2010, Wu et al.\ \cite{WuTai2010} proved the equivalence between the augmented Lagrangian method 
(ALM), dual methods, and the splitting Bregman method. 

Letting $p_n(\Bx)$ denote the indicator function $\mathbf 1_{\Omega_n}(\Bx)$ with $\Bx \in \Omega$, (\ref{eq:PiecewiseConstantMS:indicator}) can be rewritten as the non-convex constrained minimization
\begin{equation} \label{eq:PiecewiseConstantMS:indicator:nonconvex}		
 \min_{\vec{c}, \vec{p}} \left\{
 \frac{\mu}{2} \sum_{n=1}^N \norm{\nabla p_n}_{L_1} + \sum_{n=1}^N \Big\langle (f - c_n)^2 \,, p_n \Big\rangle_{L_2} 
 \text{s.t.} 
 \sum_{n=1}^N p_n(\Bx) = 1 \,, p_n(\Bx) \in \{0, 1\}
 \right\}
\end{equation}
where $\vec{c} = [c_n]_{n=1}^N \,, \vec{p} = [p_n]_{n=1}^N$.  
Note that when $N=2$, this becomes the celebrated Chan and Vese model \cite{ChanVese2001}.
Brown et al.\ \cite{BrownChanBresson2010, BrownChanBresson2012} provide a convex relaxation of (\ref{eq:PiecewiseConstantMS:indicator:nonconvex})
by relaxing the binary set to $p_n(\Bx) \in [0, 1]$ and 
Bae et al. \cite{BaeYuanTai2010} solve this relaxed version via a smoothed primal-dual method; 
see \cite{GuWangTai2012, BaeLellmannTai2005, GuWangXiongChengHuangZhou2013, GuXiongWangChengHuangZhou2014} for details.
The advantage of the multiphase segmentation is illustrated in Figure \ref{fig:MultiphaseOver2phase} using the smoothed primal-dual method in \cite{BaeYuanTai2010} to recover an object under a spectrum of illumination.

\begin{figure}
\begin{center}
 \includegraphics[width=1\textwidth]{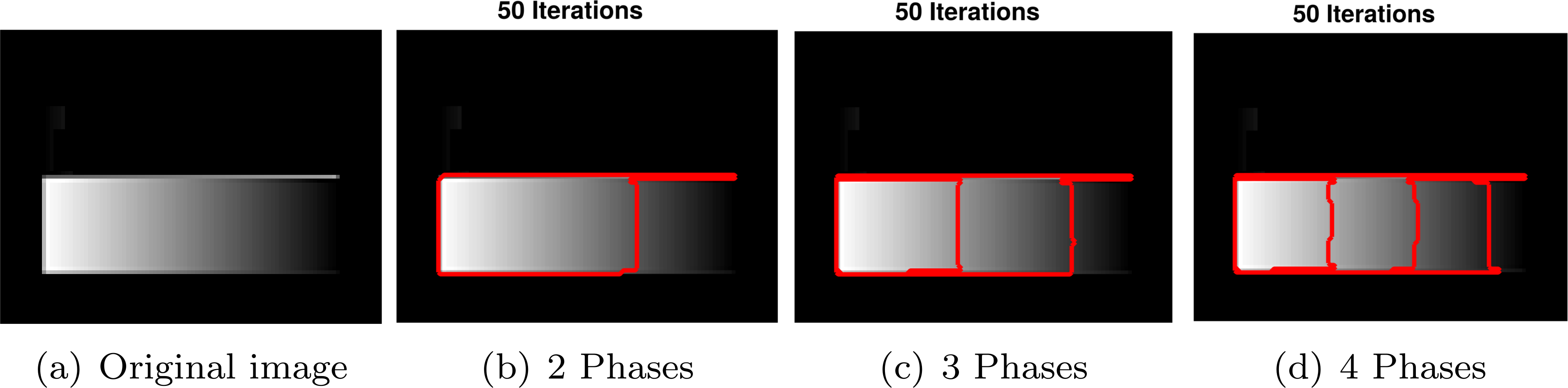}
 
 \caption{The original image (a) is segmented by \cite{BaeYuanTai2010} with two phase (b), three phase (c) and four phase  
               segmentation (with the same notation for parameters $n = 4 \,, s = 0.001 \,, \delta = 0.1$).
               Using multiple phases allows us to recover portions of the image with different illumination (gradient change).
             }
 \label{fig:MultiphaseOver2phase}
\end{center}
\end{figure}


Texture segmentation and analysis remains a challenging problem due to its oscillatory nature.  
Proposed methods include those based on texture descriptors \cite{SagivSochenZeevi2006, HouhouThiranBresson2009}, histogram metrics \cite{NiBressonChanEsedoglu2009}, 
and finding other meaningful features in an observed image for classification \cite{Unser1995}.  Among the most popular approaches is the vector-valued Chan-Vese model for texture segmentation with a Gabor filter \cite{ChanSandbergVese2000} whose convex relaxed version is defined in \cite{BrownChanBresson2010,BrownChanBresson2009}.  
This can be seen as a generalized version of the two-phase piecewise constant Mumford and Shah model for a vector valued image $\vec{f} = \big[ f_m \big]_{m=1}^M$ and constant vector 
$\vec{c}_1 = \big[ c_{1m} \big]_{m=1}^M \,, \vec{c}_2 = \big[ c_{2m} \big]_{m=1}^M$. 
The resulting minimization becomes convex by relaxing the binary constraint to $p(\Bx) \in [0 \,, 1]$ \cite{BrownChanBresson2010}.

Though these techniques have seen much success in their respective domains, many natural images like fingerprints and stem cell imaging contain both homogeneous and textural regions and it is important to define a technique to capture the entirety of this information.  The Mumford-Shah model fails in this larger class of images due to the absence of measurement for texture; see Figure \ref{fig:Barbara:MSsegmentation} for such an example where textural regions appear in each phase. 
\begin{figure}
\begin{center}  
\includegraphics[height=0.179\textheight]{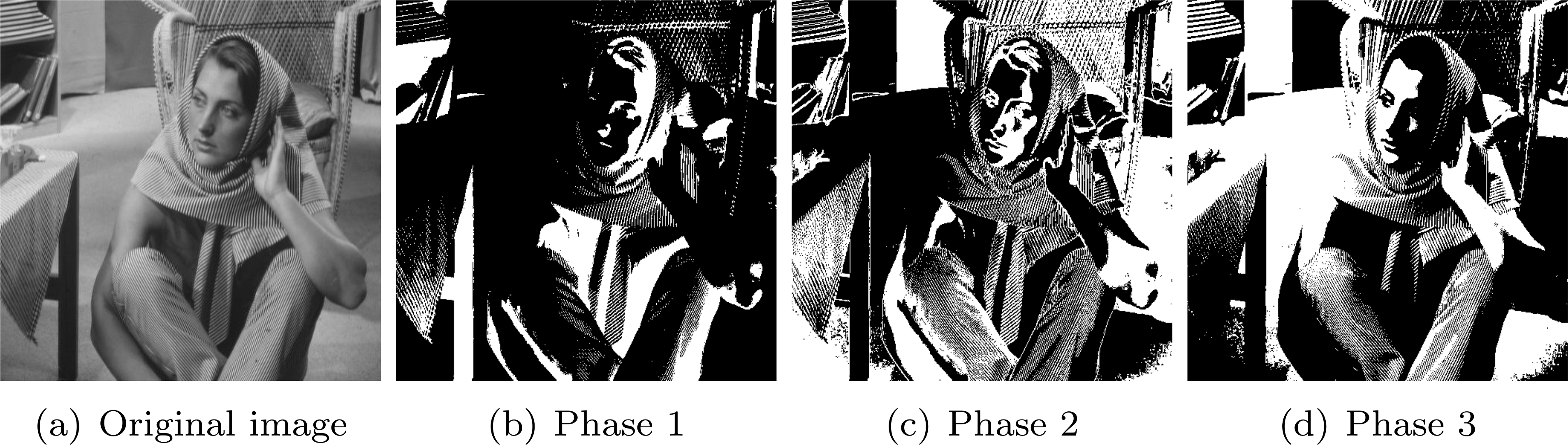}
 
 \caption{The original image (a) containing both texture and homogeneity is segmented into three phases by 
               the Mumford and Shah (MS) model with 
               $\sigma = 0, N = 3, s = 0.005, \xi = 0.001, \lambda = 10^{-3}, \tau = 0.1, \text{Iteration} = 20$.
               Note that both homogeneous and textural information appear in all three phases (b)-(d).
             }
\label{fig:Barbara:MSsegmentation}
\end{center}
\end{figure}


In this work, we provide a method for multiphase segmentation of images that simultaneously contain regions of both homogeneity and texture.  An attempt at this kind of segmentation was provided in \cite{LiuTaiHuangHuan2011} but importantly, our work here can be viewed as a decomposition of the original image which approximates the image in functional space instead of using harmonic analysis.  This approach to the inverse problem allows us to obtain a piecewise constant component as well as sparse directional information; Figure \ref{fig:fingerprint:approximation} shows a preview of results obtained using the bilevel SHT method outlined in Section \ref{sec:BiLevelMinimizationScheme}.

Following \cite{ThaiGottschlich2016G3PD, VeseOsher2003, AujolChambolle2005, ThaiGottschlich2016DG3PD, Gilles2012}, 
we adopt the idea of the discrete directional $\text{G}_S$-norm to measure texture in several directions  
and the dual of a generalized Besov space in the curvelet domain $\cC$ 
\cite{ThaiGottschlich2016DG3PD, CandesDonoho2004, CandesDemanetDonohoYing2006, StarckDonohoCandes2003, MaPlonka2010}
to measure the residual. This approach is particularly useful for many natural images such as fingerprints in which texture 
appears in many directions and can easily be adapted for the shearlet, contourlet, steerable wavelet or 2D empirical transforms 
\cite{ShearletsBook,DoVetterli2005,UnserVandeville2010, GillesTranOsher2014}.  
Because of the curvelet transform, the residual can be either independent or correlated and need not follow a Gaussian distribution.  Since our minimization involves the $\text{G}_S$-norm, we propose two alternative methods based on the two primary approaches to handling the $\text{G}_S$-norm:  a multiphase SHT method based on the approach of Aujol and Chambolle \cite{AujolChambolle2005} and a bilevel SHT method based on the approach of Vese and Osher \cite{VeseOsher2003}.

\begin{figure}
\begin{center}  
 \includegraphics[height=0.345\textheight]{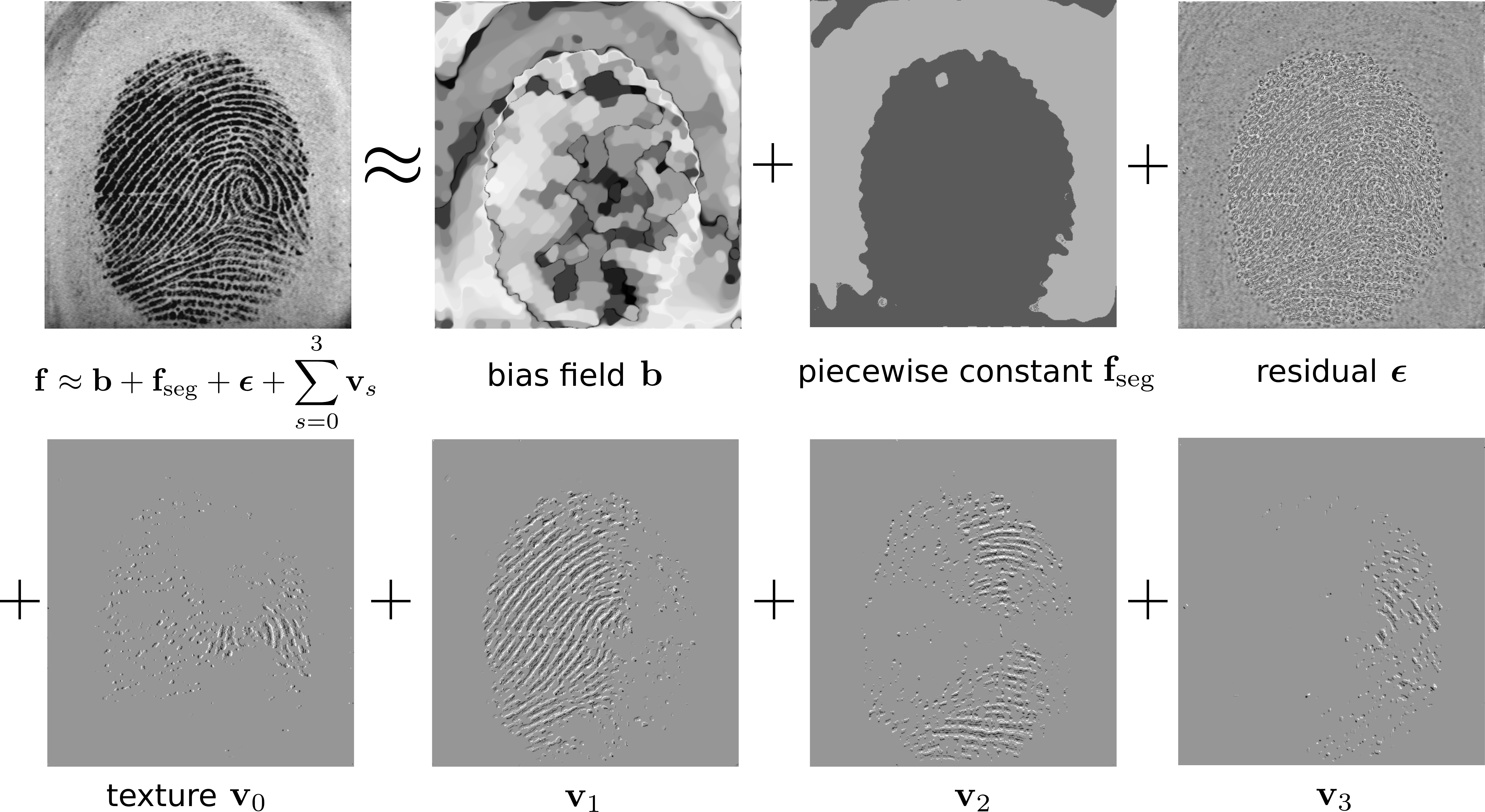}
 
 \caption{Image representation of a fingerprint decomposed according to the bilevel SHT method 
          described in Section \ref{sec:BiLevelMinimizationScheme}.  The directional texture is shown in the bottom row.}
\label{fig:fingerprint:approximation}
\end{center}
\end{figure}

The remainder of this paper is organized as follows.  In Section \ref{sec:prelim}, we define some preliminary notation and investigate a simple model where only two-phase segmentation is considered and the homogeneous portion is assumed to be piecewise constant.  In Section \ref{sec:multiSHT} we generalize this set-up both to multiphase segmentation and also to the case where homogeneous regions are considered piecewise-smooth.  In Section \ref{sec:BiLevelMinimizationScheme}
 we introduce a bilevel minimization scheme to more efficiently solve the minimization induced by the multiphase piecewise smooth context.  Finally, we apply the methodology to a number of representative images and compare to related approaches in Section \ref{sec:comparisons}.  For readability, mathematical details and proofs are provided in the Appendix.

\section{Preliminaries and Simplified Models}
\label{sec:prelim}
We begin by establishing some background notation and definitions.  
Let $X$ be the Euclidean space with dimension given by the size of the lattice 
$\Omega = \big\{ \Bk = [k_1 \,, k_2] \in [ 0 \,, d_1-1 ] \times [ 0 \,, d_2-1 ] \subset \mathbb Z^2 \big\}$.
On the bounded domain $\Omega$, we denote
the coordinates of the Fourier transform as $\Bome = [\omega_1 \,, \omega_2] \in [-\pi \,, \pi]^2$
and the coordinates of the $Z$ transform (the discrete version of the Fourier transform)
as $\Bz = \big[z_1 \,, z_2\big] = \big[e^{j\omega_1} \,, e^{j\omega_2}\big]$.
Denote the discrete Fourier transform pair as 
\begin{align*}
 f[\Bk] ~\stackrel{\cF}{\longleftrightarrow}~ 
 \cF \big\{ f[\Bk] \big\}(e^{j\Bome}) = F(e^{j\Bome}) = \sum_{\Bk \in \Omega} f[\Bk] e^{-j \langle \Bk \,, \Bome \rangle_{\ell_2}}\,.
\end{align*}
Given the discrete function $\Bf = \big[ f[\Bk] \big]_{\Bk \in \Omega} \in X$, a vector $\vec{\Bg} = [\Bg_l]_{l=0}^{L-1} \in X^L$
and the direction $l = 0, \ldots, L-1$, we make a few preliminary definitions.  
The {\bfseries directional forward/backward difference operators} are given, in matrix notation, by
 \begin{align*}
  &\partial_l^+ \Bf = \sin\left(\frac{\pi l}{L}\right) \BDm \Bf + \cos\left(\frac{\pi l}{L}\right) \Bf \BDnT
  ~~~~~~~ \stackrel{\cF}{\longleftrightarrow}~ 
  \Big[ \sin\left(\frac{\pi l}{L}\right) (z_1 - 1) + \cos\left(\frac{\pi l}{L}\right) (z_2 - 1) \Big] F(\Bz)
  \\
  &\partial_l^- \Bf = - \Big[ \sin\left(\frac{\pi l}{L}\right) \BDmT \Bf + \cos\left(\frac{\pi l}{L}\right) \Bf \BDn \Big]
  ~~ \stackrel{\cF}{\longleftrightarrow}~   
  -\Big[ \sin\left(\frac{\pi l}{L}\right) (z_1^{-1}-1) + \cos\left(\frac{\pi l}{L}\right) (z_2^{-1} - 1) \Big] F(\Bz) 
 \end{align*}
with
$\partial_l^\pm \Bf = \Big[ \partial_l^\pm f[\Bk] \Big]_{\Bk \in \Omega}$
and a matrix 
\begin{equation*}
 \BDm = \begin{pmatrix}
         -1            & 1             &0            &\hdots           &0              \\
          0            &-1             &1            &\hdots           &0              \\
          \vdots       &\vdots         &\vdots       &\ddots           &\vdots         \\
          0            &0              &0            &\hdots           &1              \\
          1            &0              &0            &\hdots            &-1            \\             
        \end{pmatrix}
  \in \mathbb R^{d_1 \times d_1}  \,,
\end{equation*} 
and similarly $\BDn \in \mathbb R^{d_2 \times d_2}$.
  
Given the adjoint operators of the difference operators as 
 $\big( \partial^\pm_x \big)^* = - \partial_x^\mp \,,~ \big( \partial^\pm_y \big)^* = - \partial_y^\mp$,
 the adjoint operators of their directional version are
  \begin{equation*} 
  \big( \partial^\pm_l \big)^* = - \Big[ \cos\left(\frac{\pi l}{L}\right) \partial_x^\mp + \sin\left(\frac{\pi l}{L}\right) \partial_y^\mp \Big] = - \partial_l^\mp \,,
 \end{equation*}
and we can define the {\bfseries discrete directional gradient} and {\bf divergence} as 
 $$
 \nabla_L^\pm \Bf = \Big[ \partial_l^\pm \Bf \Big]_{l=0}^{L-1}
\hspace{5mm}   \text{  and  }  \hspace{5mm}
 \text{div}_L^\pm \vec{\Bg} = \sum_{l=0}^{L-1} \partial_l^\pm \Bg_l    
 $$ 
respectively.  Note that the adjoint operator of $\nabla_L^\pm$ is $\big( \nabla_L^\pm \big)^* = - \text{div}_L^\mp$; that is
 \begin{equation*}
  \Big\langle \nabla_L^\pm \Bf \,, \vec{\Bg} \Big\rangle_{\ell_2} = 
  - \Big\langle \Bf \,, \text{div}_L^\mp \vec{\Bg} \Big\rangle_{\ell_2} \,. 
 \end{equation*}
Given a vector of matrices $\vec{\Bp} = \big[ \Bp_l \big]_{l=0}^{L-1} \in X^L$, 
the derivative of the divergence operator $\text{div}^-_L \vec{\Bp}$ w.r.t. $\Bp_l$ for $l = 0, \ldots, L-1$ is given by
\begin{align*} 
 \frac{ \partial }{ \partial \Bp_l } \Big\{ \text{div}^-_L \vec{\Bp} \Big\} 
 = - \frac{ \partial }{ \partial \Bp_l } \Big[ \sin\left(\frac{\pi l}{L}\right) \BDmT \Bp_l + \cos\left(\frac{\pi l}{L}\right) \Bp_l \BDn \Big]
 = \Big[ -\partial_l^+ \delta[\Bk] \Big]_{\Bk \in \Omega} \,,
\end{align*}
where $\delta(\cdot)$ denotes the Dirac delta function.  Thus, the derivative of the directional divergence w.r.t. $\vec{\Bp}$ is 
\[
 \frac{ \partial }{ \partial \vec{\Bp} } \Big\{ \text{div}^-_L \vec{\Bp} \Big\} 
 = \Big[ - \big[ \partial_l^+ \delta[\Bk] \big]_{l=0}^{L-1} \Big]_{\Bk \in \Omega}
 = \Big[ - \nabla^+_L \delta[\Bk] \Big]_{\Bk \in \Omega} \,.
\]

\noindent Finally, the {\bfseries discrete directional $\text{G}_S$-norm} \cite{ThaiGottschlich2016DG3PD} is given by 
 \begin{align*}
  \norm{\Bv}_{\text{G}_S} &= \text{inf} \Big\{ \norm{ \abs{\vec{\Bg}} }_{\ell_\infty} \,,
  \Bv = \text{div}^-_S \vec{\Bg} \,,~ \vec{\Bg} = \big[ \Bg_s \big]_{s=0}^{S-1} \in X^S
  \Big\} \,
 \end{align*}
 
\noindent and the indicator function on a convex set for noise in the curvelet domain $\cC$ 
\cite{ThaiGottschlich2016DG3PD, CandesDonoho2004, CandesDemanetDonohoYing2006} by 
 \begin{align*}
  A(\nu) = \Big\{ \Beps \in X ~:~ \norm{\cC \{ \Beps \}}_{\ell_\infty} \leq \nu \Big\}
  ~~\text{and}~~
  \mathscr G^*(\frac{\Beps}{\nu})
  = \begin{cases}
     0 \,, & \Beps \in A(\nu)   \\
     +\infty \,, & \text{else}  
    \end{cases}   \,.
 \end{align*}

\noindent
For a more thorough background on the mathematical preliminaries (including the point-wise operators as
$\pm \,, \backslash \,, \cdot^\times \,, \max \,, \text{Shrink} \,, \CST$, etc.), we refer the reader to \cite{ThaiHuckemannGottschlich2016, ThaiGottschlich2016G3PD, ThaiGottschlich2016DG3PD, Thai2015PhD}.
We point out that in the remainder of this work, we use boldface to present a matrix, e.g. $\Bu \in X$,
vector with boldface to denote a vector of matrix, e.g. $\vec{\Bp} = \left[ \Bp_n \right]_{n=1}^N \in X^N$,
and vector (without boldface) to denote constant vector, e.g. $\vec{c} = \left[ c_n \right]_{n=1}^N \in \mathbb R^N$.

\subsection{Two-phase piecewise constant and texture segmentation}
\label{sec:2phasePiecewiseConstTextureSeg}

We begin by considering the simple discrete model consisting of a two phase piecewise constant image 
(indicated by indicator function $\Bp$ and mean values $(c_1 \,, c_2) \in \mathbb R_+$) 
and texture $\Bv$ corrupted by i.i.d. (or weakly correlated) noise $\Beps$ as
\begin{align*}
 \Bf = c_1 \Bp + c_2 (1-\Bp) + \Bv + \Beps \,.
\end{align*}
As in \cite{ThaiGottschlich2016DG3PD}, we propose the model for this segmentation as
\begin{align} \label{eq:twophaseDG3PDtexture:minimization:1} 
 \min_{(\Bp, \Beps, c_1, c_2) \in X^{2} \times \mathbb R^2} & \bigg\{ 
 \norm{\nabla^+_L \Bp}_{\ell_1} + \mu_1 \norm{\Bv}_{\text{G}_S} + \mu_2 \norm{\Bv}_{\ell_1}
 + \mathscr G^*\left(\frac{\Beps}{\nu}\right)   \notag
 \\
 &\text{s.t.}~~ \Bf = c_1 \Bp + c_2 (1-\Bp) + \Bv + \Beps \,,
 p[\Bk] \in \{0, 1\}, \forall \Bk \in \Omega
 \bigg\} \,.  
\end{align}
In order to solve a non-convex minimization (\ref{eq:twophaseDG3PDtexture:minimization:1}),
we relax a binary set $p[\Bk] \in \{ 0 \,, 1\}$ to $[0 \,, 1]$ and then apply ALM and the alternating directional method of multipliers (ADMM);
see proposition \ref{prop:2phasePiecewiseConst} and Algorithms 5
in Appendix for detailed calculations.
Figure \ref{fig:2phasePiecewiseConst} illustrates the results of this model.
The histogram of the indicator function $\Bp$ in Figure \ref{fig:2phasePiecewiseConst} (d) shows that 
$\Bp$ almost converges to $\{ 0 \,, 1 \}$ after the $100^{th}$ iteration.
One can use a technique in \cite{BaeYuanTai2010} to make $\Bp$ be exactly a binary setting as a solution of 
a minimization (\ref{eq:twophaseDG3PDtexture:minimization:1}).
\setcounter{subfigure}{0}
\begin{figure}[!ht]
\begin{center}
\includegraphics[width=1\textwidth]{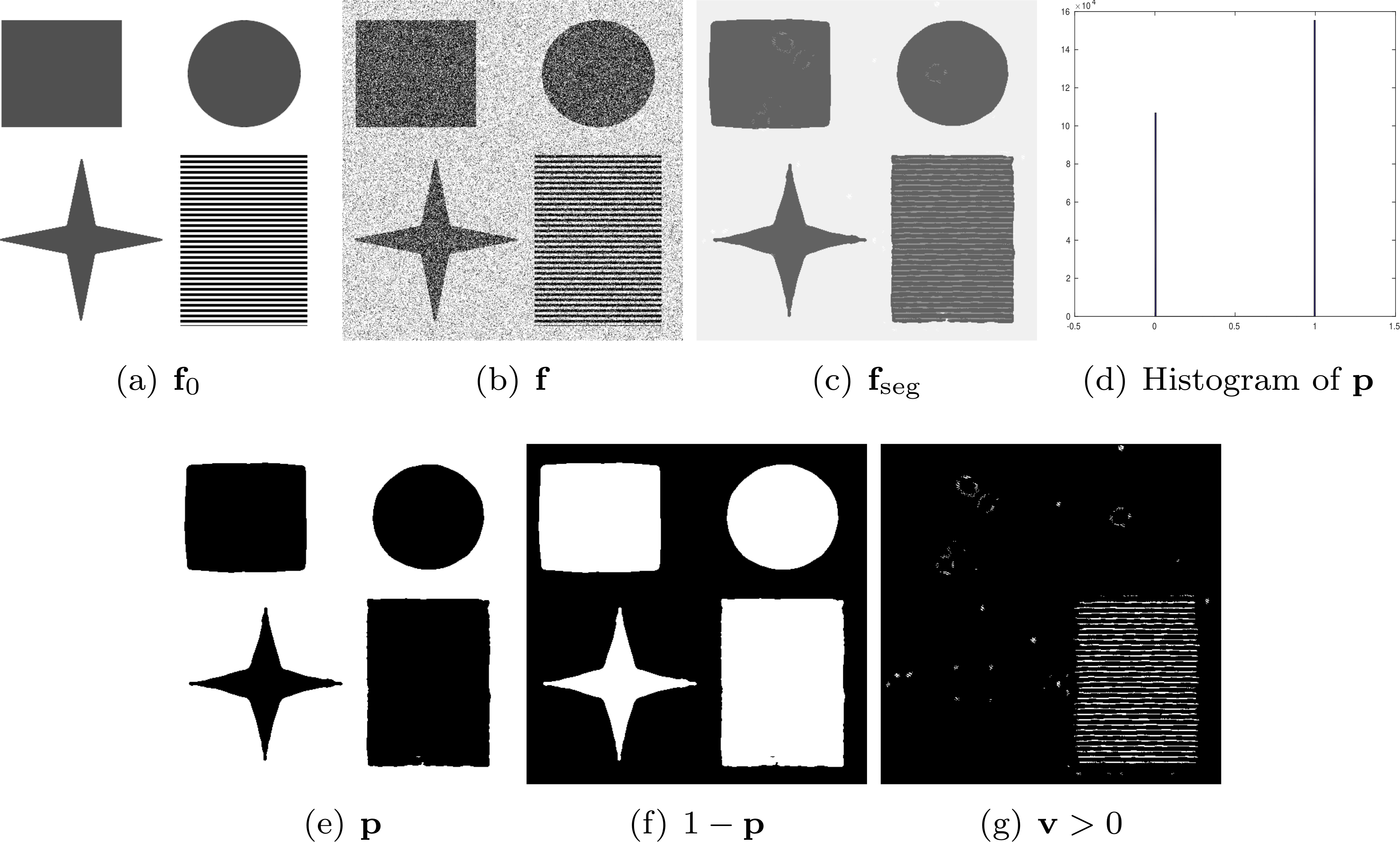}
 
\caption{
The original image $\Bf_0$ is shown in Subfigure (a).  Subfigure (b) shows the same image $\Bf$ with additional i.i.d. noise added from a Gaussian distribution with mean 0 and standard deviation $\sigma = 100$.
The segmented version $\Bf_\text{seg} = c_1 \Bp + c_2 (1 - \Bp) + c_3(v>0)$ of (b) shown in Subfigure (c)
is obtained by solving the minimization in (\ref{eq:twophaseDG3PDtexture:minimization:1}), see Algorithm 5 in the Appendix with
the parameters: 
$L = 150 \,, S = 9 \,, c_\delta = 0.1, \theta = 0.9 \,, c_{\mu_1} = c_{\mu_2} = 0.03 \,, 
 \beta_4 = 0.03 \,, \beta_3 = \frac{\theta}{1-\theta}\beta_4 \,, \beta_1 = \beta_4 \,, \beta_2 = 1.3\beta_3 \,,
 \#\text{iteration} = 100$.
The binarized texture $\Bv$ in (g) shows its sparsity by a minimization of (\ref{eq:twophaseDG3PDtexture:minimization:1})
with a percentage of non-zero coefficients in texture $\Bv$ as $\frac{\#\{\Bv \neq 0 \}}{m n}100\%  = 7.75\%$.
Figure (e) and (f) are the indicator function and its complement, respectively.
The mean values are $c_1 = 240.42 \,, c_2 = 98.02$ and we choose $c_3 = 50$.
\label{fig:2phasePiecewiseConst}}
\end{center}
\end{figure}

\section{Multiphase Segmentation SHT}
\label{sec:multiSHT}
The above models consider only two-phase segmentation in images where the homogeneous region can be considered
piecewise constant.  We now generalize this to allow for multiphase segmentation and also allow for piecewise-smooth homogeneity.  

\subsection{Multiphase piecewise smooth and texture segmentation} \label{sec:CombinedModel}
As before, we assume that a natural image $\Bf$ contains both texture $\Bv$ and homogeneous regions $\Bu$ as well as noise $\Beps$ so that $ \Bf = \Bu + \Bv + \Beps$.  However, we now further assume that $\Bu$ consists of both a multiphase $(N)$ piecewise constant (indexed by the indicator function 
$\vec{\Bp} = \big[ \Bp_n \big]_{n=1}^N$ and their mean values $\vec{c} = \big[ c_n \big]_{n=1}^N$) as well as a bias field $\Bb$

\[
\Bu = \Bb + \sum_{n=1}^N c_n \Bp_n \,,
\]

\noindent to account for the piecewise-smooth component of $\Bf$.
Following \cite{Chambolle2004}, we utilize the directional $\text{G}_S$-norm to measure the texture $\Bv$ and propose a new combined model for multiphase {\bf s}imultaneous {\bf h}omogeneous and {\bf t}exture image segmentation (the SHT model) as
\begin{align} \label{eq:CombinedModel:1}
 & \min_{\big( \vec{c}, \vec{\Bp}, \Bu, \Bv, \Beps, \Bb \big) \in \mathbb R^N \times X^{N+4}} \Bigg\{
 \norm{\nabla^+_L \Bu}_{\ell_1}
 + \mu_2 \norm{\Bv}_{\ell_1}
 + \mu_3 \sum_{n=1}^N \norm{\nabla^+_M \Bp_n}_{\ell_1}    
 + \frac{\mu_4}{2} \norm{\Bb}^2_{\ell_2}
 + \mathscr G^*(\frac{\Beps}{\nu})
 \notag
 \\& 
 \text{s.t. } \Bf = \Bu + \Bv + \Beps \,,
 \norm{\Bv}_{\text{G}_S} \leq \mu_1 \,, 
 \Bu = \Bb + \sum_{n=1}^N c_n \Bp_n \,, \notag \\
 &\hspace{25mm} \sum_{n=1}^N \Bp_n = 1 \,,
 p_n [\Bk] \in \{0, 1\} \,, n = 1, \ldots, N \,, \Bk \in \Omega  
 \Bigg\} \,.  
\end{align}
Note that in contrast with (\ref{eq:twophaseDG3PDtexture:minimization:1}), we no longer assume a piecewise-constant homogeneous region and thus must also take into account the bias term $\Bb$.  
Note further that the directional total variation norm (DTV-norm) $\norm{\nabla^+_L \cdot}_{\ell_1}$ and $\text{G}_S$-norm are a dual pair if $L = S$;
see Lemma \ref{lem:DTVDGnorm:dualpair} for details.  Finally, observe that as with the ROF model in  \cite{RudinOsherFatemi1992}, the process of smoothing the homogeneous areas while preserving 
the edge information is controlled by the DTV-norm for~ $\Bu$.

\subsubsection*{Solution to the Multiphase SHT Model}

In a similar fashion to \cite{BaeYuanTai2010}, the minimization in (\ref{eq:CombinedModel:1}) can be solved by
a smoothed primal-dual model for the $\vec{\Bp}$-problem
rather than by the Fourier approach used in the two-phase piecewise-constant model in  (\ref{eq:twophaseDG3PDtexture:minimization:1}). 
The remainder of this section provides a sketch of the proposed algorithm; for details and proofs, see Propositions \ref{prop:DirectionalTVL2}, \ref{prop:DirectionalGL1}, \ref{prop:combinedmodel:pproblem} in the Appendix.

\noindent Define the indicator function on a convex set for the $\text{G}_S$-norm as
\begin{align*}
 &\text{G}_S(\mu_1) = \Big\{ \Bv \in X \,, \vec{\Bg} \in X^S ~:~ \norm{\Bv}_{\text{G}_S} = \norm{\vec{\Bg}}_{\ell_\infty} \leq \mu_1 \Big\} \\
&\hspace{15mm} \text{ and } J^*_S \left(\frac{\Bv}{\mu_1}\right)
 = \begin{cases}
     0 \,, & \Bv \in \text{G}_S(\mu_1)   \\
     +\infty \,, & \text{otherwise}
    \end{cases}
 \,. 
\end{align*}
By applying ALM to the equality constraint $\Bf = \Bu + \Bv + \Beps$ and 
relaxing the binary setting $p_n[\Bk] \in \{0, 1\}$
to the convex set $p_n[\Bk] \in [0, 1]$, the nonconvex minimization in (\ref{eq:CombinedModel:1})
becomes convex as
\begin{align} \label{eq:CombinedModel:2}
 &\min_{(\vec{c}, \vec{\Bp}, \Bu, \Bv, \Beps) \in \mathbb R^N \times X^{N+3}} 
 \Big\{ \cL\big( \vec{c}, \vec{\Bp}, \Bu, \Bv, \Beps; \boldsymbol{\lambda} \big)
 \text{  s.t.  } 
 \sum_{n=1}^N \Bp_n = 1 \,,
 p_n [\Bk] > 0 \,, n = 1, \ldots, N \,, \Bk \in \Omega  
 \Big\}   
\end{align}
with
\begin{align*}
 \cL(\cdot; \cdot) &=  \norm{\nabla^+_L \Bu}_{\ell_1}
 + \mu_2 \norm{\Bv}_{\ell_1}
 + \mu_3 \sum_{n=1}^N \norm{\nabla^+_M \Bp_n}_{\ell_1}    
 + \frac{\mu_4}{2} \norm{\Bu - \sum_{n=1}^N c_n \Bp_n}^2_{\ell_2}
 + J_S^*(\frac{\Bv}{\mu_1})
 + \mathscr G^*(\frac{\Beps}{\nu})
 \\&
 + \frac{\beta}{2} \norm{ \Bf - \Bu - \Bv - \Beps + \frac{\boldsymbol{\lambda}}{\beta} }^2_{\ell_2} \,.
\end{align*}
Due to the multi-variable minimization, we apply ADMM to (\ref{eq:CombinedModel:2}) whose
minimizer is numerically computed through iteration $t$ with updated 
Lagrange multiplier
\begin{align} \label{eq:SHT:numericalLagrange}
 \big( \vec{c}^{(t)}, \vec{\Bp}^{(t)}, \Bu^{(t)}, \Bv^{(t)}, \Beps^{(t)} \big)
 = \argmin \cL\big( \vec{c}, \vec{\Bp}, \Bu, \Bv, \Beps \,;~ \boldsymbol{\lambda}^{(t-1)} \big) \,.
\end{align}
Given the initialization $\Bu^{(0)} = \Bf \,, \vec{\Bp}^{(0)} = \Bv^{(0)} = \Beps^{(0)} = \boldsymbol{\lambda}^{(0)} = \mathbf 0$
and $c_n = (n-1) \lfloor \frac{255}{N} \rfloor$ for $n = 1, \ldots, N$,
we solve the following five subproblems before updating the Lagrange multiplier. \vspace{2mm} \\

\noindent {\bfseries The $\Bu$-problem:} Fix $\vec{c}, \vec{\Bp}, \Bv, \Beps$ and solve
\begin{align} \label{eq:CombinedModel:uproblem:solution}
 &\min_{\Bu \in X} \left\{
 \norm{\nabla^+_L \Bu}_{\ell_1} 
 ~+~ \frac{\mu_4 + \beta}{2} \norm{\Bu - \Bh}^2_{\ell_2}
 \right\}
\end{align}
where
$\Bh = \frac{\mu_4}{\mu_4 + \beta} \sum_{n=1}^N c_n \Bp_n  +
       \frac{\beta}{\mu_4 + \beta} \Big[ \Bf - \Bv - \Beps + \frac{\boldsymbol{\lambda}}{\beta} \Big] \,.
$

\noindent From Proposition \ref{prop:DirectionalTVL2} and the numerical solver in (\ref{eq:SHT:numericalLagrange}),
a primal solution of the DTV-$\ell_2$ (\ref{eq:CombinedModel:uproblem:solution}) at iteration $t$ is given by
\begin{align*}
 \Bu^{(t)} &= \Bh - \frac{1}{\mu_4 + \beta} \text{div}^-_L \vec{\Br}^{(t)} 
\end{align*}
with dual variable 
\begin{align*}
 \vec{\Br}^{(t)}
 = 
 \frac{ \vec{\Br}^{(t-1)} + \tau \nabla^+_L \Big[ \text{div}^-_L \vec{\Br}^{(t-1)} - (\mu_4 + \beta) \Bh \Big] }
      { 1 + \tau \abs{ \nabla^+_L \Big[ \text{div}^-_L \vec{\Br}^{(t-1)} - (\mu_4 + \beta) \Bh \Big] } } \,. 
\end{align*}


\noindent {\bfseries The $\Bv$-problem:} Fix $\vec{c}, \vec{\Bp}, \Bu, \Beps$,
denote $\Bh_{\text{v}} = \Bf - \Bu - \Beps + \frac{\boldsymbol{\lambda}}{\beta}$ and solve
\begin{equation}  \label{eq:CombinedModel:vproblem:solution}
 \min_{\Bv \in X} \left\{
 J_S^*\left(\frac{\Bv}{\mu_1}\right)
 + \norm{\Bv}_{\ell_1}
 + \frac{1}{2} \frac{\beta}{\mu_2} \norm{ \Bv - \Bh_{\text{v}} }^2_{\ell_2}
 \right\}.
\end{equation}
To simplify the problem, we apply Proposition \ref{prop:DirectionalGL1} with a quadratic penalty $(\boldsymbol{\lambda_1} = 0)$.
The primal solution of the directional $\text{G}_S-\ell_1$ model (\ref{eq:CombinedModel:vproblem:solution}) 
at iteration $t$ is updated as
\begin{align*}
 v^{(t)} = \Shrink \left( \frac{ \frac{\beta}{\mu_2} }{\alpha + \frac{\beta}{\mu_2}} \Bh_{\text{v}} + \frac{\alpha \mu_1}{\alpha + \frac{\beta}{\mu_2} } \text{div}^-_S \vec{\Bg}^{(t)}
                     \,,~ \frac{1}{\alpha + \frac{\beta}{\mu_2}}
               \right) \,,
\end{align*}
with dual variable 
\begin{align*}
 \vec{\Bg}^{(t)}
 = \frac{ \vec{\Bg}^{(t-1)} + \tau \nabla_S^+ \Big[ \alpha \mu_1 \text{div}^-_S \vec{\Bg}^{(t-1)} - \alpha \Bv^{(t-1)} \Big] }
        { 1 + \tau \abs{\nabla^+_S \Big[ \alpha \mu_1 \text{div}^-_S \vec{\Bg}^{(t-1)} - \alpha \Bv^{(t-1)} \Big]} } .
\end{align*}


\noindent {\bfseries The $\Beps$-problem:} Fix $\vec{c}, \vec{\Bp}, \Bu, \Bv$ and solve
\begin{align}  \label{eq:CombinedModel:epsilonproblem:solution}
 \min_{\Beps \in X} \left\{
 \mathscr G^*\left(\frac{\Beps}{\nu}\right)
 + \frac{\beta}{2} \norm{ \Beps - \Big[\Bf - \Bu - \Bv + \frac{\boldsymbol{\lambda}}{\beta} \Big] }^2_{\ell_2}
 \right\}.
\end{align}
In a similar fashion to \cite{ThaiGottschlich2016DG3PD}, the solution of (\ref{eq:CombinedModel:epsilonproblem:solution}) is given by
\begin{equation*}
  \Beps^* = \Big[\Bf - \Bu - \Bv + \frac{\boldsymbol{\lambda}}{\beta} \Big] 
           - \CST \Big( \Big[\Bf - \Bu - \Bv + \frac{\boldsymbol{\lambda}}{\beta} \Big] \,, \nu \Big) \,.
\end{equation*}


\noindent {\bfseries The $\vec{c} = \big[ c_n \big]_{n=1}^N$-problem:} Fix $\vec{\Bp}, \Bu, \Bv, \Beps$ and solve
\begin{equation}  \label{eq:CombinedModel:cproblem:solution}
 \min_{\vec{c} \in \mathbb R^N} \left\{
 \norm{ \Bu - \sum_{n=1}^N c_n \Bp_n }^2_{\ell_2}
 \right\}.
\end{equation}
Due to its separability, the solution of (\ref{eq:CombinedModel:cproblem:solution}) is given by
\begin{equation*} \label{eq:caproblem:solution}
 c_n = \frac{\displaystyle \sum_{\Bk \in \Omega} u[\Bk] p_n[\Bk] }
             {\displaystyle \sum_{\Bk \in \Omega} p_n[\Bk] }  
 \,,~ n = 1, \ldots, N. 
\end{equation*}


\noindent {\bfseries The $\vec{\Bp} = \big[ \Bp_n \big]_{n=1}^N$-problem:} Fix $\vec{c} \,, \Bu \,, \Bv \,, \Beps$ and solve
\begin{align}  \label{eq:CombinedModel:pproblem:solution}
 \min_{\vec{\Bp} \in X^N} & \Bigg\{
 \mu_3 \sum_{n=1}^N \norm{\nabla^+_M \Bp_n}_{\ell_1}    
 + \frac{\mu_4}{2} \norm{ \Bu - \sum_{n=1}^N c_n \Bp_n }^2_{\ell_2} \notag \\
&\text{  s.t.  } 
 \sum_{n=1}^N \Bp_n = 1,
 p_n [\Bk] \in \{0, 1\}, n = 1, \ldots, N, \Bk \in \Omega   
 \Bigg\} \,.
\end{align}
From Proposition \ref{prop:combinedmodel:pproblem} with a smooth primal-dual model and Chambolle's projection,
the primal solution of (\ref{eq:CombinedModel:pproblem:solution}) at iteration $t$ (for $n = 1, \ldots, N$) is
\begin{align*}
 \Bp_n &= \frac{\displaystyle \exp\left\{ -\frac{1}{\xi} \Big[ \text{div}^-_M \vec{\Bq}_n 
                + \frac{\mu_4}{2 \mu_3} \big( \Bu - c_n \big)^{.2}
                \Big] \right\} }
              {\displaystyle \sum_{i=1}^N \exp \left\{ - \frac{1}{\xi} \Big[ \text{div}^-_M \vec{\Bq}_i 
                + \frac{\mu_4}{2 \mu_3} \big( \Bu - c_i \big)^{.2} \Big] \right\} }
 \\
 &= \frac{\displaystyle \exp\left\{ -\frac{1}{\xi} \Big[ -\sum_{m=0}^{M-1} \Big[ \sin(\frac{\pi m}{M}) \BDmT \Bq_{nm} + \cos(\frac{\pi m}{M}) \Bq_{nm} \BDn \Big] 
          + \frac{\mu_4}{2 \mu_3} \big( \Bu - c_n \big)^{.2}
          \Big] \right\} }
         {\displaystyle \sum_{i=1}^N \exp \left\{ - \frac{1}{\xi} \Big[ -\sum_{m=0}^{M-1} \Big[ \sin(\frac{\pi m}{M}) \BDmT \Bq_{im} + \cos(\frac{\pi m}{M}) \Bq_{im} \BDn \Big] 
          + \frac{\mu_4}{2 \mu_3} \big( \Bu - c_i \big)^{.2} \Big] \right\} } \,, 
\end{align*}
with dual variable 
\begin{align*}  
 \vec{\Bq}_n^{(t)} &= \frac{ \vec{\Bq}_n^{(t-1)} + \tau \nabla^+_M \Bp_n^{(t)} }{ 1 + \tau \abs{\nabla^+_M \Bp_n^{(t)}} }  
 ~ \in X^M
 \\ \Leftrightarrow~
 \Bq_{nm}^{(t)} &= \frac{\displaystyle \Bq_{nm}^{(t-1)} + \tau \left[ \sin(\frac{\pi m}{M}) \BDm \Bp_n^{(t)} + \cos(\frac{\pi m}{M}) \Bp_n^{(t)} \BDnT \right] }
                        {\displaystyle 1 + \tau \left[ \sum_{m=0}^{M-1} \big[ \sin(\frac{\pi m}{M}) \BDm \Bp_n^{(t)} + \cos(\frac{\pi m}{M}) \Bp_n^{(t)} \BDnT \big]^{.2} \right]^{.\frac{1}{2}} } 
\,,~ m = 0, \ldots, M-1 \,.
\end{align*}

\vspace{2mm}
\noindent Finally, we update the Lagrange multiplier as
\begin{equation*}
 \boldsymbol{\lambda}^{(t+1)} = \boldsymbol{\lambda}^{(t)} 
 + \beta \big( \Bf - \Bu - \Bv - \Beps \big) \,.
\end{equation*}
This solution is summarized in Algorithm 1.
Figures \ref{fig:Barbara:CombinedModel:NoNoise} and \ref{fig:Barbara:CombinedModel:Noise} depict
the segmented results without noise and with independent Gaussian noise, respectively.
In both cases, our proposed method provides good segmented results, though some large-scale texture
(e.g.\ the books shown in the upper left-hand corner) still remains in the piecewise constant images; see 
Figures \ref{fig:Barbara:CombinedModel:NoNoise} (f) and \ref{fig:Barbara:CombinedModel:Noise} (f).
This is likely due to the minimizer obtained by the primal-dual method with Chambolle's projection \cite{Chambolle2004}
and since there is no shrinkage to produce sparse signals in some transform domains.
Similar to (\ref{eq:twophaseDG3PDtexture:minimization:1}), one can use the technique in 
\cite{BaeYuanTai2010} to obtain a binary setting of $\Bp$. 
As in \cite{ThaiGottschlich2016G3PD, ThaiGottschlich2016DG3PD}, the convergence of the algorithm is
defined by a relative error on the log scale
\begin{align} \label{eq:relativeError:u}
 \text{Err}_{\Bu}(t) = \log \frac{\norm{\Bu^{(t)} - \Bu^{(t-1)}}_{\ell_2}}{\norm{\Bu^{(t-1)}}_{\ell_2}} \,,~ t = 1, 2,  \ldots \,
\end{align}


\begin{figure}[!ht]
\begin{center}  

 \includegraphics[height=0.58\textheight]{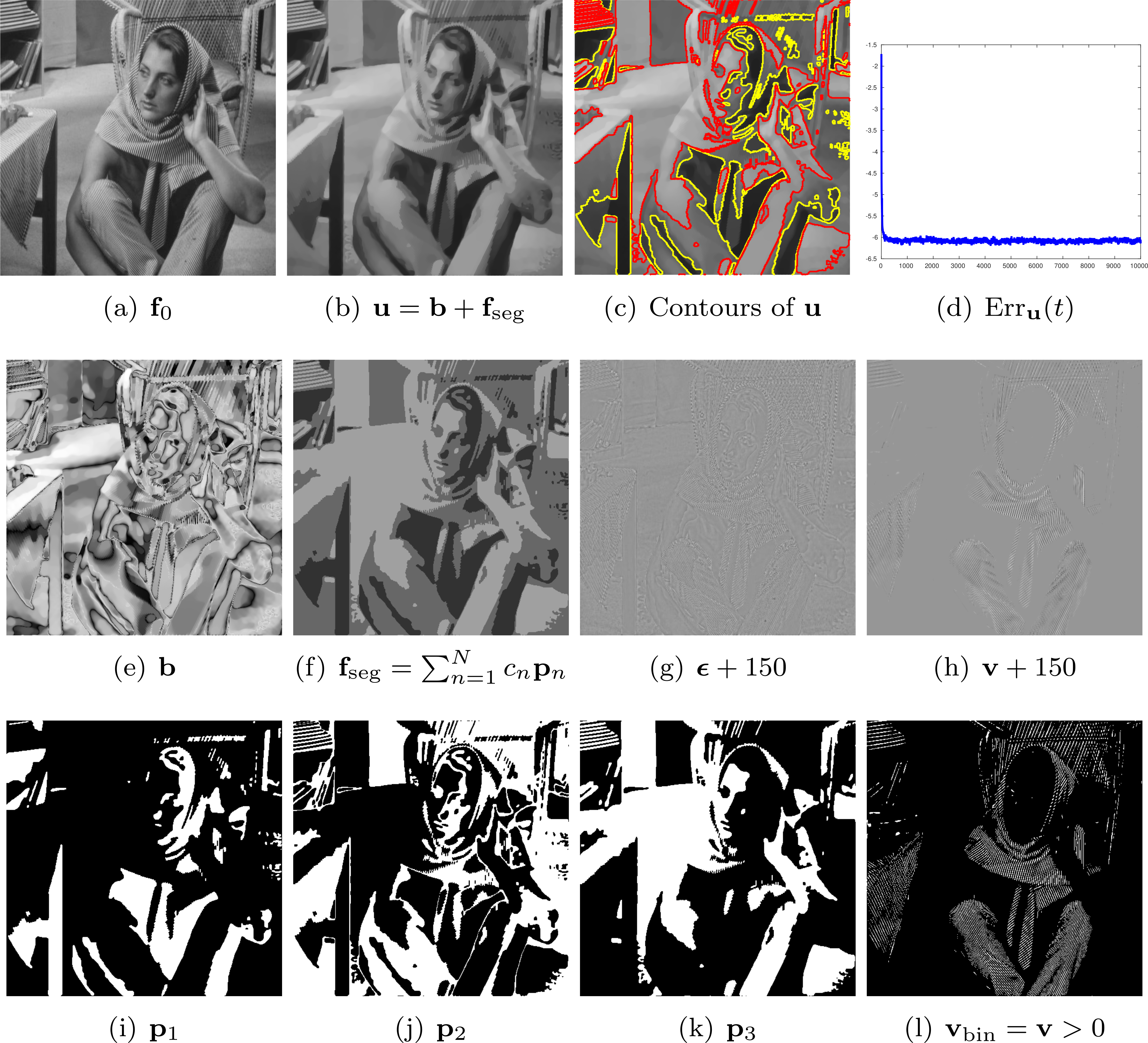}
 
\caption{The original image $\Bf$ (a) is decomposed into a bias field $\Bb$ (e), 
         a piecewise constant image $\Bf_\text{seg}$ (f), small scale objects (residual) $\Beps$ (g), 
         and sparse texture $\Bv$ (h) with binarized verion $\Bv_\text{bin}$ in (l).
         A piecewise smooth image $\Bu$ (b)
         is obtained by a summation of $\Bb$ (e) and $\Bf_\text{seg}$ (f).
         Subfigure (c) shows segmented contours superimposed on $\Bu$.  The relative error of $\Bu$ is shown in (d).
         The indicator functions for phases 1, 2, and 3 are shown in Subfigures i, j, and k, respectively.        
         The parameters are $\nu = 10, N = 3, L = S = M = 2, \tau = 0.1, \xi = 0.001, \alpha = \mu_1 = \mu_2 = \mu_3 = 0.1, 
         c_{\mu_1} = 0.14, \mu_4 = 0.01, \beta=0.04, \#\text{iteration} = 10000$.
         The mean square error of the original image $\Bf$ and a reconstructed image 
         $\Bf_\text{re} = \Bb + \Bf_\text{seg} + \Bv + \Beps$ is
         $\text{MSE} = 9.02 \times 10^{-5}$.}
\label{fig:Barbara:CombinedModel:NoNoise}
\end{center}
\end{figure}

\begin{figure}[!ht]
\begin{center}  

 \includegraphics[height=0.58\textheight]{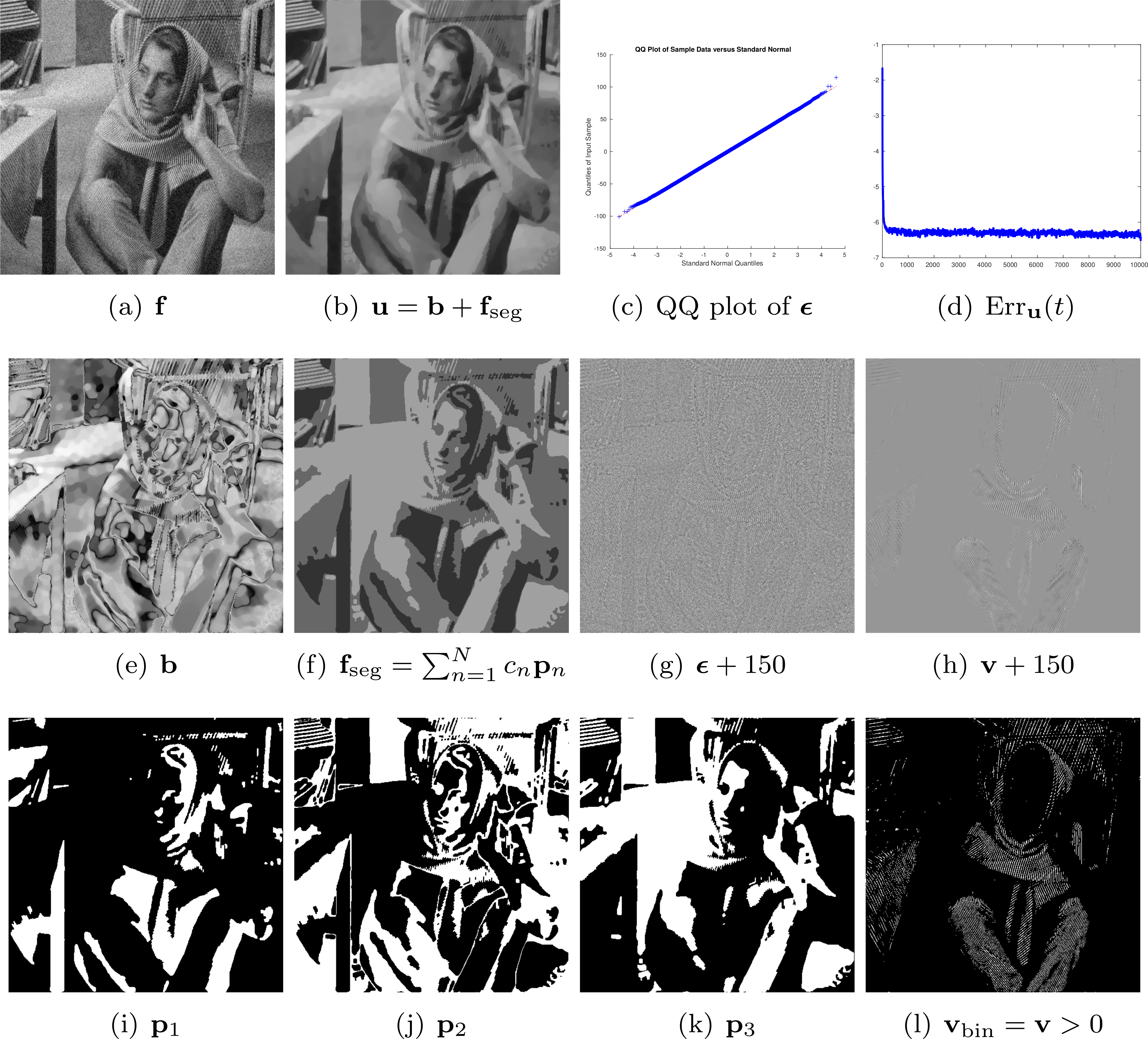}
 
 
 \caption{Original image $\Bf$ with added i.i.d. noise from $\mathcal N(0 \,, 20^2)$ is shown in (a).
          With the addition of noise, we choose $\nu = 16$ with the remaining parameters set similar to those in Figure \ref{fig:Barbara:CombinedModel:NoNoise}.
          The QQplot in (c) for noise $\Beps$ in (g) shows that $\nu = 16$ can separate most of the noise and
          some texture information.  The MSE is $7.47 \times 10^{-5}$.
          Note that increasing $L$ will not make $\Bu$ (b) smoother due to the lack of a sparsity constraint
          in Chambolle's projection.  The algorithm still performs well with sparse texture $\Bv$ as illustrated in (l).
         }
\label{fig:Barbara:CombinedModel:Noise}
\end{center}
\end{figure}

\begin{algorithm}
\label{alg:CombinedModel}
\caption{The SHT model}
\begin{algorithmic}
\small 
 
 \STATE{
  {\bfseries Initialization:}
  $\Bu^{(0)} = \Bf \,, \Bv^{(0)} = \Beps^{(0)} = \vec{\Br}^{(0)} = \vec{\Bg}^{(0)} = \vec{\Bp}^{(0)} = \vec{\Bq}^{(0)} =\boldsymbol 0 \,,~ 
   c_n^{(0)} = (n-1) \lfloor \frac{255}{N} \rfloor \,, n = 1, \ldots, N$.
 }  
 
\STATE{} 
\FOR{$t = 1 \,, \ldots \,, T$}
\STATE
{ 
{\bfseries I. Compute $\big( \vec{c}, \vec{\Br}, \Bu, \vec{\Bg}, \Bv, \Beps, \vec{\Bp}, \vec{\Bq}  \big) \in \mathbb R^N \times X^{L+S+N+NM+3}$:}
  \begin{align*}
   &\text{1.}~
   \Bh^{(t)} = \frac{\mu_4}{\mu_4 + \beta} \sum_{n=1}^N c_n^{(t-1)} \Bp_n^{(t-1)} + \frac{\beta}{\mu_4 + \beta} 
                \Big[ \Bf - \Bv^{(t-1)} - \Beps^{(t-1)} + \frac{\boldsymbol{\lambda}^{(t-1)}}{\beta} \Big]  
   \\&\text{2.}~
   \vec{\Br}^{(t)} = 
   \frac{ \vec{\Br}^{(t-1)} + \tau \nabla^+_L \Big[ \text{div}^-_L \vec{\Br}^{(t-1)} - (\mu_4 + \beta) \Bh^{(t)} \Big] }
        { 1 + \tau \abs{ \nabla^+_L \Big[ \text{div}^-_L \vec{\Br}^{(t-1)} - (\mu_4 + \beta) \Bh^{(t)} \Big] } }
   \\&\text{3.}~
   \Bu^{(t)} = \Bh^{(t)} - \frac{1}{\mu_4 + \beta} \text{div}^-_L \vec{\Br}^{(t)} 
   \\&\text{4.}~
   \vec{\Bg}^{(t)} = \frac{ \vec{\Bg}^{(t-1)} + \tau \nabla_S^+ \Big[ \alpha \mu_1 \text{div}^-_S \vec{\Bg}^{(t-1)} - \boldsymbol{\lambda}^{(t-1)} - \alpha \Bv^{(t)} \Big] }
                          { 1 + \tau \abs{\nabla^+_S \Big[ \alpha \mu_1 \text{div}^-_S \vec{\Bg}^{(t-1)} - \boldsymbol{\lambda}^{(t-1)} - \alpha \Bv^{(t)} \Big]} }  
   \\&\text{5.}~
   \Bv^{(t)} = \Shrink \Big( \underbrace{ \frac{ \frac{\beta}{\mu_2} }{\alpha + \frac{\beta}{\mu_2}} \big[ \Bf - \Bu^{(t)} - \Beps^{(t-1)} + \frac{\boldsymbol{\lambda}^{(t-1)}}{\beta} \big] + \frac{\alpha \mu_1}{\alpha + \frac{\beta}{\mu_2} } \text{div}^-_S \vec{\Bg}^{(t)} }_{:= \cT_\Bv}
                        , \frac{1}{\alpha + \frac{\beta}{\mu_2}}
                 \Big) \,,
   \\&\qquad
   \mu_2 = \frac{\displaystyle \beta c_{\mu_2} \max_{\Bk \in \Omega} \abs{\cT_v[\Bk]}}{\displaystyle 1 - \alpha c_{\mu_2} \max_{\Bk \in \Omega} \abs{\cT_v[\Bk]}}
   \\&\text{6.}~
   \Beps^{(t)} = \Big[\Bf - \Bu^{(t)} - \Bv^{(t)} + \frac{\boldsymbol{\lambda}^{(t-1)}}{\beta} \Big] 
               - \CST \Big( \Big[\Bf - \Bu^{(t)} - \Bv^{(t)} + \frac{\boldsymbol{\lambda}^{(t-1)}}{\beta} \Big] \,, \nu \Big)
   \\&\text{7.}~
   \Bp_n^{(t)} = \frac{\displaystyle \exp\Big\{ -\frac{1}{\xi} \Big[ \text{div}^-_M \vec{\Bq}_n^{(t-1)} 
                  + \frac{\mu_4}{2 \mu_3} \big( \Bu^{(t)} - c_n^{(t-1)} \big)^{\cdot 2}
                  \Big] \Big\} }
                {\displaystyle \sum_{i=1}^N \exp \Big\{ - \frac{1}{\xi} \Big[ \text{div}^-_M \vec{\Bq}_i^{(t-1)} 
                  + \frac{\mu_4}{2 \mu_3} \big( \Bu^{(t)} - c_i^{(t-1)} \big)^{\cdot 2} \Big] \Big\} }
   \,,~ n = 1, \ldots, N  
   \\&\text{8.}~    
   \vec{\Bq}_n^{(t)} = \frac{ \vec{\Bq}_n^{(t-1)} + \tau \nabla^+_M \Bp_n^{(t)} }{ 1 + \tau \abs{\nabla^+_M \Bp_n^{(t)}} }\,,~ n = 1, \ldots, N     
   \\&\text{9.}~
   c_n^{(t)} = \frac{\displaystyle \sum_{\Bk \in \Omega} u^{(t)}[\Bk] p_n^{(t)}[\Bk] }
              {\displaystyle \sum_{\Bk \in \Omega} p_n^{(t)}[\Bk] }  
   \,,~ n = 1, \ldots, N   
  \end{align*}

{\bfseries II. Update $\boldsymbol{\lambda} \in X$:}
\begin{equation*}
 \boldsymbol{\lambda}^{(t)} = \boldsymbol{\lambda}^{(t-1)} + \beta \Big( \Bf - \Bu^{(t)} - \Bv^{(t)} - \Beps^{(t)} \Big)
\end{equation*}


}
\ENDFOR
\end{algorithmic}
\end{algorithm}


%
%
%
%

\section{The Bilevel-SHT Model}
\label{sec:BiLevelMinimizationScheme}

We now propose an alternative to the multiphase SHT model.  As above, we assume that an image $\Bf$ is composed of a homogeneous region (consisting of a bias field $\Bb$ and piecewise-constant with mean values $\vec{c}$ and indicator functions $\vec{\Bp}$) as well as texture $\Bv$ and residual $\Beps$, but we now consider a bilevel scheme 
for decomposing the image into these base components.  Specifically, we consider the decomposition and segmentation 
as separate levels:
\begin{itemize}
 \item {\bfseries Level 1:} Image decomposition
  \begin{equation*}
   \Bf = \Bu + \Bv + \Beps
  \end{equation*}
  
 \item {\bfseries Level 2:} Multiphase piecewise-smooth image segmentation
  \begin{equation*}
   \Bu = \Bb + \sum_{n=1}^N c_n \Bp_n \,,~  \sum_{n=1}^N \Bp_n = 1 \,,~
   p_n [\Bk] \in \{0, 1\} \,,~ n = 1, \ldots, N \,,~ \Bk \in \Omega \, .
  \end{equation*}
\end{itemize}
The bilevel minimization scheme for simultaneously homogeneous and textural (SHT) image segmentation is defined as
\begin{align} \label{eq:BiLevelMinimization:1}
 \min_{\big( \vec{c} \,, \vec{\Bp} \,, \Bb \big) \in \mathbb R^N \times X^{N+1}} \Bigg\{
 &\min_{(\Bu, \Bv, \Beps) \in X^3} \bigg\{ 
 \mathcal I_1 \big(\Bu, \Bv, \Beps \big) \text{  s.t.  } \Bf = \Bu + \Bv + \Beps
 \bigg\}     
 + \mathcal I_2 \big( \vec{\Bp}, \Bb \big)
 \text{  s.t.  } \mathcal S \big( \vec{c}, \vec{\Bp}, \Bb; \Bu \big)
 \Bigg\}
\end{align}
with set $\mathcal S$
\begin{equation*}
 \mathcal S \big( \vec{c} \,, \vec{\Bp} \,, \Bb \,; \Bu \big) = \left\{
 \Bu = \Bb + \sum_{n=1}^N c_n \Bp_n \,,~ \sum_{n=1}^N \Bp_n = 1 \,,~
 p_n [\Bk] \in \{0, 1\} \,,~ n = 1, \ldots, N \,,~ \Bk \in \Omega 
 \right\}
\end{equation*}
and energy functions
\begin{align*}
 \mathcal I_1(\Bu, \Bv, \Beps) &~=~ \norm{\nabla^+_L \Bu}_{\ell_1} + \mu_1 \norm{\Bv}_{\text{G}_S} 
 + \mu_2 \norm{\Bv}_{\ell_1} + \mathscr G^*\left(\frac{\Beps}{\nu}\right) ~~ \text{and}
 \\
 \mathcal I_2 \big( \vec{\Bp} \,, \Bb \big) &~=~ 
 \sum_{n=1}^N \norm{\nabla^+_M \Bp_n}_{\ell_1} + \frac{\mu_3}{2} \norm{ \Bb }^2_{\ell_2}.
\end{align*}
Similar to the above multiphase SHT model, this bilevel-SHT model also measures a bias field $\Bb$ via $\ell_2$ distance
as data fidelity term in the regularization. 
In contrast with \cite{GuWangXiongChengHuangZhou2013, GuXiongWangChengHuangZhou2014}, 
we enforce the constraint for the smoothness in $\Bu$ with $\norm{ \nabla^+_L \Bu }_{\ell_1}$.

\subsection{Solution of the bilevel-SHT model}
We now describe a numerical algorithm to obtain the solution of the bilevel-SHT model (\ref{eq:BiLevelMinimization:1}):
\begin{itemize}
 \item {\bfseries Level 1:} {\bfseries D}irectional {\bfseries G}lobal {\bfseries T}hree-{\bfseries p}art {\bfseries D}ecomposition (DG3PD)
  \begin{equation}  \label{eq:BiLevelMinimization:2:step1}
   (\Bu^*, \Bv^*, \Beps^*) ~= \argmin_{(\Bu, \Bv, \Beps) \in X^3} \Big\{ 
   \mathcal I_1 \big(\Bu, \Bv, \Beps \big) \text{  s.t.  } \Bf = \Bu + \Bv + \Beps
   \Big\}
  \end{equation}

 \item {\bfseries Level 2:} {\bfseries S}imultaneously {\bfseries H}omogeneous and {\bfseries T}exture {\bfseries M}ultiphase {\bfseries S}egmentation (SHTMS)
  \begin{equation} \label{eq:BiLevelMinimization:2:step2}
   \big( \vec{c}^* \,, \vec{\Bp}^* \,, \Bb^* \big) ~=
   \argmin_{\big( \vec{c}, \vec{\Bp}, \Bb \big) \in \mathbb R^N \times X^{N+1}} \Big\{
   \mathcal I_2 \big( \vec{\Bp}, \Bb \big)
   \text{  s.t.  } \mathcal S \big( \vec{c} \,, \vec{\Bp} \,, \Bb \,; \Bu^* \big)
   \Big\}
  \end{equation}
\end{itemize}
As alluded to above, we first decompose the original image $\Bf$
into piecewise-smooth, texture, and residual components $\Bu$, $\Bv$, and $\Beps$.
We then segment the piecewise-smooth image $\Bu$ into multiphase $(N)$ piecewise-constant images and 
a bias field $\Bb$.
Sparse (or segmented) texture $\Bv$ is measured by $\norm{\Bv}_{\ell_1}$ and $\norm{\Bv}_{\text{G}_S}$ and we note that $\text{G}_S$ is a generalized version of the Banach space G in the discrete setting \cite{AujolChambolle2005, ThaiGottschlich2016DG3PD, Meyer2001}.
Note also that though we assume the original image $\Bf$ contains both texture 
and homogeneous areas, only the homogeneous areas are segmented by level 2.

\subsubsection{Solution of Level 1 - DG3PD} 
The solution of the convex minimization in (\ref{eq:BiLevelMinimization:2:step1}) is defined in
\cite[Algorithm 1]{ThaiGottschlich2016DG3PD} and \cite{ThaiGottschlich2016G3PD} and is solved by introducing new variables and applying ALM and ADMM.
In an effort to make this paper self-contained, the kernel of the DG3PD method is provided in Algorithm 3.

Note that DG3PD approximates $\norm{ \vec{\Bg} }_{\ell_\infty}$ in the $\text{G}_S$-norm by $\norm{ \vec{\Bg} }_{\ell_1}$;
see \cite{VeseOsher2003} for details. This approximation in $\ell_1$-norm enforces sparsity of $\vec{\Bg}$ (in our case, the texture $\Bv$).

\subsubsection{Solution of Level 2 - SHTMS}
Note that since the bias field is defined as 
\[
\Bb = \Bu - \sum_{n=1}^N c_n \Bp_n
\]
with the binary set $p_n[\Bk] \in \{ 0 \,, 1 \}$, 
we can rewrite the $\ell_2$-norm and recast the non-convex minimization in (\ref{eq:BiLevelMinimization:2:step2}) as
\begin{align} \label{eq:BiLevelMinimization:2:step2:1}
 \min_{(\vec{c}, \vec{\Bp}) \in \mathbb R^N \times X^N} \bigg\{
 & \sum_{n=1}^N \norm{\nabla^+_M \Bp_n}_{\ell_1} 
 + \frac{\mu_3}{2} \sum_{n=1}^N \Big\langle \big(\Bu - c_n \big)^{.2} \,, \Bp_n \Big\rangle_{\ell_2} \,, \notag \\
&\hspace{10mm} \sum_{n=1}^N p_n[\Bk] = 1 , p_n [\Bk] \in \{0, 1\} , n = 1, \ldots, N , \Bk \in \Omega 
 \bigg\} \,.
\end{align}
As in \cite{ThaiGottschlich2016DG3PD, ThaiGottschlich2016G3PD},
a solution of the multivariate minimization (\ref{eq:BiLevelMinimization:2:step2:1}) can be obtained
by alternating between solving the following two subproblems:

{\bfseries a. The $\vec{c} = \big[ c_n \big]_{n=1}^N$ problem:} Fix $\vec{\Bp}$ and solve

\begin{equation} \label{eq:BiLevelMinimization:level2:cproblem}
 \min_{\vec{c} \in \mathbb R^N} 
 \bigg\{ \mathcal L (\vec{c}) =
 \frac{\mu_3}{2} \sum_{n=1}^N \Big\langle \big(\Bu - c_n \big)^{.2} \,, \Bp_n \Big\rangle_{\ell_2}
 \bigg\}.
\end{equation}
Due to its separability, the solution of (\ref{eq:BiLevelMinimization:level2:cproblem}) is given by 
\begin{equation} \label{eq:caproblem:solution}
 c_n = \frac{\displaystyle \sum_{\Bk \in \Omega} u[\Bk] p_n[\Bk] }
             {\displaystyle \sum_{\Bk \in \Omega} p_n[\Bk] }  
 \,,~ n = 1, \ldots, N. 
\end{equation}

{\bfseries b. The $\vec{\Bp} = \big[ \Bp_n \big]_{n=1}^N$ problem:} Fix $\vec{c}$ and find $\vec{\Bp}$.

As in Section \ref{sec:CombinedModel}, the
nonconvex minimization in (\ref{eq:BiLevelMinimization:2:step2:1}) w.r.t. $\vec{\Bp}$
is relaxed and made convex by setting a binary set $p_n[\Bk] \in \{0 \,, 1\}$ to $[0 \,, 1]$.
Following \cite{BaeYuanTai2010, GuXiongWangChengHuangZhou2014},                   
we apply a smoothed dual formulation by introducing 
the primal, primal-dual, and dual models:
              
{\bfseries The primal model:} Solve 
\begin{equation} \label{eq:BiLevelMinimization:2:step2:3}
 \min_{\vec{\Bp} \in \mathcal Q_+} 
 \bigg\{ \mathcal L^\text{P} (\vec{\Bp}) =
 \sum_{n=1}^N \norm{\nabla^+_M \Bp_n}_{\ell_1}
 + \frac{\mu_3}{2} \sum_{n=1}^N \Big\langle \big(\Bu - c_n \big)^{.2} \,, \Bp_n \Big\rangle_{\ell_2}
 \bigg\}
\end{equation}
over the convex set 
\begin{equation} \label{eq:BiLevelMinimization:Level2:pproblem:PrimalModel} 
 \mathcal Q_+ = \Big\{ \big[\Bp_n\big]_{n=1}^N \in X^N ~:~ 
 \sum_{n=1}^N p_n[\Bk] = 1 \,,~
 p_n [\Bk] > 0 \,,~ n = 1, \ldots, N \,,~ \Bk \in \Omega  
 \Big\}.
\end{equation}

{\bfseries The primal-dual model:}

Denote a convex set
$K_M(1) = \Big\{ \vec{\Bq}_n \in X^M ~:~ \norm{\vec{\Bq}_n}_{\ell_\infty} \leq 1 \Big\}$
with a dual variable $\vec{\Bq} = \big[ \vec{\Bq}_n \big]_{n=1}^N = \big[ \Bq_{nm} \big]_{n=[1,N]}^{m=[0,M-1]} \in X^{N M}$ 
of a primal variable $\vec{\Bp} = \big[ \Bp_n \big]_{n=1}^N \in X^N$.
From Lemma \ref{lem:DTVDGnorm:dualpair} in the Appendix (for the dual formulation of the directional total variation norm)
and the minimax theorem as found in \cite[Chapter 6]{EkelandTeman1999} and \cite{BaeYuanTai2010}, 
the primal-dual model is defined as
\begin{equation} \label{eq:BiLevelMinimization:2:step2:primaldual}
 \max_{ \vec{\Bq} \in \big[K_M(1)\big]^N }
 \min_{ \vec{\Bp} \in \mathcal Q_+} \bigg\{
 \mathcal L^\text{PD}(\vec{\Bp}; \vec{\Bq}) =
 \sum_{n=1}^N \Big\langle \Bp_n \,, \frac{\mu_3}{2} \big(\Bu - c_n \big)^{.2} + \text{div}^-_M \vec{\Bq}_n \Big\rangle_{\ell_2}  
 \bigg\}.
\end{equation}

{\bfseries The smoothed primal-dual model:}  Solve
\begin{equation} \label{eq:BiLevelMinimization:2:step2:SmoothedPrimalDual1}
 \max_{ \vec{\Bq} \in \big[K_M(1)\big]^N }
 \underbrace{
 \min_{ \vec{\Bp} \in \mathcal Q_+} \bigg\{
 \mathcal L^\text{PD}_{\xi>0}(\vec{\Bp}; \vec{\Bq}) = 
 \sum_{n=1}^N \Big\langle \Bp_n \,, \frac{\mu_3}{2} \big(\Bu - c_n \big)^{.2} + \text{div}^-_M \vec{\Bq}_n \Big\rangle_{\ell_2}
 + \xi \sum_{n=1}^N \Big\langle \Bp_n \,, \log \Bp_n \Big\rangle_{\ell_2} 
 \bigg\}
 }_{\displaystyle = \cL^\text{D}_{\xi>0} (\vec{\Bq}) := -\xi \sum_{\Bk \in \Omega} \log \Big[ \sum_{n=1}^N
    \exp \Big\{ -\frac{1}{\xi} \Big[ \frac{\mu_3}{2} (u[\Bk] - c_n)^2 + \text{div}^-_M \vec{q}_n[\Bk] \Big] \Big\}
    \Big] } \,.
\end{equation}
From Proposition \ref{prop:bilevel:smoothedprimaldualproblem} in the Appendix,
the solution of the primal $\vec{\Bp}$-problem
\begin{equation*}
 \min_{ \vec{\Bp} \in \mathcal Q_+} \mathcal L^\text{PD}_{\xi>0}(\vec{\Bp}; \vec{\Bq}) 
\end{equation*}
is given by (with $n = 1, \ldots, N$)
\begin{align*}
 \Bp^*_n &= \frac{\displaystyle \exp\bigg[ - \frac{1}{\xi} \Big[ \frac{\mu_3}{2} \big(\Bu - c_n \big)^{.2} + \text{div}^-_M \vec{\Bq}_n \Big]\bigg] }
                 {\displaystyle \sum_{i=1}^N \exp\bigg[ - \frac{1}{\xi} \Big[ \frac{\mu_3}{2} \big(\Bu - c_i \big)^{.2} + \text{div}^-_M \vec{\Bq}_i \Big]\bigg]  }
 \\& 
 = \frac{\displaystyle \exp\bigg[ - \frac{1}{\xi} \Big[ \frac{\mu_3}{2} \big(\Bu - c_n \big)^{.2} - \sum_{m=0}^{M-1} \big[ \sin\left(\frac{\pi m}{M}\right) \BDmT \Bq_{nm} + \cos\left(\frac{\pi m}{M}\right) \Bq_{nm} \BDn \big] \Big]\bigg] }
        {\displaystyle \sum_{i=1}^N \exp\bigg[ - \frac{1}{\xi} \Big[ \frac{\mu_3}{2} \big(\Bu - c_i \big)^{.2} - \sum_{m=0}^{M-1} \big[ \sin\left(\frac{\pi m}{M}\right) \BDmT \Bq_{im} + \cos\left(\frac{\pi m}{M}\right) \Bq_{im} \BDn \big] \Big]\bigg]  } \,.
\end{align*}
Due to its separability, we consider the dual $\vec{\Bq}$-problem 
\begin{equation} \label{eq:BiLevelMinimization:2:step2:SmoothedDual1}
 \max_{ \vec{\Bq} \in \big[K_M(1)\big]^N } \cL^\text{D}_{s>0} (\vec{\Bq}) 
\end{equation}
at $n = 1, \ldots, N$.
Given $\displaystyle  \vec{\Bq}_n = \big[ \Bq_{nm} \big]_{m=0}^{M-1}$ and
$\nabla^+_M \Bp_n = \big[ \partial^+_m \Bp_n \big]_{m=0}^{M-1}$,
the solution of (\ref{eq:BiLevelMinimization:2:step2:SmoothedDual1}) which is solved by Chambolle's projection \cite{Chambolle2004} at each iteration $t$ is
\begin{align*}
 \vec{\Bq}_n^{(t+1)} = \frac{ \vec{\Bq}_n^{(t)} + \tau \nabla^+_M \Bp_n^{(t)} }{ 1 + \tau \abs{\nabla^+_M \Bp_n^{(t)}} }
 \,,~~ n = 1, \ldots, N 
\end{align*}
and its element form is (with $m = 0, \ldots, M-1$)
\begin{align*}
 \Bq_{nm}^{(t+1)} = \frac{\displaystyle \Bq_{nm}^{(t)} + \tau \big[ \sin\left(\frac{\pi m}{M}\right) \BDm \Bp_n + \cos\left(\frac{\pi m}{M}\right) \Bp_n \BDnT \big] }
                         {\displaystyle 1 + \tau \Bigg[ \sum_{m=0}^{M-1} \bigg[ \sin\left(\frac{\pi m}{M}\right) \BDm \Bp_n + \cos\left(\frac{\pi m}{M}\right) \Bp_n \BDnT \bigg]^{.2} \Bigg]^{.\frac{1}{2}} } \,.
\end{align*}
The numerical solution of the bilevel SHT model is described in
Algorithms 2-4.  
Figure \ref{fig:Barbara:BilevelModel} shows the bilevel SHT model applied to the same noisy image as in Figure 
\ref{fig:Barbara:CombinedModel:Noise}.  Note, by comparing the upper left-hand corners of subfigures (l) and (k) of 
Figures \ref{fig:Barbara:CombinedModel:Noise} and \ref{fig:Barbara:BilevelModel} respectively, that the bilevel SHT
model does a better job of fully segmenting the large scale texture from the homogeneous regions.  However, these 
binarized versions also reveal that the bilevel SHT model is slightly oversensitive as some small artifacts are introduced.
Finally, in Figure \ref{fig:galaxy:BilevelModel:2} we apply the bilevel SHT model to an image of a galaxy with many stars in the background.  Although the stars may constitute small-scale texture, in cases such as these we may set the 
texture component ($\Bv$) to 0, thereby treating this fine texture as noise.



\begin{algorithm}
\label{alg:SHTMS}
\caption{The Bilevel-SHT model}
\begin{algorithmic}
\small 
 \STATE{ {\bfseries Denote parameters:}
 $\kappa_{\text{d}} = \Big[  L, S, c_{\mu_1}, c_{\mu_2}, \big[\beta_i\big]_{i=1}^4, \nu  \Big]$
 and
 $\kappa_{\text{s}} = \Big[  M, N, \xi, \mu_3, \tau \Big]$
 }

 \STATE{ {\bfseries Denote variables:}
 \begin{align*}
  \theta = \bigg[
  \big[ \Br_l \big]_{l=0}^{L-1} \,, \big[ \Bw_s \big]_{s=0}^{S-1} \,, \big[ \Bg_s \big]_{s=0}^{S-1}
  \,, \big[\boldsymbol{\lambda}_{\mathbf{1}l}\big]_{l=0}^{L-1} \,, \big[\boldsymbol{\lambda}_{\mathbf{2}a}\big]_{a=0}^{S-1}
  \,, \boldsymbol{\lambda}_{\mathbf 3} \,, \boldsymbol{\lambda}_{\mathbf 4}  
  \bigg]
 \end{align*}
 }
 
 \STATE{
  {\bfseries Initialization:}$ \, \Bu^{(0)} = \Bf \,, \Bv^{(0)} = \Beps^{(0)} = \theta^{(0)} = \vec{\Bp}^{(0)} = \vec{\Bq}^{(0)} =\boldsymbol 0 \,,~ c_n^{(0)} = (n-1) \lfloor \frac{255}{N} \rfloor \,, n = 1, \ldots, N$.
 }  
 
\STATE{} 
\FOR{$t_1 = 1 \,, \ldots \,, T_1$}
\STATE
{ 

  \STATE{{\bfseries Level 1 (Decomposition):} $(\Bu, \Bv, \Beps) \in X^3$}
  \FOR{$t_2 = 1 \,, \ldots \,, T_2$}    
  \STATE
  {
  \begin{equation*}
   \Big[ \Bu^{(t_2)} \,, \Bv^{(t_2)} \,, \Beps^{(t_2)} \,, \theta^{(t_2)} \Big] = 
   \text{DG3PD} \Big( \Bu^{(t_2-1)} \,, \Bv^{(t_2-1)} \,, \Beps^{(t_2-1)} \,, \theta^{(t_2-1)} \,;~ \Bf \,, \kappa_\text{d} \Big)
  \end{equation*}
  }
  \ENDFOR
  
  \STATE{} 
  
  \STATE{{\bfseries Level 2 (Multiphase Segmentation of $\Bu$)}: $(\vec{c} \,, \vec{\Bp} \,, \vec{\Bq}) \in \mathbb R^N \times X^{N + NM}$} 
  \begin{equation*}
   \Big[ \vec{c}^{(t_1)} \,, \vec{\Bp}^{(t_1)} \,, \vec{\Bq}^{(t_1)} \Big] = 
   \text{SHTMS} \Big( \vec{c}^{(t_1-1)} \,, \vec{\Bp}^{(t_1-1)} \,, \vec{\Bq}^{(t_1-1)} 
   \,;~ \Bu^{(T_2)} \,, \kappa_\text{s}
   \Big)
  \end{equation*}
 }
\ENDFOR

\STATE{} 
\STATE{} 

\STATE{{\bfseries Global minimizer of (\ref{eq:BiLevelMinimization:1}):}
\begin{align*}
 \Bu^* &= \Bu^{T_1 T_2} \,, \Bv^* = \Bv^{T_1 T_2} \,, \Beps^* = \Beps^{T_1 T_2} \,,
 \vec{\Bq}^* = \vec{\Bq}^{T_1} \,, \vec{c}^* = \vec{c}^{T_1} \,,
 ~ \text{and}
 \\
 \Bp_h^* &= 
 \begin{cases}
  1 \,, & h = \displaystyle \argmin_{1 \leq n \leq N} \bigg\{ \underbrace{ -\sum_{m=0}^{M-1} \big[ \sin\left(\frac{\pi m}{M}\right) \BDmT \Bq^*_{nm} + \cos\left(\frac{\pi m}{M}\right) \Bq^*_{nm} \BDn \big] }_{ = \text{div}^-_M \vec{\Bq}_n^*} 
              + \frac{\beta_5}{2} \big( \Bu^* - c_n^* \big)^{.2}  
              \bigg\}
  \\
  0 \,, & \text{else}
 \end{cases}
 \,, h = 1, \ldots, N
\end{align*}

} 
\end{algorithmic}
\end{algorithm}

\begin{algorithm}
\label{alg:Level1:DG3PD:Part1}
\caption{Level 1: The Discrete DG3PD Model \cite{ThaiGottschlich2016DG3PD}}
\begin{algorithmic}
\small

\STATE
{
\begin{align*}
 \Big[ \Bu^{(\text{new})} \,, \Bv^{(\text{new})} \,, \Beps^{(\text{new})} \,, \theta^{(\text{new})} \Big] = 
 \text{DG3PD} \Big( \Bu^{(\text{old})} \,, \Bv^{(\text{old})} \,, \Beps^{(\text{old})} \,, \theta^{(\text{old})} \,;~ \Bf \,, \kappa_\text{d} \Big)
\end{align*}
}

\STATE
{
  {\bfseries 1. Compute}
   $\Big( \big[ \Br_b^{(t)} \big]_{b=0}^{L-1} \,, \big[ \mathbf{w}_a^{(t)} \big]_{a=0}^{S-1} \,,
    \big[ \mathbf{g}_a^{(t)} \big]_{a=0}^{S-1} \,, \Bv^{(t)} \,, \Bu^{(t)} \,, \Beps^{(t)} \Big) \in X^{L+2S+3}
   $:
   \begin{align*}
    \Br_b^{(t)} &~=~ \Shrink \Big( \sin\left( \frac{\pi b}{L} \right) \BDm \Bu^{(t-1)} + \cos\left( \frac{\pi b}{L} \right) \Bu^{(t-1)} \BDnT 
    - \frac{\boldsymbol{\lambda}_{\boldsymbol{1} b}^{(t-1)}}{\beta_1} \,, \frac{1}{\beta_1} \Big)
    \,,~ b = 0 \,, \ldots \,, L-1
    \\
    \mathbf{w}_a^{(t)} &~=~
    \Shrink \Big( \mathbf{t}_{\mathbf{w}_a} ~:=~ \mathbf{g}_a^{(t-1)} -
    \frac{ \boldsymbol{\lambda}_{\boldsymbol{2}a}^{(t-1)} }{\beta_2} \,, \frac{\mu_1}{\beta_2} \Big)
    \,,~ \quad a = 0 \,, \ldots \,, S-1
    \\
    \mathbf{g}_a^{(t)} &~=~ \RE \bigg[ \mathcal F^{-1} \Big\{ \mathcal A^{(t)}(\Bz) \cdot \mathcal B^{(t)}(\Bz) \Big\} \bigg] [\Bk] \Big |_{\Bk \in \Omega}
    \,,~ a = 0 \,, \ldots \,, S-1
    \\
    \Bv^{(t)} &= \Shrink \bigg( \mathbf{t_v} ~:=~ \frac{\beta_3}{\beta_3 + \beta_4}
    \bigg( \sum_{s=0}^{S-1} \Big[ \sin\left( \frac{\pi s}{S} \right) \BDm \Bg_s^{(t)} + \cos\left( \frac{\pi s}{S} \right) \Bg_s^{(t)} \BDnT \Big] - \frac{\boldsymbol{\lambda}_{\boldsymbol 3}^{(t-1)}}{\beta_3} \bigg)
    \\& \qquad \qquad \qquad \qquad
    + \frac{\beta_4}{\beta_3 + \beta_4} \bigg( \Bf - \Bu^{(t-1)} - \Beps^{(t-1)} + \frac{\boldsymbol{\lambda}_{\boldsymbol 4}^{(t-1)}}{\beta_4} \bigg)
    \,,~ \frac{\mu_2}{\beta_3 + \beta_4}
    \bigg)   
    \\
    \Bu^{(t)} &~=~ \RE \bigg[ \mathcal F^{-1} \Big\{ \mathcal X^{(t)}(\Bz) \cdot \mathcal Y^{(t)}(\Bz) \Big\} \bigg] [\Bk] \Big |_{\Bk \in \Omega}
    \\
    \Beps^{(t)} &~=~ \Big( \Bf - \Bu^{(t)} - \Bv^{(t)} + \frac{\boldsymbol{\lambda}_{\boldsymbol 4}^{(t-1)}}{\beta_4} \Big) ~-~
    \CST \big( \Bf - \Bu^{(t)} - \Bv^{(t)} + \frac{\boldsymbol{\lambda}_{\boldsymbol 4}^{(t-1)}}{\beta_4} \,, \nu \big)
   \end{align*} 
  {\bfseries 2. Update}
  $\Big( \big[\boldsymbol{\lambda}_{\mathbf{1}b}^{(t)}\big]_{b=0}^{L-1} \,, \big[\boldsymbol{\lambda}_{\mathbf{2}a}^{(t)}\big]_{a=0}^{S-1}
  \,, \boldsymbol{\lambda}_{\boldsymbol 3}^{(t)} \,, \boldsymbol{\lambda}_{\boldsymbol 4}^{(t)} \Big) \in X^{L+S+2}$:
  \begin{align*}
   \boldsymbol{\lambda}_{\mathbf{1}b}^{(t)} &~=~ \boldsymbol{\lambda}_{\mathbf{1}b}^{(t-1)}
   ~+~ \gamma \beta_1 \Big( \mathbf{r}_b^{(t)} - \sin\left( \frac{\pi b}{L} \right) \BDm \Bu^{(t)} - \cos\left( \frac{\pi b}{L} \right) \Bu^{(t)} \BDnT \Big) 
   \,, \quad b = 0 \,, \ldots \,, L-1
   \\
   \boldsymbol{\lambda}_{\mathbf{2} a}^{(t)} &~=~ \boldsymbol{\lambda}_{\mathbf{2} a}^{(t-1)}
   ~+~ \gamma \beta_2 \Big( \mathbf{w}_a^{(t)} - \mathbf{g}_a^{(t)} \Big) 
   \,, \quad a = 0 \,, \ldots \,, S-1
   \\
   \boldsymbol{\lambda}_{\mathbf{3}}^{(t)} &~=~ \boldsymbol{\lambda}_{\mathbf{3}}^{(t-1)}
   ~+~ \gamma \beta_3 \Big( \Bv^{(t)} - \sum_{s=0}^{S-1} \big[ \sin\left( \frac{\pi s}{S} \right) \BDm \mathbf{g}_s^{(t)} + \cos\left(\frac{\pi s}{S} \right) \mathbf{g}_s^{(t)} \BDnT \big] \Big) 
   \\
   \boldsymbol{\lambda}_{\mathbf{4}}^{(t)} &~=~ \boldsymbol{\lambda}_{\mathbf{4}}^{(t-1)}
   ~+~ \gamma \beta_4 \big( \Bf - \Bu^{(t)} - \Bv^{(t)} - \Beps^{(t)} \big) 
  \end{align*}
 }
\end{algorithmic}
\end{algorithm}

\begin{algorithm}
\label{alg:Level1:DG3PD:Part2}
\begin{algorithmic}
\small
\STATE
{\bfseries (continued Algorithm 3)}
\begin{align*}
 &\mathcal A(\Bz) ~=~ \Bigg[ \beta_2 \mathbf{1_{mn}} + \beta_3 \abs{ \sin \left( \frac{\pi a}{S} \right) (z_1 - 1) + \cos \left( \frac{\pi a}{S} \right) (z_2 - 1) }^2 \Bigg]^{-1}
 \,,
 \\
 &\mathcal B(\Bz) ~=~
 \beta_2 \Big[ W_a(\Bz)  + \frac{\Lambda_{2a}(\Bz) }{\beta_2} \Big]
 ~+~ \beta_3 \Big[ \sin\left( \frac{\pi a}{S} \right) (z_1^{-1} - 1) + \cos\left( \frac{\pi a}{S} \right) (z_2^{-1} - 1) \Big] \times
 \\&
 \bigg[ V(\Bz) - \sum_{s=[0\,,S-1] \backslash \{a\} }
 \Big[ \sin\left( \frac{\pi s}{S} \right) (z_1 - 1) + \cos\left( \frac{\pi s}{S} \right) (z_2 - 1) \Big] G_s(\Bz)
 + \frac{\Lambda_3(\Bz)}{\beta_3} \bigg] \,,
 \\
 &\mathcal X(\Bz) =
 \Bigg[ \beta_4 \mathbf{1_{mn}} + \beta_1 \sum_{l=0}^{L-1} \abs{ \sin\left( \frac{\pi l}{L} \right) (z_1 - 1) + \cos\left( \frac{\pi l}{L} \right) (z_2 - 1) }^2 \Bigg]^{-1} \,,
 \\
 &\mathcal Y(\Bz) =
 \beta_4 \Big[ F(\Bz) - V(\Bz) - \mathcal{E}(\Bz) + \frac{\Lambda_4(\Bz)}{\beta_4} \Big]
 + \beta_1 \sum_{l=0}^{L-1} \Big[ \sin \left( \frac{\pi l}{L} \right) (z_1^{-1} - 1) + \cos \left( \frac{\pi l}{L} \right) (z_2^{-1} -1) \Big]
 \Big[ R_l(\Bz) + \frac{\Lambda_{1l}(\Bz)}{\beta_1} \Big] .
\end{align*}
{\bfseries Choice of Parameters}
\begin{align*}
 \mu_1 &= c_{\mu_1} \beta_2 \cdot \max_{\Bk \in \Omega} \big( \abs{t_{\Bw_a}[\Bk]} \big) \,,~
 \mu_2 = c_{\mu_2} (\beta_3 + \beta_4) \cdot \max_{\Bk \in \Omega} \big( \abs{t_\Bv[\Bk]} \big)
 \text{  and  }
 \beta_2 = c_{\beta_{2}} \beta_3 \,, \beta_3 = \frac{\theta}{1 - \theta} \beta_4 \,, \beta_1 = c_{\beta_{1}} \beta_4.
\end{align*}
\end{algorithmic}
\end{algorithm}

\begin{algorithm}
\label{alg:Level2:MultiphaseSegmentation}
\caption{Level 2: The SHTMS}
\begin{algorithmic}
\small

\STATE
{
\begin{equation*}
 \Big[ \vec{c}^{(t+1)} \,, \vec{\Bp}^{(t+1)} \,, \vec{\Bq}^{(t+1)} \Big] = 
 \text{SHTMS} \Big( \vec{c}^{(t)} \,, \vec{\Bp}^{(t)} \,, \vec{\Bq}^{(t)} 
 \,;~ \Bu \,, \kappa_{\text{s}}
 \Big)
\end{equation*}
}

\STATE{
{\bfseries Compute}
$\big( \vec{c} \,, \vec{\Bp} \,, \vec{\Bq} \big) \in \mathbb R^N \times X^{N + NM}$:
\begin{align*}
 c_n^{(t+1)} &= \frac{\displaystyle \sum_{\Bk \in \Omega} u[\Bk] p_n^{(t)}[\Bk] }
             {\displaystyle \sum_{\Bk \in \Omega} p_n^{(t)}[\Bk] }  
 \,,~ n = 1, \ldots, N
 \\
 \Bp^{(t+1)}_n 
 &= \frac{\displaystyle \exp\bigg[ - \frac{1}{\xi} \Big[ -\sum_{m=0}^{M-1} \big[ \sin\left(\frac{\pi m}{M}\right) \BDmT \Bq_{nm}^{(t)} + \cos\left(\frac{\pi m}{M}\right) \Bq_{nm}^{(t)} \BDn \big] + \frac{\mu_3}{2} \big(\Bu - c_n^{(t+1)} \big)^{.2} \Big]\bigg] }
         {\displaystyle \sum_{i=1}^N \exp\bigg[ - \frac{1}{\xi} \Big[- \sum_{m=0}^{M-1} \big[ \sin\left(\frac{\pi m}{M}\right) \BDmT \Bq_{im}^{(t)} + \cos\left(\frac{\pi m}{M}\right) \Bq_{im}^{(t)} \BDn \big] + \frac{\mu_3}{2} \big(\Bu - c_i^{(t+1)} \big)^{.2} \Big]\bigg]  }
 \,,~ n = 1, \ldots, N
 \\
 \Bq_{nm}^{(t+1)} &= \frac{\displaystyle \Bq_{nm}^{(t)} + \tau \big[ \sin\left(\frac{\pi m}{M}\right) \BDm \Bp_n^{(t+1)} + \cos\left(\frac{\pi m}{M}\right) \Bp_n^{(t+1)} \BDnT \big] }
                          {\displaystyle 1 + \tau \Big[ \sum_{m=0}^{M-1} \big[ \sin\left(\frac{\pi m}{M}\right) \BDm \Bp_n^{(t+1)} + \cos\left(\frac{\pi m}{M}\right) \Bp_n^{(t+1)} \BDnT \big]^{.2} \Big]^{.\frac{1}{2}} }
 \,,~~ n = 1, \ldots, N \,,~ m = 0, \ldots, M-1
\end{align*}
}

\end{algorithmic}
\end{algorithm}


\begin{figure}[!ht]
\begin{center}  

 \includegraphics[height=0.58\textheight]{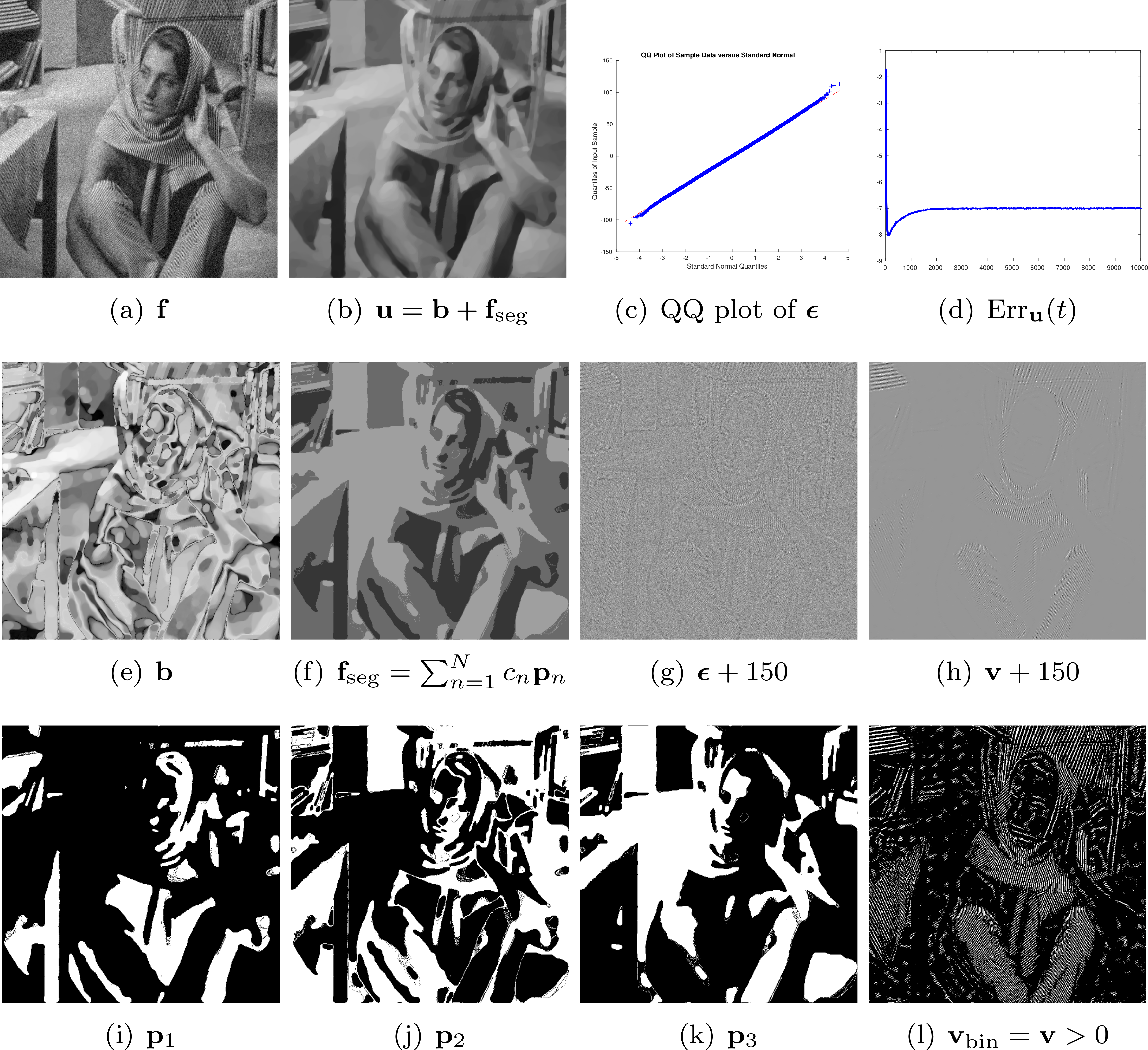}
 
 \caption{The bilevel SHT model applied to the noise image in Figure 6 with parameters
          $\sigma = 20 \,, L = S = 9 \,, c_{\mu_1} = c_{\mu_2} = 0.03 \,, 
           \theta = 0.9 \,, c_1 = 1, c_2 = 1.3 \,, \beta_4 = 0.04 \,, 
           \beta_3 = \frac{\theta}{1-\theta}\beta_4 \,, \beta_1 = c_{\beta_{1}} \beta_4 \,, \beta_2 = c_{\beta_{2}} \beta_3 \,, 
           \nu = 16 \,, M = 2 \,, N = 3 \,, \xi = 0.001 \,, \tau = 0.1 \,, \beta_5 = 100 \,, T_1 = T_2 = 100. 
          $
         } 
\label{fig:Barbara:BilevelModel}
\end{center}
\end{figure}

\begin{figure}[!ht]
\begin{center}  

 \includegraphics[height=0.6\textheight]{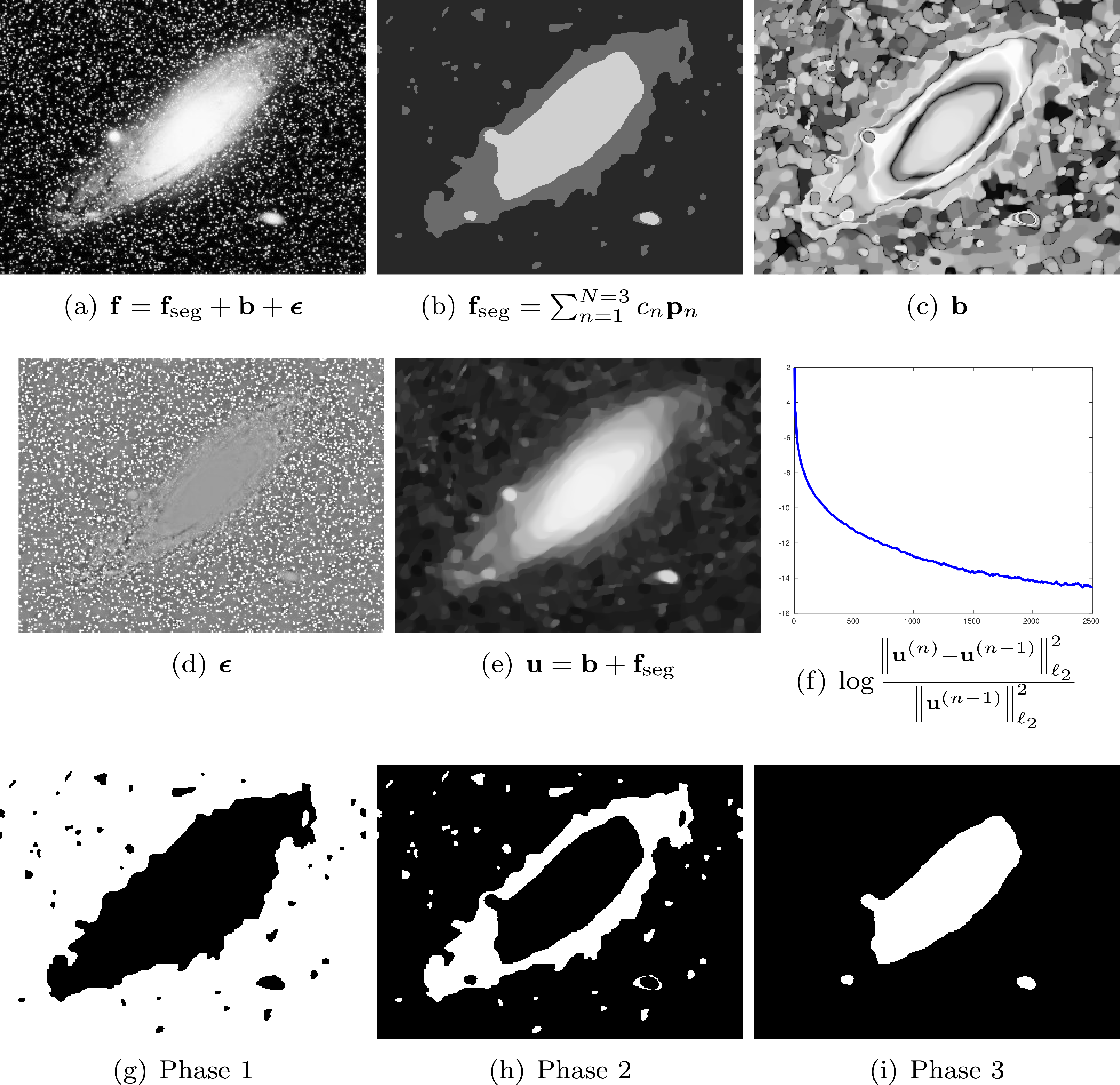}
 
 \caption{The bilevel SHT model applied to the galaxy image with no additional noise added.  The background stars
          represent small-scale texture which can be treated as noise and measured by $\norm{ \cC\{ \Beps \} }_{\ell_\infty}$.
          The constant values are $\vec{c} = [40.1 \,, 106.81 \,,  208.21]$ and $\text{MSE} = 3.8 10^{-7}$.
          Parameters were chosen as in Figure 7 with the exception of $\nu = 40 \,, T_1 = T_2 = 50.$ 
         } 
\label{fig:galaxy:BilevelModel:2}
\end{center}
\end{figure}

\section{Comparison with Alternative Approaches}
\label{sec:comparisons}
We now apply our approach to several images in order to demonstrate and compare the performance with alternative approaches.  The proficiency of and some properties of our models were demonstrated in Figures 
\ref{fig:Barbara:CombinedModel:NoNoise}, \ref{fig:Barbara:CombinedModel:Noise} and \ref{fig:Barbara:BilevelModel}. 
Here we focus on more subtle properties and compare our approach with existing methods.

\setcounter{subfigure}{0}
\begin{figure}[!ht]
\begin{center}

 \includegraphics[width=1\textwidth]{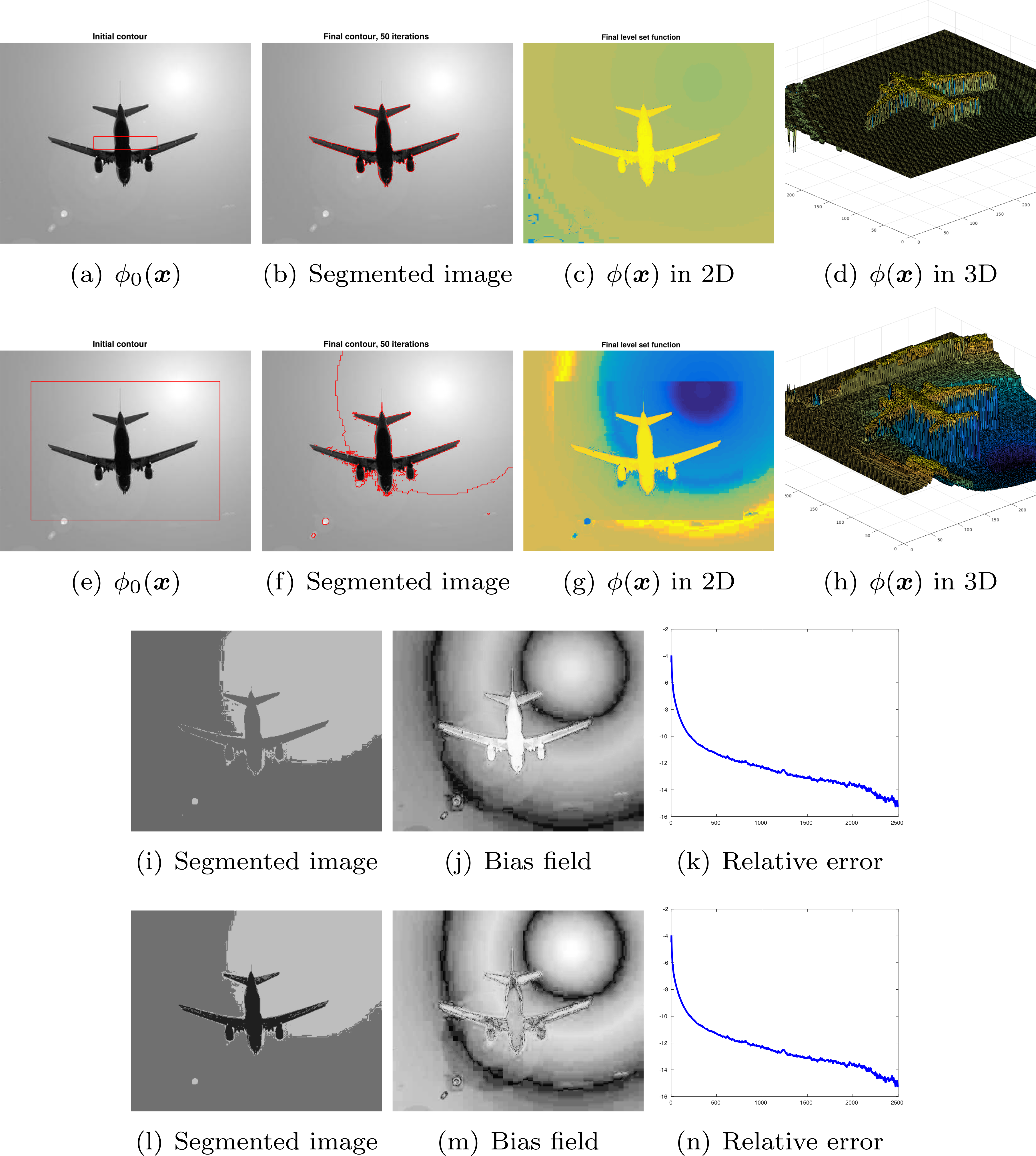}
 
 \caption{The Chan-Vese model (rows one and two) produces different segmentations for different initializations of the level set function $\phi(\Bx)$.  The Chan-Vese model parameters were chosen as $\text{Iteration} = 50 \,, \mu = \lambda_1 = \lambda_ 2 = \epsilon = 1 \,, \text{timestep} = 0.1 \,, v = 0$.  The third and fourth rows show reconstructed images from the 2 and 3-phase bilevel SHT model with $\Bu_\text{re} = \Bb + \sum_{n=1}^{N} c_n \Bp_n$.  The MSE for $\Bu_\text{re}$ in the 3-phase bilevel SHT model is $3.15 10^{-6}$ and the relative error in Subfigures (k) and (n) is measured by (\ref{eq:relativeError:u}).  The bilevel model parameters were selected as in Figure 7 with $\nu = 0$ and $T_1 = T_2 = 50$.
%
         } 
         \label{fig:comparison:ChanVese_Bilevel}
\end{center}
\end{figure}

Figure \ref{fig:comparison:ChanVese_Bilevel} depicts a homogeneous image of an airplane where we compare our bilevel SHT model to the classic Chan-Vese model \cite{ChanVese2001}.  The Chan-Vese model is applied in the first two rows and note that for different initial conditions -- Subfigures (a) and (e) -- that the model produces very different segmentations.  However, in our bilevel SHT model, we solve a convex minimization and as a result, produce a nearly unique result.  The bilevel SHT model is applied in rows three and four with 2 and 3 phases respectively.  Note that the 3-phase model is able to segment the sun, sky, and airplane whereas the 2-phase model combines the sky and airplane.   Subfigures (k) and (n) show convergence on the log scale, thus implying very fast convergence on the linear scale.

\setcounter{subfigure}{0}
\begin{figure}[!ht]
\begin{center}

 \includegraphics[width=1\textwidth]{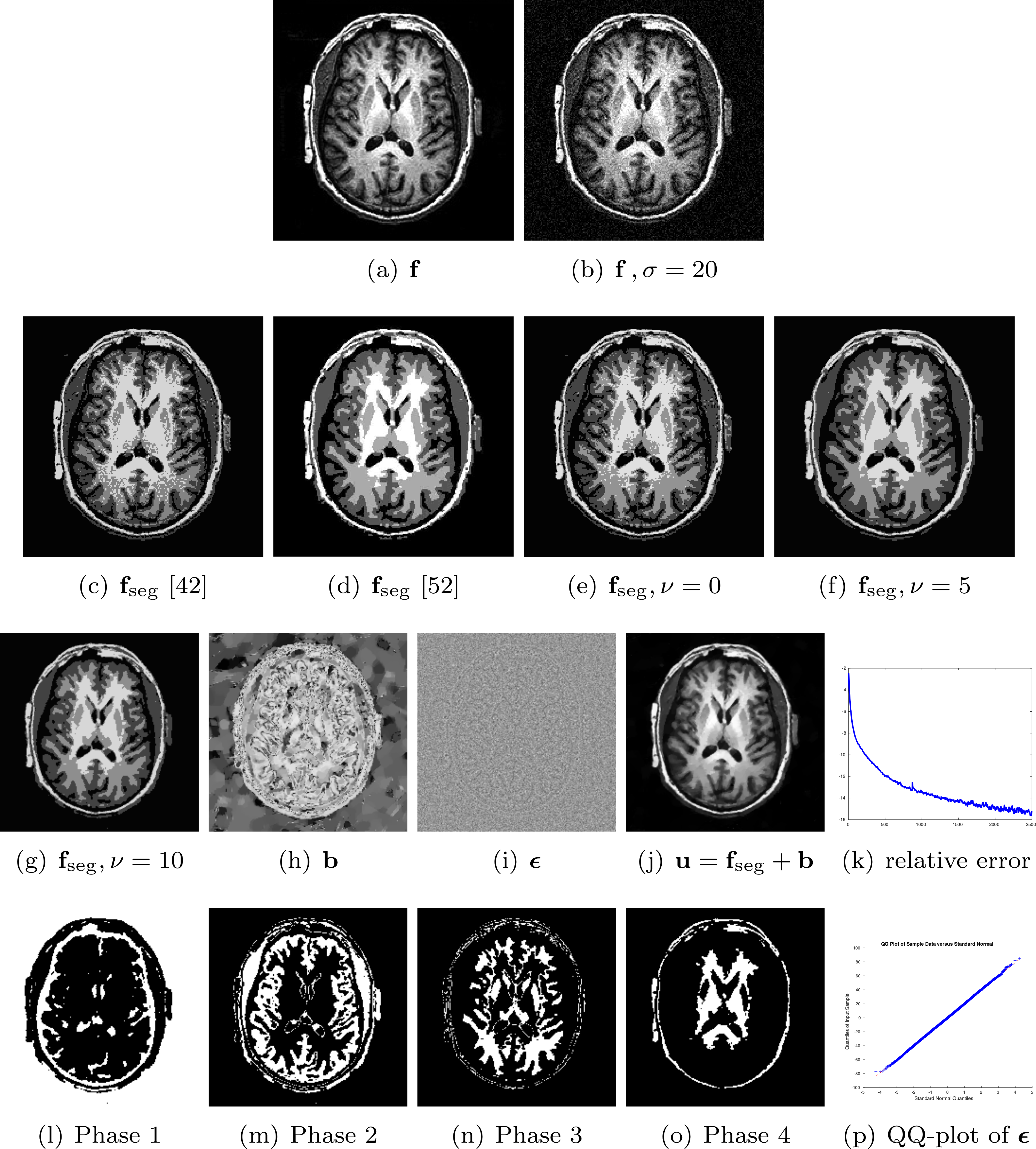}

 \caption{A comparison with the methods in \cite{BaeYuanTai2010} and \cite{BrownChanBresson2009}. 
          The segmented image in Subfigure (d) is taken from \cite{BrownChanBresson2009}.  The results in row two were based on the segmentation of Subfigure (a); those in rows three and four were applied to the noisy version of the image in Subfigure (b).  The bilevel model parameters are the same as in Figure 7 with $\nu = 0, 5, \text{ and } 10$ as noted in Subfigures (e), (f), and (g), respectively.  The relative error of $\Bu$ is shown in Subfigures (k) and~ (p).} 
 \label{fig:comparison:BaeYuanTai2010_Bilevel:brain}
\end{center}
\end{figure}

Figure \ref{fig:comparison:BaeYuanTai2010_Bilevel:brain} contains (entirely) homogeneous images of a brain and were analyzed in \cite{BrownChanBresson2012}.  Subfigure (a) contains the original image with no noise added and in Subfigure (b) we add i.i.d.\ Gaussian noise with standard deviation 20.  The second row (Subfigures (c) - (f)) show the resulting segmented images following the procedures in \cite{BaeYuanTai2010, BrownChanBresson2009}, and the bilevel SHT method with $\nu = 0$ and $\nu = 5$, respectively.  We note that Subfigure (d) was taken directly from \cite{BrownChanBresson2009} and Subfigure (c) was programmed by hand.  Note that in Subfigure (e) with $\nu = 0$, some small-scale residual still remains but when we increase the threshold to $\nu = 5$ in Subfigure (f), this residual is removed resulting in a smoother segmented image.  Thus, our procedure compares favorably even to 
other methods that apply only to homogeneous images.  Further note that when noise is added (Subfigure (b)), our procedure is able to not only filter out the additional noise from the segmented images -- see Subfigures (l) - (o) -- but also produces a decomposition with well-separated meaningful components.  Note also that by examining Subfigure (p),
we see that with $\nu = 10$, almost all of the resulting noise was that which was added (i.e. very little of the information from the original image was classified as noise).  As in Figure \ref{fig:comparison:BaeYuanTai2010_Bilevel:brain}, Subfigure (k) shows convergence on the log scale implying very fast convergence on the linear scale.

\setcounter{subfigure}{0}
\begin{figure}[!ht]
\begin{center}

 \includegraphics[width=1\textwidth]{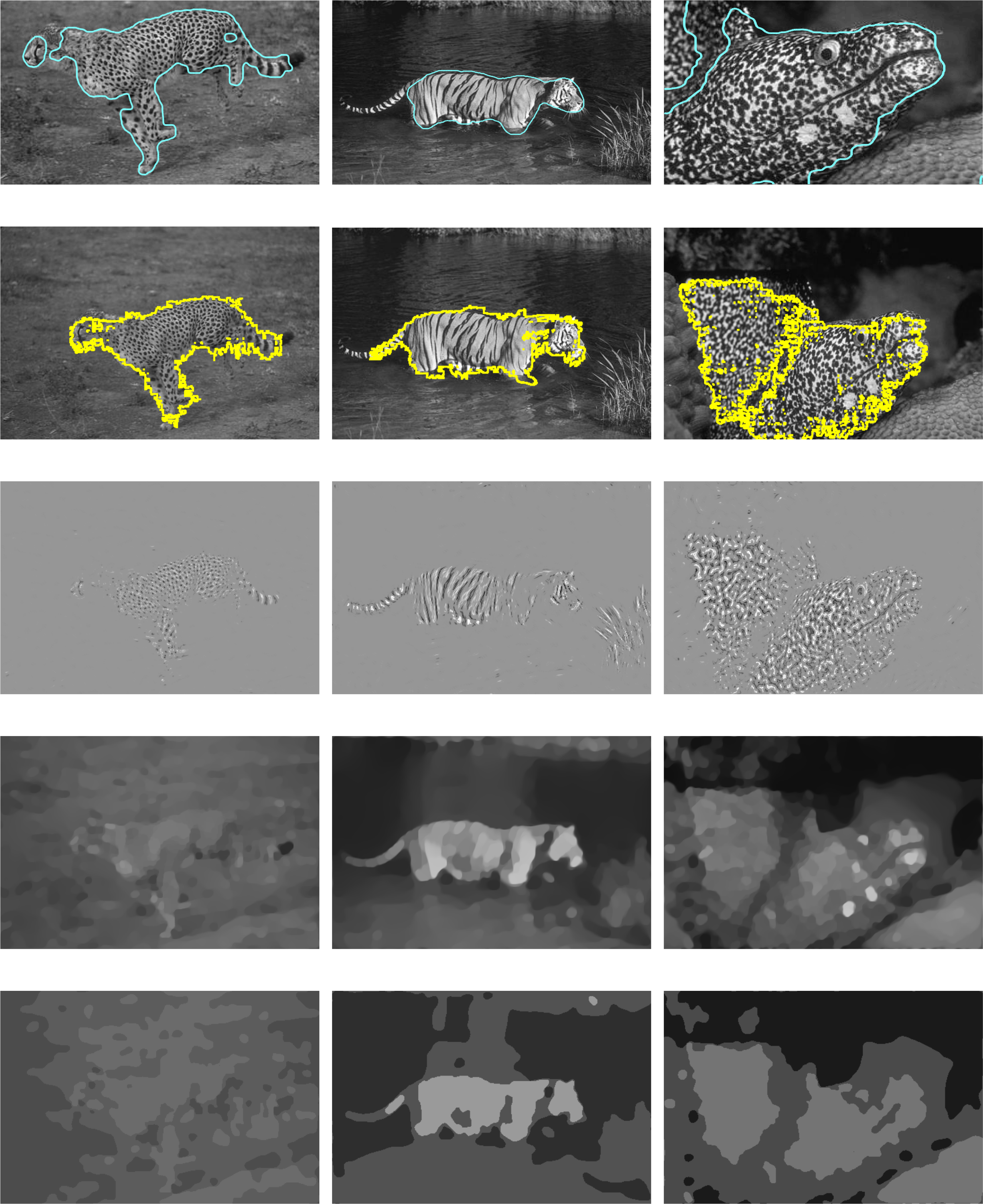}
 
 \caption{Animal images with textural regions of interest.  Row one shows the original images and segmentations taken from \cite{NiBressonChanEsedoglu2009}.  Rows 2 - 5 show the various components of our bilevel SHT model; see main text for details.  The bilevel model parameters are the same as in Figure 7 with the following exceptions:  $T_2 = 1$ for each of the three images; $L = 40, 40, 12$ and $\nu = 30, 25, 15$ for the images in columns 1, 2, and 3, respectively.
         } 
         \label{fig:comparison:texture}
\end{center}
\end{figure}

Finally, we move on to consider images that contain only a textural region of interest.  The images in Figure \ref{fig:comparison:texture} depict various animals each with well-defined textural markings; the first row of images are taken directly from \cite{NiBressonChanEsedoglu2009}.  We begin by noting that many methods already exist to define a region of texture; see for example the methodology in \cite{NiBressonChanEsedoglu2009}.  
Our SHT procedures were not designed for this goal, though extracting such a region is possible with our bilevel SHT model.  Row three of Figure \ref{fig:comparison:texture} shows the texture component of the bilevel SHT decomposition.  This texture component was then binarized and a morphological operator applied to obtain the textural boundaries shown in row 2.  Rows four and five show the piecewise-smooth and piecewise-constant bilevel SHT components, respectively.  Note that our bilevel model, though not designed for this purpose, still does an admirable job of capturing the textural boundary.  One advantage to our approach is that instead of only defining this boundary, our procedure also allows one to separate the texture inside from the remainder of the image.

\section{Conclusion}
\label{sec:conclusion}
This work provides algorithms to simultaneously decompose and segment images containing 
regions of both texture and homogeneity.  This can be seen as an extension of the Mumford and Shah model to a much larger class of natural images.  Two approaches are presented corresponding to 
the two alternative solutions to the $\text{G}_S$-norm for texture $\Bv$; the multiphase SHT approach 
based on the $\text{G}_S$-norm solution provided by Aujol and Chambolle \cite{AujolChambolle2005}
and bilevel SHT approach based on the solution of Vese and Osher \cite{VeseOsher2003}.  
In practice we find that the bilevel SHT algorithm is better able to discriminate between 
the homogeneous and textural regions and thus we focus on this approach in Section 5 and recommend it in practical applications.  The likely reason for the superior performance of the bilevel model is that the Vese and Osher \cite{VeseOsher2003} approach to solving the $\text{G}_S$-norm utilized in the bilevel SHT model approximates $\norm{ \vec{\Bg} }_{\ell_\infty}$ with $\norm{\vec{\Bg}}_{\ell_1}$.  This enhances the sparsity of $\vec{\Bg}$ and though the original image is typically not sparse, it is sparse in some transform domain which is usually measured by $\ell_0$-norm (or its relaxed $\ell_1$-norm) in function space.  One shortcoming of our models is the large number of required parameters and we hope to reduce the size of the parameter set as well as to analyze the convergence of the proposed minimization in future work.

\section*{Acknowledgements}
The authors thank Professors Len Stefanski, David Banks and Ingrid Daubechies for their helpful comments.  This material was based upon work partially supported by the National Science Foundation under Grant DMS-1127914 to the Statistical and Applied Mathematical Sciences Institute. Any opinions, findings, and conclusions or recommendations expressed in this material are those of the author(s) and do not necessarily reflect the views of the National Science Foundation. 


\begin{thebibliography}{10}

\bibitem{DigitalImageProcessing2002}
R.C. Gonzalez and R.E. Woods.
\newblock {\em Digital Image Processing}.
\newblock Prentice Hall, Upper Saddle River, NJ, USA, 2002.

\bibitem{Szeliski2011}
R.~Szeliski.
\newblock {\em Computer Vision: Algorithms and Applications}.
\newblock Springer, London, United Kingdom, 2011.

\bibitem{IntroductiontoBiomedicalImaging2003}
A.G. Webb.
\newblock {\em Introduction to Biomedical Imaging}.
\newblock Wiley-IEEE Press, New York, U.S.A., 2003.

\bibitem{AstronomicalImageandDataAnalysis2006}
J.L. Starck and F.~Murtagh, editors.
\newblock {\em Astronomical Image and Data Analysis}.
\newblock Springer, New York, NY, USA, 2006.

\bibitem{IntroductiontoBiometrics2011}
A.~Jain, A.A. Ross, and K.~Nandakumar.
\newblock {\em Introduction to Biometrics}.
\newblock Springer, New York, NY, USA, 2011.

\bibitem{ThaiHuckemannGottschlich2016}
D.H. Thai, S.~Huckemann, and C.~Gottschlich.
\newblock Filter design and performance evaluation for fingerprint image
  segmentation.
\newblock {\em PLoS ONE}, 11(5):e0154160, May 2016.

\bibitem{ThaiGottschlich2016G3PD}
D.H. Thai and C.~Gottschlich.
\newblock Global variational method for fingerprint segmentation by three-part
  decomposition.
\newblock {\em IET Biometrics}, 5(2):120--130, June 2016.

\bibitem{Otsu1979}
N.~Otsu.
\newblock A threshold selection method from gray-level histograms.
\newblock {\em IEEE Transactions on Systems, Man and Cybernetics}, 9(1):62--66,
  January 1979.

\bibitem{SahooWilkinsYeager1997}
P.~Sahoo, C.~Wilkins, and J.~Yeager.
\newblock Threshold selection using {R}enyi's entropy.
\newblock {\em Pattern Recognition}, 30(1):71--84, January 1997.

\bibitem{AlbuquerqueEsquefMelloAlbuquerque2004}
M.P.d. Albuquerque, I.A. Esquef, A.R.G. Mello, and M.P.d. Albuquerque.
\newblock Image thresholding using {T}sallis entropy.
\newblock {\em Pattern Recognition Letters}, 25(9):1059--1065, July 2004.

\bibitem{ChanVese2001}
T.F. Chan and L.A. Vese.
\newblock Active contours without edges.
\newblock {\em {IEEE} Transactions on Image Processing}, 10(2):266--277,
  February 2001.

\bibitem{BressonEsedogluVandergheynstThiranOsher2007}
X.~Bresson, S.~Esedoglu, P.~Vandergheynst, J.P. Thiran, and S.~Osher.
\newblock Fast global minimization of the active contour/snake model.
\newblock {\em Journal of Mathematical Imaging and Vision}, 28(2):151--167,
  June 2007.

\bibitem{ChanEsedogluNikolova2012}
T.F. Chan, S.~Esedoglu, and M.~Nikolova.
\newblock Algorithms for finding global minimizers of image segmentation and
  denoising models.
\newblock {\em SIAM J. Appl. Math.}, 66(5):1632--1648, February 2012.

\bibitem{LieLysakerTai2006}
J.~Lie, M.~Lysaker, and X.C. Tai.
\newblock A binary level set model and some applications to {M}umford-{S}hah
  image segmentation.
\newblock {\em {IEEE} Transactions on Image Processing}, 15(5):1171--1181, May
  2006.

\bibitem{KassWitkinTerzopoulos1988}
M.~Kass, A.~Witkin, and D.~Terzopoulos.
\newblock Snakes: Active contour models.
\newblock {\em International Journal of Computer Vision}, 1(4):321--331,
  January 1988.

\bibitem{MumfordShah1989}
D.~Mumford and J.~Shah.
\newblock Optimal approximations by piecewise smooth functions and associated
  variational problems.
\newblock {\em Communications on Pure and Applied Mathematics}, 42(5):577--685,
  July 1989.

\bibitem{Potts1952}
R.B. Potts.
\newblock Some generalized order-disorder transformations.
\newblock {\em In Proceedings of the Cambridge Philosophical Society},
  48:106--109, 1952.

\bibitem{RudinOsherFatemi1992}
L.~Rudin, S.~Osher, and E.~Fatemi.
\newblock Nonlinear total variation based noise removal algorithms.
\newblock {\em Physica D}, 60(1-4):259--268, November 1992.

\bibitem{AujolGilboaChanOsher2006}
J.-F. Aujol, G.~Gilboa, T.~Chan, and S.~Osher.
\newblock Structure-texture image decomposition - modeling, algorithms, and
  parameter selection.
\newblock {\em International Journal of Computer Vision}, 67(1):111--136, April
  2006.

\bibitem{AujolAubertFeraudChambolle2005}
J.F. Aujol, G.~Aubert, L.B. Feraud, and A.~Chambolle.
\newblock Image decomposition into a bounded variation component and an
  oscillating component.
\newblock {\em Journal of Mathematical Imaging and Vision}, 22(1):71--88,
  January 2005.

\bibitem{AujolGilboa2006}
J.F. Aujol and G.~Gilboa.
\newblock Constrained and {SNR}-based solutions for {TV}-{H}ilbert space image
  denoising.
\newblock {\em Journal of Mathematical Imaging and Vision}, 26(1-2):217--237,
  November 2006.

\bibitem{BuadesLeMorelVese2010}
A.~Buades, T.M. Le, J.-M. Morel, and L.A. Vese.
\newblock Fast cartoon + texture image filters.
\newblock {\em {IEEE} Transactions on Image Processing}, 19(8):1978--1986,
  August 2010.

\bibitem{VeseOsher2003}
L.A. Vese and S.~Osher.
\newblock Modeling textures with total variation minimization and oscillatory
  patterns in image processing.
\newblock {\em Journal of Scientific Computing}, 19(1-3):553--572, December
  2003.

\bibitem{AubertVese1997}
G.~Aubert and L.~Vese.
\newblock A variational method in image recovery.
\newblock {\em SIAM J. Numer. Anal}, 34(5):1948--1979, October 1997.

\bibitem{ChanMarquinaMulet2000}
T.~Chan, A.~Marquina, and P.~Mulet.
\newblock High-order total variation-based image restoration.
\newblock {\em SIAM Journal on Scientific Computing}, 22(2):503--516, July
  2000.

\bibitem{LysakerLundervoldTai}
M.~Lysaker, A.~Lundervold, and X.C. Tai.
\newblock Noise removal using fourth-order partial differential equation with
  applications to medical magnetic resonance images in space and time.
\newblock {\em {IEEE} Transactions on Image Processing}, 12(12):1579--1590,
  December 2003.

\bibitem{RahmanTaiOsher2007}
T.~Rahman, X.C. Tai, and S.~Osher.
\newblock A {TV}-{S}tokes denoising algorithm.
\newblock {\em Lecture Notes in Computer Science}, 4485:473--483, June 2007.

\bibitem{HahnWuTai2012}
J.~Hahn, C.~Wu, and X.C. Tai.
\newblock Augmented {L}agrangian method for generalized {TV}-{S}tokes model.
\newblock {\em Journal of Scientific Computing}, 50(2):235--264, February 2012.

\bibitem{ZhuChan2012}
W.Zhu and T.~Chan.
\newblock Image denoising using mean curvature of image surface.
\newblock {\em SIAM Journal on Imaging Sciences}, 5(1):1--32, January 2012.

\bibitem{TaiHahnChung2011}
X.C. Tai, J.~Hahn, and G.J. Chung.
\newblock A fast algorithm for {E}uler's elastica model using augmented
  {L}agrangian method.
\newblock {\em SIAM Journal on Imaging Sciences}, 4(1):313--344, February 2011.

\bibitem{ZhuTaiChan2013}
W.~Zhu, X.C. Tai, and T.~Chan.
\newblock Image segmentation using {E}uler's elastica as the regularization.
\newblock {\em Journal of Scientific Computing}, 57(2):414--438, April 2013.

\bibitem{PapafitsorosSchoenlieb2014}
K.~Papafitsoros and C.B. Sch\"onlieb.
\newblock A combined first and second order variational approach for image
  reconstruction.
\newblock {\em J. Math. Imaging Vis.}, 48(2):308--338, 2014.

\bibitem{CalatroniDueringSchoenlieb2014}
L.~Calatroni, B.~D\"uring, and C.B. Sch\"onlieb.
\newblock {ADI} splitting schemes for a fourth-order nonlinear partial
  differential equation from image processing.
\newblock {\em DCDS Series A}, 34(3):931--957, March 2014.

\bibitem{Chambolle2004}
A.~Chambolle.
\newblock An algorithm for total variation minimization and applications.
\newblock {\em Journal of Mathematical Imaging and Vision}, 20(1-2):89--97,
  January 2004.

\bibitem{GoldsteinOsher2009}
T.~Goldstein and S.~Osher.
\newblock The split {Bregman} method for {L1}-regularized problems.
\newblock {\em SIAM Journal on Imaging Sciences}, 2(2):323--343, April 2009.

\bibitem{DaubechiesDefriseMol2004}
I.~Daubechies, M.~Defrise, and C.~D. Mol.
\newblock An iterative thresholding algorithm for linear inverse problems with
  a sparsity constraint.
\newblock {\em Communications on Pure and Applied Mathematics},
  57(11):1413--1457, August 2004.

\bibitem{BeckTeboulle2009}
A.~Beck and M.~Teboulle.
\newblock A fast iterative shrinkage-thresholding algorithm for linear inverse
  problems.
\newblock {\em SIAM Journal on Imaging Sciences}, 2(1):183--202, January 2009.

\bibitem{DiasFigueiredo2007}
J.B. Dias and M.~Figueiredo.
\newblock A new twist: Two-step iterative shrinkage/thresholding algorithms for
  image restoration.
\newblock {\em {IEEE} Transactions on Image Processing}, 16(12):2992--3004,
  December 2007.

\bibitem{WuTai2010}
C.~Wu and X.~C. Tai.
\newblock Augmented {L}agrangian method, dual methods, and split {B}regman
  iteration for {ROF}, vectorial {TV}, and higher order methods.
\newblock {\em SIAM Journal on Imaging Sciences}, 3(3):300--339, July 2010.

\bibitem{BrownChanBresson2010}
E.S. Brown, T.F. Chan, and X.~Bresson.
\newblock A convex relaxation method for a class of vector-valued minimization
  problems with applications to {M}umford-{S}hah segmentation.
\newblock {\em UCLA cam report}, 2010.

\bibitem{BrownChanBresson2012}
E.S. Brown, T.F. Chan, and X.~Bresson.
\newblock Completely convex formulation of the {C}han-{V}ese image segmentation
  model.
\newblock {\em International Journal of Computer Vision}, 98(1):103--121, May
  2012.

\bibitem{BaeYuanTai2010}
E.~Bae, J.~Yuan, and X.C. Tai.
\newblock Global minimization for continuous multiphase partitioning problems
  using a dual approach.
\newblock {\em International Journal of Computer Vision}, 92(1):112--129, March
  2010.

\bibitem{GuWangTai2012}
L.L.~Wang Y.~Gu and X.C. Tai.
\newblock A direct approach toward global minimization for multiphase labeling
  and segmentation problems.
\newblock {\em {IEEE} Transactions on Image Processing}, 21(5):2399--2411, May
  2012.

\bibitem{BaeLellmannTai2005}
E.~Bae, J.~Lellmann, and X.C. Tai.
\newblock {\em Convex Relaxations for a Generalized {C}han-{V}ese Model}, pages
  223--236.
\newblock Springer, Berlin, Germany, 2005.

\bibitem{GuWangXiongChengHuangZhou2013}
Y.~Gu, L.L. Wang, W.~Xiong, J.~Cheng, W.~Huang, and J.~Zhou.
\newblock Efficient and robust image segmentation with a new piecewise-smooth
  decomposition model.
\newblock In {\em Proc. Int. Conf. IEEE ICIP}, pages 2718--2722, Melbourne,
  Australia, September 2013.

\bibitem{GuXiongWangChengHuangZhou2014}
Y.~Gu, W.~Xiong, L.L. Wang, J.~Cheng, W.~Huang, and J.~Zhou.
\newblock A new approach for multiphase piecewise smooth image segmentation.
\newblock In {\em Proc. Int. Conf. IEEE ICIP}, pages 4417--4421, Paris, France,
  October 2014.

\bibitem{SagivSochenZeevi2006}
C.~Sagiv, N.A. Sochen, and Y.Y. Zeevi.
\newblock Integrated active contours for texture segmentation.
\newblock {\em {IEEE} Transactions on Image Processing}, 15(6):1633--1646, June
  2006.

\bibitem{HouhouThiranBresson2009}
N.~Houhou, J.P. Thiran, and X.~Bresson.
\newblock Fast texture segmentation based on semi-local region descriptor and
  active contour.
\newblock {\em Numer. Math. Theor. Meth. Appl.}, 2(4):445--468, November 2009.

\bibitem{NiBressonChanEsedoglu2009}
K.~Ni, X.~Bresson, T.~Chan, and S.~Esedoglu.
\newblock Local histogram based segmentation using the wasserstein distance.
\newblock {\em Int J Comput Vis}, 84(1):97--111, April 2009.

\bibitem{Unser1995}
M.~Unser.
\newblock Texture classification and segmentation using wavelet frames.
\newblock {\em {IEEE} Transactions on Image Processing}, 4(11):1549--1560,
  November 1995.

\bibitem{ChanSandbergVese2000}
T.F. Chan, B.Y. Sandberg, and L.A. Vese.
\newblock Active contours without edges for vector-valued images.
\newblock {\em Journal of Visual Communication and Image Representation},
  11(2):130--141, June 2000.

\bibitem{BrownChanBresson2009}
E.S. Brown, T.F. Chan, and X.~Bresson.
\newblock Convex formulation and exact global solutions for multi-phase
  piecewise constant {M}umford-{S}hah image segmentation.
\newblock {\em UCLA cam report}, 2009.

\bibitem{LiuTaiHuangHuan2011}
J.~Liu, X.C. Tai, H.~Huang, and Z.~Huan.
\newblock A fast segmentation method based on constraint optimization and its
  applications: Intensity inhomogeneity and texture segmentation.
\newblock {\em Pattern Recognition}, 44(9):2093--2108, September 2011.

\bibitem{AujolChambolle2005}
J.-F. Aujol and A.~Chambolle.
\newblock Dual norms and image decomposition models.
\newblock {\em International Journal of Computer Vision}, 63(1):85--104, June
  2005.

\bibitem{ThaiGottschlich2016DG3PD}
D.H. Thai and C.~Gottschlich.
\newblock Directional global three-part image decomposition.
\newblock {\em EURASIP Journal on Image and Video Processing}, 2016(12):1--20,
  March 2016.

\bibitem{Gilles2012}
J.~Gilles.
\newblock Multiscale texture separation.
\newblock {\em Multiscale Model. Simul.}, 10(4):1409--1427, December 2012.

\bibitem{CandesDonoho2004}
E.~Cand{\`e}s and D.~Donoho.
\newblock New tight frames of curvelets and optimal representations of objects
  with piecewise singularities.
\newblock {\em Communications on Pure and Applied Mathematics}, 57(2):219--266,
  February 2004.

\bibitem{CandesDemanetDonohoYing2006}
E.~Cand\`{e}s, L.~Demanet, D.~Donoho, and L.~Ying.
\newblock Fast discrete curvelet transforms.
\newblock {\em Multiscale Model. Simul.}, 5(3):861--899, September 2006.

\bibitem{StarckDonohoCandes2003}
Starck, D.L. Donoho, and E.J. Cand{\`e}s.
\newblock Astronomical image representation by the curvelet transform.
\newblock {\em Astron. Astrophys.}, 398(2):785--800, February 2003.

\bibitem{MaPlonka2010}
J.~Ma and G.~Plonka.
\newblock The curvelet transform.
\newblock {\em IEEE Signal Processing Magazin}, 27(2):118--133, March 2010.

\bibitem{ShearletsBook}
G.~Kutyniok and D.~Labate, editors.
\newblock {\em Shearlets. Multiscale Analysis for Multivariate Data}.
\newblock Birkh{\"a}user, Boston, MA, USA, 2012.

\bibitem{DoVetterli2005}
M.N. Do and M.~Vetterli.
\newblock The contourlet transform: An efficient directional multiresolution
  image representation.
\newblock {\em {IEEE} Transactions on Image Processing}, 14(12):2091--2106,
  December 2005.

\bibitem{UnserVandeville2010}
M.~Unser and D.~Van~De Ville.
\newblock Wavelet steerability and the higher-order {R}iesz transform.
\newblock {\em {IEEE} Transactions on Image Processing}, 19(3):636--652, March
  2010.

\bibitem{GillesTranOsher2014}
J.~Gilles, G.~Tran, and S.~Osher.
\newblock 2{D} {E}mpirical transforms. {W}avelets, ridgelets, and curvelet
  revisited.
\newblock {\em SIAM J. Imaging Sci.}, 7(1):157--186, January 2014.

\bibitem{Thai2015PhD}
D.H. Thai.
\newblock {\em Fourier and Variational Based Approaches for Fingerprint
  Segmentation}.
\newblock PhD thesis, University of Goettingen, Goettingen, Germany, January
  2015.

\bibitem{Meyer2001}
Y.~Meyer.
\newblock {\em Oscillating Patterns in Image Processing and Nonlinear Evolution
  Equations: The Fifteenth Dean Jacqueline B. Lewis Memorial Lectures}.
\newblock American Mathematical Society, Boston, MA, USA, 2001.

\bibitem{EkelandTeman1999}
I.~Ekeland and R.~T/'eman.
\newblock {\em Convex analysis and variational problems}.
\newblock Society for Industrial and Applied Mathematics, Philadelphia, PA,
  USA, 1999.

\bibitem{Rockafellar1970}
R.~T. Rockafellar.
\newblock {\em Convex analysis}.
\newblock Princeton Mathematical Series, No. 28. Princeton University Press,
  Princeton, N.J., 1970.

\end{thebibliography}

\bibliographystyle{unsrt}


\section{Appendix} \label{sec:appendix}

\begin{prop} \label{prop:2phasePiecewiseConst}
 The numerical solution of the two-phase piecewise constant and texture image segmentation (\ref{eq:twophaseDG3PDtexture:minimization:1}) 
 as described in Algorithm 5.
\end{prop}

\begin{proof}
This proof follows many of the techniques in \cite{ThaiGottschlich2016DG3PD}.  The minimization in (\ref{eq:twophaseDG3PDtexture:minimization:1}) can be written as
\begin{align} \label{eq:twophaseDG3PDtexture:minimization:1:appendix} 
 &\min_{(\Bp, \Beps, c_1, c_2) \in X^{2} \times \mathbb R^2} \Bigg\{ 
 \sum_{l=0}^{L-1} \norm{ \sin\left(\frac{\pi l}{L}\right) \BDm \Bp + \cos\left(\frac{\pi l}{L}\right) \Bp \BDnT }_{\ell_1}
 + \mu_1 \sum_{s=0}^{S-1}\norm{\Bg_s}_{\ell_1} + \mu_2 \norm{\Bv}_{\ell_1}
 \notag
 \\ &\text{s.t.}~  
 \norm{\cC\{\Beps\}}_{\ell_\infty} \leq \nu \,,
 \Bv = \sum_{s=0}^{S-1} \Big[ \sin\left(\frac{\pi s}{S}\right) \BDm \Bg_s + \cos\left(\frac{\pi s}{S}\right)\Bg_s \BDnT \Big] \,,  \notag
 \\&
 \Bf = (c_1 + \Bv + \Beps) \cdot^\times \Bp + (c_2 + \Bv + \Beps) \cdot^\times (1 - \Bp) \,,
 p[\Bk] \in \{0, 1\}, \forall \Bk \in \Omega
 \Bigg\} \,.  
\end{align} 
We define a convex set (by relaxing the binary set) and its indicator function as
\begin{align*}
 \mathcal D = \Big\{ \Bp \in X ~:~ p[\Bk] \in [0, 1] \,, \forall \Bk \in \Omega \Big\} 
 ~~\text{and}~~
 H^*(\Bp) = \begin{cases} 0 \,, & \Bp \in \mathcal D \\ +\infty \,, & \Bp \notin \mathcal D \end{cases} \,.
\end{align*}
We introduce two new variables
\begin{align*}
  \Br_b &= \sin\left(\frac{\pi b}{L}\right) \BDm \Bp + \cos\left(\frac{\pi b}{L}\right) \Bp \BDnT \,,~~ b = 0, \ldots, L-1
  \\
  &\hspace{10mm} \text{and}
  \\
  &\hspace{-3.5mm} \Bw_a = \Bg_a \,,~ a = 0, \ldots, S-1 \,.
\end{align*}
Given $\vec{\Bg} = \big[ \Bg_s \big]_{s=0}^{S-1} \,, \vec{\Br} = \big[ \Br_l \big]_{l=0}^{L-1} \,, \vec{\Bw} = \big[ \Bw_s \big]_{s=0}^{S-1}$,
the augmented Lagrangian method of (\ref{eq:twophaseDG3PDtexture:minimization:1:appendix}) is
\begin{equation} \label{eq:twophaseDG3PDtexture:minimization:ALM}
 \min_{(c_1, c_2, \Bp, \Bv, \Beps, \vec{\Bg}, \vec{\Br}, \vec{\Bw}) \in \mathbb R^2 \times X^{3 + L + 2S}} 
 \cL(\cdot \,; \vec{\boldsymbol{\lambda}}_{\boldsymbol 1}, \vec{\boldsymbol{\lambda}}_{\boldsymbol 2}, \boldsymbol{\lambda_3}, \boldsymbol{\lambda_4})
\end{equation}
with
\begin{align*}
 &\cL(\cdot \,; \cdot) =  
 \sum_{l=0}^{L-1} \norm{ \Br_l }_{\ell_1} + \mu_1 \sum_{s=0}^{S-1}\norm{\Bw_s}_{\ell_1} + \mu_2 \norm{\Bv}_{\ell_1}
 + G^* \big( \frac{\Beps}{\nu} \big) + H^*(\Bp)
 \\&
 + \frac{\beta_1}{2} \sum_{l=0}^{L-1} \norm{ \Br_l - \sin\left(\frac{\pi l}{L}\right) \BDm \Bp - \cos\left(\frac{\pi l}{L}\right) \Bp \BDnT + \frac{\boldsymbol{\lambda}_{\boldsymbol 1 l}}{\beta_1} }^2_{\ell_2}
 + \frac{\beta_2}{2} \sum_{s=0}^{S-1} \norm{ \Bw_s - \Bg_s + \frac{\boldsymbol{\lambda}_{\boldsymbol 2 s}}{\beta_2} }^2_{\ell_2}
 \\&
 + \frac{\beta_3}{2} \norm{ \Bv - \sum_{s=0}^{S-1} \Big[ \sin\left(\frac{\pi s}{S}\right) \BDm \Bg_s + \cos\left(\frac{\pi s}{S}\right)\Bg_s \BDnT \Big] + \frac{\boldsymbol{\lambda}_{\boldsymbol 3}}{\beta_3} }^2_{\ell_2}
 \\& +
 \underbrace{
 \frac{\beta_4}{2} \Big\langle \big[ \Bf - c_1 - \Bv - \Beps + \frac{\boldsymbol{\lambda_4}}{\beta_4} \big]^{\cdot 2} \,, \Bp \Big\rangle_{\ell_2}
 + \frac{\beta_4}{2} \Big\langle \big[ \Bf - c_2 - \Bv - \Beps + \frac{\boldsymbol{\lambda_4}}{\beta_4} \big]^{\cdot 2} \,, 1 - \Bp \Big\rangle_{\ell_2}
 }_{ \approx \frac{\beta_4}{2} \norm{ \Bf - c_1 \Bp - c_2 (1 - \Bp) - \Bv - \Beps + \frac{\boldsymbol{\lambda_4}}{\beta_4} }^2_{\ell_2} \quad \text{(due to the relaxed version of the binary set)} } \,.
\end{align*}
We solve (\ref{eq:twophaseDG3PDtexture:minimization:ALM}) by ADMM
with the ordering subproblems and then update the Lagrange multipliers at each iteration as

\noindent {\bfseries The $c_1$-problem:} Fix $c_2 \,, \Bp \,, \Bv \,, \Beps \,, \vec{\Bg} \,, \vec{\Br} \,, \vec{\Bw}$ and solve
\begin{align*}
 \min_{c_1 \in \mathbb R} \Big\{ \cL(c_1) 
  = \frac{\beta_4}{2} \sum_{\Bk \in \Omega} \Big[ f[\Bk] - c_1 - v[\Bk] - \epsilon[\Bk] + \frac{\lambda_4[\Bk]}{\beta_4} \Big]^2 p[\Bk] \Big\rangle_{\ell_2}
 \Big\}.
\end{align*}
The Euler-Lagrange equation is
\begin{align*}
 0 &= \frac{ \partial \cL(c_1) }{\partial c_1} = -\beta_4 \sum_{\Bk \in \Omega} p[\Bk] \Big[ f[\Bk] - c_1 - v[\Bk] - \epsilon[\Bk] + \frac{\lambda_4[\Bk]}{\beta_4} \Big]
 \\ \Leftrightarrow~ 
 c_1 &= \frac{ \sum_{\Bk \in \Omega} \big[ f[\Bk] - v[\Bk] - \epsilon[\Bk] + \frac{\lambda_4[\Bk]}{\beta_4} \big] p[\Bk] }{ \sum_{\Bk \in \Omega} p[\Bk] } \,. 
\end{align*}

\noindent {\bfseries The $c_2$-problem:} Fix $c_1 \,, \Bp \,, \Bv \,, \Beps \,, \vec{\Bg} \,, \vec{\Br} \,, \vec{\Bw}$ and solve
\begin{align*}
 &\min_{c_2 \in \mathbb R} \Big\{ \cL(c_2) =
 \frac{\beta_4}{2} \sum_{\Bk \in \Omega} \Big[ f[\Bk] - c_2 - v[\Bk] - \epsilon[\Bk] + \frac{\lambda_4[\Bk]}{\beta_4} \Big]^2 \Big[ 1 - p[\Bk] \Big]
 \Big\}.
\end{align*}
The Euler-Lagrange equation is
\begin{align*}
 0 &= \frac{\partial \cL(c_2)}{\partial c_2} = -\beta_4 \sum_{\Bk \in \Omega} \Big[ f[\Bk] - c_2 - v[\Bk] - \epsilon[\Bk] \frac{\lambda_4[\Bk]}{\beta_4} \Big] \big[ 1 - p[\Bk] \big]
 \\ \Leftrightarrow~ 
 c_2 &= \frac{ \sum_{\Bk \in \Omega} \big[ f[\Bk] - v[\Bk] - \epsilon[\Bk] + \frac{\lambda_4[\Bk]}{\beta_4} \big] \big[ 1 - p[\Bk] \big] }{ \sum_{\Bk \in \Omega} (1 - p[\Bk]) } \,.
\end{align*}

\noindent {\bfseries The $\vec{\Br} = \big[ \Br_l \big]_{l=0}^{L-1}$-problem:} 
Fix $c_1 \,, c_2 \,, \Bp \,, \Bv \,, \Beps \,, \vec{\Bg} \,, \vec{\Bw}$ and solve
\begin{align*}
 \min_{\vec{\Br} \in X^L} \left\{
 \sum_{l=0}^{L-1} \norm{ \Br_l }_{\ell_1} 
 + \frac{\beta_1}{2} \sum_{l=0}^{L-1} \norm{ \Br_l - \sin\left(\frac{\pi l}{L}\right) \BDm \Bp - \cos\left(\frac{\pi l}{L}\right) \Bp \BDnT + \frac{\boldsymbol{\lambda}_{\boldsymbol 1 l}}{\beta_1} }^2_{\ell_2}
 \right\}
\end{align*}
Due to the separable problems, we consider the problem at $a = 0, \ldots, L-1$
\begin{align*}
 \Br^*_a &= \argmin_{\Br_a \in X} \left\{ \norm{\Br_a}_{\ell_1} + \frac{\beta_1}{2} \norm{\Br_a - \sin\left(\frac{\pi a}{L}\right) \BDm \Bp - \cos\left(\frac{\pi a}{L}\right)\Bp \BDnT + \frac{\boldsymbol{\lambda}_{\boldsymbol 1 a}}{\beta_1} }\right\}
 \\
 &= \Shrink \Big( \sin\left(\frac{\pi a}{L}\right) \BDm \Bp + \cos\left(\frac{\pi a}{L}\right) \Bp \BDnT - \frac{\boldsymbol{\lambda}_{\boldsymbol 1 a}}{\beta_1} \,, \frac{1}{\beta_1} \Big)
 \,,~~ a = 0, \ldots, L-1.
\end{align*}

\noindent {\bfseries The $\vec{\Bw} = \big[ \Bw_s \big]_{s=0}^{S-1}$-problem:}
Fix $c_1 \,, c_2 \,, \Bp \,, \Bv \,, \Beps \,, \vec{\Bg} \,, \vec{\Br}$ and solve
\begin{align*}
 \min_{\vec{\Bw} \in X^S} \left\{
 \mu_1 \sum_{s=0}^{S-1}\norm{\Bw_s}_{\ell_1} 
 + \frac{\beta_2}{2} \sum_{s=0}^{S-1} \norm{ \Bw_s - \Bg_s + \frac{\boldsymbol{\lambda}_{\boldsymbol 2 s}}{\beta_2} }^2_{\ell_2}
 \right\}
\end{align*}
Due to the separable problems, we consider the problem at $a = 0, \ldots, S-1$
\begin{align*}
 \Bw_a^* &= \argmin_{\Bw_a \in X} \left\{
 \mu_1 \norm{\Bw_a}_{\ell_1} 
 + \frac{\beta_2}{2} \norm{ \Bw_a - \Bg_a + \frac{\boldsymbol{\lambda}_{\boldsymbol 2 a}}{\beta_2} }^2_{\ell_2}
 \right\}
 \\
 &= \Shrink \Big( \Bg_a - \frac{\boldsymbol{\lambda}_{\boldsymbol 2 a}}{\beta_2} \,, \frac{\mu_1}{\beta_2} \Big)
 \,, a = 0, \ldots, S-1 \,.
\end{align*}

\noindent {\bfseries The $\vec{\Bg} = \big[ \Bg_s \big]_{s=0}^{S-1}$-problem:}
Fix $c_1 \,, c_2 \,, \Bp \,, \Bv \,, \Beps \,, \vec{\Br} \,, \vec{\Bw}$ and solve
\begin{align*}
 \min_{\vec{\Bg} \in X^S} \left\{
 \frac{\beta_2}{2} \sum_{s=0}^{S-1} \norm{ \Bw_s - \Bg_s + \frac{\boldsymbol{\lambda}_{\boldsymbol 2 s}}{\beta_2} }^2_{\ell_2}
 + \frac{\beta_3}{2} \norm{ \Bv - \sum_{s=0}^{S-1} \Big[ \sin\left(\frac{\pi s}{S}\right) \BDm \Bg_s + \cos\left(\frac{\pi s}{S}\right)\Bg_s \BDnT \Big] + \frac{\boldsymbol{\lambda}_{\boldsymbol 3}}{\beta_3} }^2_{\ell_2}
 \right\}
\end{align*}
The solution of its separable problem is
\begin{align*}
 \Bg_a^* = \RE \left[ \cF^{-1} \big\{ \frac{\mathcal B_a(\Bz)}{\mathcal A_a(\Bz)} \big\} \right][\Bk] \Big|_{\Bk \in \Omega} \,, a = 0, \ldots, S-1
\end{align*}
\begin{align*}
 \mathcal A_a(\Bz) &= \beta_2 + \beta_3 \abs{ \sin\left(\frac{\pi a}{S}\right)(z_1 - 1) + \cos\left(\frac{\pi a}{S}\right)(z_2 - 1) }^2 
 \\
 \mathcal B_a(\Bz) &= \beta_2 \Big[ W_a(\Bz) + \frac{\Lambda_{2a}(\Bz)}{\beta_2} \Big] 
 + \beta_3 \Big[ \sin\left(\frac{\pi a}{S}\right)(z_1^{-1} - 1) + \cos\left(\frac{\pi a}{S}\right) (z_2^{-1} - 1) \Big] \times
 \\&
 \left[ V(\Bz) - \sum_{s = [0, S-1] \backslash \{a\}} \left[ \sin\left(\frac{\pi s}{S}\right)(z_1 - 1) + \cos\left(\frac{\pi s}{S}\right)(z_2 - 1) \right] G_s(\Bz) + \frac{\Lambda_3(\Bz)}{\beta_3} \right]
\end{align*}

\noindent {\bfseries The $\Bv$-problem:} Fix $c_1 \,, c_2 \,, \Bp \,, \Beps \,, \vec{\Bg} \,, \vec{\Br} \,, \vec{\Bw}$ and solve
\begin{align*}
 &\min_{\Bv \in X} \Big\{ \cL(\Bv) =
 \mu_2 \norm{\Bv}_{\ell_1}
 + \frac{\beta_3}{2} \norm{ \Bv - \sum_{s=0}^{S-1} \Big[ \sin\left(\frac{\pi s}{S}\right) \BDm \Bg_s + \cos\left(\frac{\pi s}{S}\right)\Bg_s \BDnT \Big] + \frac{\boldsymbol{\lambda}_{\boldsymbol 3}}{\beta_3} }^2_{\ell_2}
 \\& +
 \frac{\beta_4}{2} \Big\langle \big[ \Bf - c_1 - \Bv - \Beps + \frac{\boldsymbol{\lambda_4}}{\beta_4} \big]^{\cdot 2} \,, \Bp \Big\rangle_{\ell_2}
 + \frac{\beta_4}{2} \Big\langle \big[ \Bf - c_2 - \Bv - \Beps + \frac{\boldsymbol{\lambda_4}}{\beta_4} \big]^{\cdot 2} \,, 1 - \Bp \Big\rangle_{\ell_2}
 \Big\}.
\end{align*}
The Euler-Lagrange equation (in matrix form) is
\begin{align*}
 0 &= \frac{\partial \cL(\Bv)}{\partial \Bv} = 
 \mu_2 \frac{\Bv}{\abs{\Bv}} 
 + \beta_3 \left[ \Bv - \sum_{s=0}^{S-1} \left[ \sin\left(\frac{\pi s}{S}\right) \BDm \Bg_s + \cos\left(\frac{\pi s}{S}\right)\Bg_s \BDnT \right] + \frac{\boldsymbol{\lambda}_{\boldsymbol 3}}{\beta_3} \right]
 \\&
 - \beta_4 \big[ \Bf - c_1 - \Bv - \Beps + \frac{\boldsymbol{\lambda_4}}{\beta_4} \big] \cdot^\times \Bp 
 - \beta_4 \big[ \Bf - c_2 - \Bv - \Beps + \frac{\boldsymbol{\lambda_4}}{\beta_4} \big] \cdot^\times \big[ 1 - \Bp \big]
 \\ \Leftrightarrow~ 
 \Bv 
 &= \mathcal J -\frac{\mu_2}{\beta_3 + \beta_4} \frac{\Bv}{\abs{\Bv}}
 ~:=~ \Shrink \Big( \mathcal J \,, \frac{\mu_2}{\beta_3 + \beta_4} \Big)
\end{align*}
with
\begin{align*}
 \mathcal J &= \frac{\beta_3}{\beta_3 + \beta_4} \left[ \sum_{s=0}^{S-1} \left[ \sin\left(\frac{\pi s}{S}\right) \BDm \Bg_s + \cos\left(\frac{\pi s}{S}\right)\Bg_s \BDnT \right] - \frac{\boldsymbol{\lambda}_{\boldsymbol 3}}{\beta_3} \right]
 + \frac{\beta_4}{\beta_3 + \beta_4} \big[ \Bf - c_1 - \Beps + \frac{\boldsymbol{\lambda_4}}{\beta_4} \big] \cdot^\times \Bp 
 \\&
 + \frac{\beta_4}{\beta_3 + \beta_4} \big[ \Bf - c_2 - \Beps + \frac{\boldsymbol{\lambda_4}}{\beta_4} \big] \cdot^\times \big[ 1 - \Bp \big] \,.
\end{align*}

\noindent {\bfseries The $\Bp$-problem:} Fix $c_1 \,, c_2 \,, \Bv \,, \Beps \,, \vec{\Bg} \,, \vec{\Br} \,, \vec{\Bw}$ and solve
\begin{align} \label{eq:SHT:problem:p:appendix}
 \min_{\Bp \in \mathcal D} \bigg\{ \cL(\Bp) = 
 &\frac{\beta_1}{2} \sum_{l=0}^{L-1} \norm{ \Br_l - \sin\left(\frac{\pi l}{L}\right) \BDm \Bp - \cos\left(\frac{\pi l}{L}\right) \Bp \BDnT + \frac{\boldsymbol{\lambda}_{\boldsymbol 1 l}}{\beta_1} }^2_{\ell_2} \notag
 \\& +
 \frac{\beta_4}{2} \Big\langle \big[ \Bf - c_1 - \Bv - \Beps + \frac{\boldsymbol{\lambda_4}}{\beta_4} \big]^{\cdot 2} \,, \Bp \Big\rangle_{\ell_2}
 + \frac{\beta_4}{2} \Big\langle \big[ \Bf - c_2 - \Bv - \Beps + \frac{\boldsymbol{\lambda_4}}{\beta_4} \big]^{\cdot 2} \,, 1 - \Bp \Big\rangle_{\ell_2}
 \bigg\}
\end{align}
The Euler-Lagrange equation is
\begin{align*}
 &0 = \frac{\partial \cL(\Bp)}{\partial \Bp} = 
 - \beta_1 \sum_{l=0}^{L-1} 
 \bigg[ 
 \cos\left(\frac{\pi l}{L}\right) \Big[ \Br_l - \sin\left(\frac{\pi l}{L}\right) \BDm \Bp - \cos\left(\frac{\pi l}{L}\right) \Bp \BDnT 
 + \frac{\boldsymbol{\lambda}_{\boldsymbol 1 l}}{\beta_1} \Big] \BDn
 \\&
 + \sin\left(\frac{\pi l}{L}\right) \BDmT \Big[ \Br_l - \sin\left(\frac{\pi l}{L}\right) \BDm \Bp - \cos\left(\frac{\pi l}{L}\right) \Bp \BDnT  
 + \frac{\boldsymbol{\lambda}_{\boldsymbol 1 l}}{\beta_1} \Big]
 \bigg]
 \\&
 + \frac{\beta_4}{2} \Big[ \Bf - c_1 - \Bv - \Beps + \frac{\boldsymbol{\lambda_4}}{\beta_4} \Big]^{\cdot 2} 
 - \frac{\beta_4}{2} \Big[ \Bf - c_2 - \Bv - \Beps + \frac{\boldsymbol{\lambda_4}}{\beta_4} \Big]^{\cdot 2} \,.
\end{align*}
Its Fourier transform is
\begin{align*}
 -\beta_1 \sum_{l=0}^{L-1}  \Big[ & \sin\left(\frac{\pi l}{L}\right) (z_1^{-1} - 1) + \cos\left(\frac{\pi l}{L}\right) (z_2^{-1} - 1) \Big] \\
 &\times \Big[ R_l(\Bz) - \big[ \sin\left(\frac{\pi l}{L}\right) (z_1 - 1) + \cos\left(\frac{\pi l}{L}\right) (z_2 - 1) \big] P(\Bz)
 + \frac{\Lambda_{1 l}(\Bz)}{\beta_1} \Big] 
 \\& \hspace{10mm}
 - \beta_4 \mathcal H(\Bz) = 0 
\end{align*}
with the Fourier transform of $\mathcal H(\Bz)$ as
\begin{align*}
 \mathcal H(\Bz) = 
 \cF \left\{
 - \frac{1}{2} \big[ \Bf - c_1 - \Bv - \Beps + \frac{\boldsymbol{\lambda_4}}{\beta_4} \big]^{\cdot 2} 
 + \frac{1}{2} \big[ \Bf - c_2 - \Bv - \Beps + \frac{\boldsymbol{\lambda_4}}{\beta_4} \big]^{\cdot 2} 
 \right\} (\Bz) \,.
\end{align*}
Thus,
\begin{align*}
 &\beta_1 \sum_{l=0}^{L-1} \abs{ \sin\left(\frac{\pi l}{L}\right) (z_1 - 1) + \cos\left(\frac{\pi l}{L}\right) (z_2 - 1) }^2 P(\Bz)
 \\&
 = \beta_1 \sum_{l=0}^{L-1} 
 \left[ \sin\left(\frac{\pi l}{L}\right) (z_1^{-1} - 1) + \cos\left(\frac{\pi l}{L}\right) (z_2^{-1} - 1) \right]
 \left[ R_l(\Bz) + \frac{\Lambda_{1 l}(\Bz)}{\beta_1} \right] 
 + \beta_4 \mathcal H(\Bz)
\end{align*}
To avoid the singularity, we add $\frac{\beta_4}{2} \norm{\Bp}^2_{\ell_2}$ to the ``$\Bp$-problem'' (\ref{eq:SHT:problem:p:appendix})
which results in its solution as
\begin{align*}
 \tilde p[\Bk] = \RE \Big[ \cF^{-1} \Big\{ \frac{\mathcal Y(\Bz)}{\mathcal X (\Bz)} \Big\} \Big] [\Bk] \,, \Bk \in \Omega 
\end{align*}
with
\begin{align*}
 \mathcal Y(\Bz) &= \beta_1 \sum_{l=0}^{L-1} 
 \Big[ \sin\left(\frac{\pi l}{L}\right) (z_1^{-1} - 1) + \cos\left(\frac{\pi l}{L}\right) (z_2^{-1} - 1)\Big]
 \Big[ R_l(\Bz) + \frac{\Lambda_{1 l}(\Bz)}{\beta_1} \Big] 
 + \beta_4 \mathcal H(\Bz) \,,
 \\
 \mathcal X(\Bz) &= \beta_1 \sum_{l=0}^{L-1} \abs{ \sin\left(\frac{\pi l}{L}\right) (z_1 - 1) + \cos\left(\frac{\pi l}{L}\right) (z_2 - 1) }^2 + \beta_4 \,.
\end{align*}
Due to $\tilde \Bp = \big[ \tilde p[\Bk] \big]_{\Bk \in \Omega} \in \mathcal D$, we have
\begin{align*}
 p^*[\Bk] = \begin{cases}
             0 \,, & \tilde p[\Bk] < 0 \\
             1 \,, & \tilde p[\Bk] > 1 \\
             \tilde p[\Bk] \,, & \tilde p[\Bk] \in [0 \,, 1]
            \end{cases} \,.
\end{align*}

\noindent {\bfseries The $\Beps$-problem:} Fix $c_1 \,, c_2 \,, \Bp \,, \Bv \,, \vec{\Bg} \,, \vec{\Br} \,, \vec{\Bw}$ and
given a convex set with the indicator function $G^*(\frac{\Beps}{\nu})$ 
\begin{equation*}
 \mathcal A(\nu) = \Big\{ \Beps \in X ~:~ \norm{ \cC \{ \Beps \} }_{\ell_\infty} \leq \nu \Big\}
\end{equation*}
the $\Beps$-problem can be rewritten as
\begin{align*}
 \min_{\Beps \in \mathcal A(\nu)} \left\{ \cL(\Beps) =
 \frac{\beta_4}{2} \Big\langle \big[ \Bf - c_1 - \Bv - \Beps + \frac{\boldsymbol{\lambda_4}}{\beta_4} \big]^{\cdot 2} \,, \Bp \Big\rangle_{\ell_2}
 + \frac{\beta_4}{2} \Big\langle \big[ \Bf - c_2 - \Bv - \Beps + \frac{\boldsymbol{\lambda_4}}{\beta_4} \big]^{\cdot 2} \,, 1 - \Bp \Big\rangle_{\ell_2}
 \right\}
\end{align*}
The Euler-Lagrange equation is
\begin{align*}
 0 &= \frac{\partial \cL(\Beps)}{\partial \Beps} = 
 -\beta_4 \big[ \Bf - c_1 - \Bv - \Beps + \frac{\boldsymbol{\lambda_4}}{\beta_4} \big] \cdot^\times \Bp 
 - \beta_4 \big[ \Bf - c_2 - \Bv - \Beps + \frac{\boldsymbol{\lambda_4}}{\beta_4} \big] \cdot^\times \big[ 1 - \Bp \big]
 \\ \Leftrightarrow 
 \tilde{\Beps} &= \big[ \Bf - c_1 - \Bv + \frac{\boldsymbol{\lambda_4}}{\beta_4} \big] \cdot^\times \Bp + 
 \big[ \Bf - c_2 - \Bv + \frac{\boldsymbol{\lambda_4}}{\beta_4} \big] \cdot^\times \big[ 1 - \Bp \big] \,.
\end{align*}
Thus, the solution of the $\Beps$-problem is
\begin{align*}
 \Beps^* &= \mathbb{P}_{\mathcal A(\nu)} \big( \tilde{\Beps} \big)
 = \tilde{\Beps} - \CST \big( \tilde{\Beps} \,, \nu \big) \,.
\end{align*}

\noindent {\bfseries Update the Lagrange multipliers:}
\begin{align*}
 \boldsymbol{\lambda}_{\boldsymbol 1 a}^{(t)} &= \boldsymbol{\lambda}_{\boldsymbol 1 a}^{(t-1)}
 + \beta_1 \Big[ \Br_a - \sin\left(\frac{\pi a}{L}\right) \BDm \Bp - \cos\left(\frac{\pi a}{L}\right) \Bp \BDnT \Big] \,,~ a = 0, \ldots, L-1
 \\
 \boldsymbol{\lambda}_{\boldsymbol 2 a}^{(t)} &= \boldsymbol{\lambda}_{\boldsymbol 2 a}^{(t-1)} 
 + \beta_2 \Big[ \Bw_a - \Bg_a \Big] \,, a = 0, \ldots, S-1
 \\
 \boldsymbol{\lambda}_{\boldsymbol 3}^{(t)} &= \boldsymbol{\lambda}_{\boldsymbol 3}^{(t-1)} 
 + \beta_3 \Big[ \Bv - \sum_{s=0}^{S-1} \big[ \sin\left(\frac{\pi s}{S}\right) \BDm \Bg_s + \cos\left(\frac{\pi s}{S}\right) \Bg_s \BDnT \big] \Big]
 \\
 \boldsymbol{\lambda}_{\boldsymbol 4}^{(t)} &= \boldsymbol{\lambda}_{\boldsymbol 4}^{(t-1)} 
 + \beta_4 \Big[ \big( \Bf - c_1 - \Bv - \Beps \big) \cdot^\times \Bp 
 + \big( \Bf - c_2 - \Bv - \Beps \big) \cdot^\times (1 - \Bp) \Big]
\end{align*}

\end{proof}




\begin{Lem} \label{lem:DTVDGnorm:dualpair}
 The discrete directional total variation (DTV) norm $\norm{\nabla^+_L \Bu}_{\ell_1}$ 
 and the discrete directional $\text{G}_L$-norm $\norm{\Bv}_{\text{G}_L}$ are a dual pair, i.e.
 \begin{equation} \label{eq:dualpair_TV-G}
  J_L(\Bu) = \mu \norm{\nabla^+_L \Bu}_{\ell_1} \stackrel{\text{Legendre Fenchel transform}}{\longleftrightarrow}
  J_L^* \big( \frac{\Bv}{\mu} \big)
  = \begin{cases} 0 \,, & \Bv \in G_L(\mu) \\ +\infty \,, & \text{else}  
    \end{cases}
 \end{equation}
with a convex set of $\text{G}_L$-norm 
$$
\text{G}_L(\mu) = \Big\{ \Bv = \text{div}^-_L \vec{\Bp} \in X \,,~ \vec{\Bp} = \big[ \Bp_l \big]_{l=0}^{L-1} \in X^L ~:~
\norm{\Bv}_{\text{G}_L} = \inf_{\vec{\Bp}} \big\{ \underbrace{\max_{\Bk \in \Omega} \sqrt{\sum_{l=0}^{L-1} p_l^2[\Bk]} }_{=\norm{\vec{\Bp}}_{\ell_\infty}} \big\} \leq \mu \Big\} \,.
$$
\end{Lem}

\begin{proof}

{\bfseries 1. The discrete DTV-norm}

Given a convex set
$K_L(\mu) = \big\{ \text{div}^-_L \vec{\Bp} \,,~ \vec{\Bp} \in X^L : \norm{\vec{\Bp}}_{\ell_\infty} \leq \mu \big\}$
and an adjoint operator $(\nabla^+_L)^* = -\text{div}^-_L$, 
the discrete DTV-norm $J_L(u)$ in (\ref{eq:dualpair_TV-G}) can be rewritten as
\begin{align} \label{eq: defTVnorm}
 J_L(\Bu) &= \mu \norm{\nabla^+_L \Bu}_{\ell_1} 
 = \max_{\vec{\Bp} \in K_L(\mu)} \big\langle \Bu \,, \text{div}^-_L \vec{\Bp} \big\rangle_{\ell_2} 
\end{align}
since
\begin{align*}
 J_L(\Bu) &= \mu \sum_{\Bk \in \Omega} \abs{ \nabla_L^+ u[\Bk] } 
 = \max_{ \vec{\Bp} \in K_L(\mu)} - \sum_{\Bk \in \Omega} 
 \big\langle \nabla^+_L u[\Bk] \,, \vec{p}[\Bk] \big\rangle_{\ell_2} 
 = \max_{\vec{\Bp} \in K_L(\mu)} \sum_{\Bk \in \Omega} 
 \underbrace{ \langle u[\Bk] \,, \text{div}^-_L \vec{p}[\Bk] \rangle_{\ell_2} }_{ = u[\Bk] \text{div}^-_L \vec{p}[\Bk] } \,.
\end{align*}


\noindent {\bfseries 2. The Legendre-Fenchel transform (in a discrete setting) of $J_L(\Bu) \,, \Bu \in X (\text{or } \Bu \in \ell_2(\Omega, \mathbb R))$}
\begin{align*}
 J_L^* \big( \frac{\Bv}{\mu} \big) = \sup_{\Bu \in X} \Big\{ \langle \Bv, \Bu \rangle_{\ell_2(\Omega, \mathbb R)}
 - \underbrace{ J_L(\Bu) }_{=\mu \norm{\nabla^+_L \Bu}_{L_1}} \Big\}
 = \sup_{\Bu \in X} \Big\{ \underbrace{ \big\langle \frac{\Bv}{\mu}, \Bu \big\rangle_{\ell_2(\Omega, \mathbb R)} - \norm{\nabla^+_L \Bu}_{L_1} }_
                                      { = \mathscr H(\Bu)} \Big\}
 \,,~ \Bv \in X \,.
\end{align*}
Denote the dual norm of $\norm{\nabla^+_L \Bu}_{L_1}$ as $\norm{\Bv}_{\text{G}_L}$.
Given $\norm{\Bv}_{\text{G}_L} \leq \mu$, the Cauchy Schwarz inequality gives
\begin{align} \label{eq:CauchySchwarzInequality:TVGnorms}
 &\frac{1}{\mu} \langle \Bv \,, \Bu \rangle_{\ell_2(\Omega, \mathbb R)} \leq \frac{1}{\mu} \norm{\Bv}_{\text{G}_L} \norm{\nabla^+_L \Bu}_{\ell_1}
 \leq \norm{\nabla^+_L \Bu}_{\ell_1}
 \\
 \Leftrightarrow~
 &\mathscr H(\Bu) \leq 0 ~~\text{and}~~
 \Bv \in \text{G}_L (\mu) = \Big\{ \Bv \in \ell_2(\Omega, \mathbb R) ~:~ \norm{\Bv}_{\text{G}_L} \leq \mu \Big\}  \notag
 \\
 \Leftrightarrow~
 &J_L^*(\frac{\Bv}{\mu}) = \sup_{\Bu \in X} \mathscr H(u) = \begin{cases} 0 \,, & \Bv \in \text{G}_L(\mu) \\ +\infty \,, & \text{else} \end{cases}
 = \chi_{\text{G}_L(\mu)} (\Bv) \,.  \notag
\end{align}
We have a biconjugate form $J_L = J_L^{**}$:
\begin{align} \label{eq:defTVnorm:compare}
 J_L(\Bu) &= \Big(J_L^* \big(\frac{\Bv}{\mu} \big)\Big)^*(\Bu) = \chi^*_{\text{G}_L(\mu)}(\Bu) 
 = \sup_{\Bv \in X} \Big\{ \langle \Bu \,, \Bv \rangle_{\ell_2(\Omega, \mathbb R)} - \chi_{\text{G}_L(\mu)}(\Bv) \Big\}  \notag
 \\
 &= \sup_{\Bv \in \text{G}_L(\mu)} \langle \Bu \,, \Bv \rangle_{\ell_2(\Omega, \mathbb R)}
\end{align}
From (\ref{eq:defTVnorm:compare}) and definition of the DTV-norm (\ref{eq: defTVnorm}), we have
\begin{align*}
&\begin{cases}
v[\Bk] &= \text{div}^-_L \vec{p}[\Bk] \,,~ \forall \Bk \in \Omega
\\
\Bv \in \text{G}_L(\mu) &= \Big\{ \Bv \in X ~:~ \norm{\Bv}_{\text{G}_L} \leq \mu \Big\}
\\
\vec{\Bp} \in K_L(\mu) &= \Big\{ \text{div}^-_L \vec{\Bp} \,,~ \vec{\Bp} \in X^L ~:~ \norm{\vec{\Bp}}_{\ell_\infty} \leq \mu \Big\}
\end{cases}
\\
\Leftrightarrow~
& \vec{\Bp} \in K_L(\mu) = \Big\{ \Bv = \text{div}^-_L \vec{\Bp} \in X \,,~ 
\vec{\Bp} \in X^L ~:~ \norm{\Bv}_{\text{G}_L} = \norm{ \vec{\Bp} }_{\ell_\infty} \leq \mu \Big\}
\equiv \text{G}_L(\mu) \,.
\end{align*}
Therefore, the discrete $(L)$ DTV-norm  and the discrete directional $\text{G}_L$-norm are a dual pair by the definition of dual space.

\end{proof}


%
\begin{prop} \label{prop:DirectionalTVL2}
The DTV-$\ell_2$ model is
\begin{align} \label{eq:DirectionalTVL2}
 \Bu^* &= \argmin_{\Bu \in X} \Big\{ \mathscr F(\Bu) = J_L(\Bu) + \frac{\mu}{2} \norm{\Bu - \Bf}^2_{\ell_2} \Big\}
 \\
 &= \Bf - \frac{1}{\mu} \underbrace{ \text{div}^-_L \vec{\Bp}^* }_{= \Bw^*}  \notag
\end{align}
with $\vec{\Bp}^* \approx \vec{\Bp}^{(T)}$ and
\begin{align*}
 \vec{\Bp}^{(t+1)}
 = 
 \frac{ \vec{\Bp}^{(t)} + \tau \nabla^+_L \Big[ \text{div}^-_L \vec{\Bp}^{(t)} - \mu \Bf \Big] }
      { 1 + \tau \abs{ \nabla^+_L \Big[ \text{div}^-_L \vec{\Bp}^{(t)} - \mu \Bf \Big] } }
 \,,~ t = 1, \ldots, T \,.
\end{align*}
\end{prop}

\begin{proof} 
The proof is extended directly from \cite{Chambolle2004} for the multi-directional total variation norm $J_L(\Bu)$.
The Euler-Lagrange of (\ref{eq:DirectionalTVL2}) is
\begin{align*}
 &0 \in \frac{\partial \mathscr F(\Bu)}{\partial \Bu} = \partial J_L(\Bu) + \mu (\Bu - \Bf)
 ~~ \Leftrightarrow ~~
 \Bu \in \partial J^*_L \Big(\mu (\Bf - \Bu) \Big)
 \\ \Leftrightarrow &
 \Bu - \Bf + \Bf \in \partial J^*_L \Big(\mu (\Bf - \Bu) \Big)
 ~~ \Leftrightarrow ~~ 
 \Bf \in \partial J^*_L \Big(\mu (\Bf - \Bu) \Big) + \Bf - \Bu \,.
\end{align*}
Let $\Bw = \mu(\Bf - \Bu)$
\begin{align*}
 \Leftrightarrow~ 
 0 &\in \partial J^*_L (\Bw) + \frac{\Bw}{\mu} - \Bf
 \\ \Leftrightarrow~ 
 \Bw^* &= \argmin_{\Bw \in X} \left\{ J^*_L(\Bw) + \frac{\mu}{2} \norm{ \frac{\Bw}{\mu} - \Bf }^2_{\ell_2} \right\}
 \\ &
 = \argmin_{\Bw \in K_L(1)} \Big\{ \norm{ \Bw - \mu \Bf }^2_{\ell_2} \Big\} \,,~~   
 K_L(1) = \Big\{ \text{div}^-_L \vec{\Bp} \,, \vec{\Bp} = \big[ \Bp_l \big]_{l=0}^{L-1} \in X^L
  ~:~ \abs{\vec{p}[\Bk]} \leq 1 \,, \forall \Bk \in \Omega \Big\}
 \\ &
 = \text{div}^-_L \vec{\Bp}^*
\end{align*}
with
\begin{align*}
 \vec{\Bp}^* = \argmin_{\vec\Bp \in X^L} \bigg\{ \norm{ \text{div}^-_L \vec{\Bp} - \mu \Bf }^2_{\ell_2}
 ~~\text{s.t.}~~  \underbrace{ \sum_{l=0}^{L-1} p_l^2[\Bk] - 1 \leq 0 }_{\Leftrightarrow~ \abs{\vec{p}[\Bk]}^2 -1 \leq 0 } \,, \forall \Bk \in \Omega
 \bigg\} \,.
\end{align*}
Due to an inequality constraint,
the KKT condition in matrix form $(\lambda[\Bk] > 0 \,, \forall \Bk \in \Omega)$ is
\begin{align*}
 \min_{\vec\Bp \in X^L} \Bigg\{\mathscr L(\vec{\Bp}) = \norm{ \text{div}^-_L \vec{\Bp} - \mu \Bf }^2_{\ell_2}
 ~+~ \Big\langle \boldsymbol{\lambda} \,,  \sum_{l=0}^{L-1} \Bp_l^{\cdot 2} - 1 \Big\rangle_{\ell_2}
 \Bigg\} \,.
\end{align*}
Due to its separability, we consider this problem at $l = l' \,, l' = 0, \ldots, L-1$ and
its Euler-Lagrange equation is given by 
\begin{align*}
 &0 = \frac{\partial \mathscr L(\vec{\Bp})}{\partial \Bp_{l'}} 
 = 2 \underbrace{ \frac{\partial}{\partial \Bp_{l'}} \Big\{ \text{div}^-_L \vec{\Bp} \Big\} }_{ = -\partial^+_{l'} \delta } \Big[ \text{div}^-_L \vec{\Bp} - \mu \Bf \Big]
 ~+~ 2 \boldsymbol{\lambda} \cdot^\times \Bp_{l'}
 \\ \Leftrightarrow &
 -\partial^+_{l'} \Big[ \text{div}^-_L \vec{\Bp} - \mu \Bf \Big]
 ~+~ \boldsymbol{\lambda} \cdot^\times \Bp_{l'} = 0 
 \,,~ l' = 0, \ldots L-1
 \\ \Leftrightarrow &
 -\nabla_L^+ \Big[ \text{div}^-_L \vec{\Bp} - \mu \Bf \Big]
 ~+~ \boldsymbol{\lambda} \cdot^\times \vec{\Bp} = 0 
 \,.
\end{align*}
Given $\Bk \in \Omega$, its element form is
\begin{align} \label{eq:KKT:1}
 &-\nabla^+_L \Big[ \text{div}^-_L \vec{p}[\Bk] - \mu f[\Bk] \Big]
 ~+~ \lambda[\Bk] \vec{p}[\Bk] = 0 
 \\
 \Leftrightarrow~ 
 &\lambda[\Bk] = \frac{\displaystyle \nabla^+_L \Big[ \text{div}^-_L \vec{p}[\Bk] - \mu f[\Bk] \Big]}
                      {\vec{p}[\Bk]}
 := \Bigg[ \frac{\displaystyle \partial^+_l \Big[ \text{div}^-_L \vec{p}[\Bk] - \mu f[\Bk] \Big]}
                {p_l[\Bk]}
    \Bigg]_{l=0}^{L-1}
 \notag
 \\
 \Leftrightarrow~ 
 &\abs{\lambda[\Bk]} \stackrel{\lambda[\Bk]>0}{=} \lambda[\Bk] = \frac{\abs{ \nabla^+_L \Big[ \text{div}^-_L \vec{p}[\Bk] - \mu f[\Bk] \Big] }}{\abs{\vec{p}[\Bk]}}
 = \begin{cases}
    \abs{ \nabla^+_L \Big[ \text{div}^-_L \vec{p}[\Bk] - \mu f[\Bk] \Big] } \,, & \abs{\vec{p}[\Bk]} = 1
    \\
    0 \,, & \abs{\vec{p}[\Bk]} < 1 
   \end{cases}
 \notag
 \\
 \Leftrightarrow~ 
 & \lambda[\Bk] = \abs{ \nabla^+_L \Big[ \text{div}^-_L \vec{p}[\Bk] - \mu f[\Bk] \Big] } \,.  \notag
\end{align}
Note that Lagrange multiplier is not active on an open set. We rewrite (\ref{eq:KKT:1}) as
\begin{align*}
 -\nabla^+_L \Big[ \text{div}^-_L \vec{p}[\Bk] - \mu f[\Bk] \Big]
 ~+~ \abs{ \nabla^+_L \Big[ \text{div}^-_L \vec{p}[\Bk] - \mu f[\Bk] \Big] } \vec{p}[\Bk] = 0 \,,~  \forall \Bk \in \Omega \,.
\end{align*}
Applying gradient descent and a fixed point algorithm (i.e.\ a semi-implicit gradient descent scheme), we have
\begin{align*}
 &-\nabla^+_L \Big[ \text{div}^-_L \vec{p}^{(t-1)}[\Bk] - \mu f[\Bk] \Big]
 ~+~ \abs{ \nabla^+_L \Big[ \text{div}^-_L \vec{p}^{(t-1)}[\Bk] - \mu f[\Bk] \Big] } \vec{p}^{(t)}[\Bk] = 
 - \frac{\vec{p}^{(t)}[\Bk] - \vec{p}^{(t-1)}[\Bk]}{\tau} \,,~ 
\end{align*}
and the solution in matrix form is given by
\begin{align*}
 \vec{\Bp}^{(t)}
 = 
 \frac{ \vec{\Bp}^{(t-1)} + \tau \nabla^+_L \Big[ \text{div}^-_L \vec{\Bp}^{(t-1)} - \mu \Bf \Big] }
      { 1 + \tau \abs{ \nabla^+_L \Big[ \text{div}^-_L \vec{\Bp}^{(t-1)} - \mu \Bf \Big] } } 
 \,,~ t = 1 \,, \ldots
\end{align*}

\end{proof}


\begin{prop}  \label{prop:DirectionalGL1}
 The directional $\text{G}_S-\ell_1$ model is
 \begin{align}  \label{eq:L1DGnorm:1}
  \min_{\Bv \in X} \Big\{ \mathscr F(\Bv) = \norm{\Bv}_{\ell_1} + \frac{\beta}{2} \norm{\Bf - \Bv}^2_{\ell_2} 
  ~~\text{s.t.}~~ \norm{\Bv}_{\text{G}_S} \leq \mu
  \Big\} \,.
 \end{align}
For iteration $t$, we have
 
\noindent  {\bfseries Dual variable $\vec{\Bg} = \big[ \Bg_s \big]_{s=0}^{S-1}$}
  \begin{align*}
   \vec{\Bg}^{(t)}
   = \frac{ \vec{\Bg}^{(t-1)} + \tau \nabla_S^+ \Big[ \alpha \mu \text{div}^-_S \vec{\Bg}^{(t-1)} - \boldsymbol{\lambda}_{\boldsymbol 1}^{(t-1)} - \alpha \Bv^{(t-1)} \Big] }
          { 1 + \tau \abs{\nabla^+_S \Big[ \alpha \mu \text{div}^-_S \vec{\Bg}^{(t-1)} - \boldsymbol{\lambda}_{\boldsymbol 1}^{(t-1)} - \alpha \Bv^{(t-1)} \Big]} } \,,
  \end{align*} 
 
\noindent {\bfseries Primal variable}
  \begin{align*}
   \Bv^{(t)} = 
   \Shrink \Big( \frac{\beta}{\alpha + \beta} \Bf + \frac{\alpha \mu}{\alpha + \beta} \text{div}^-_S \vec{\Bg}^{(t)}
                 - \frac{1}{\alpha + \beta} \boldsymbol{\lambda}_{\boldsymbol 1}^{(t-1)}
                 \,,~ \frac{1}{\alpha + \beta}
           \Big) \,,
  \end{align*}     
  
\noindent {\bfseries Updated Lagrange multiplier $\boldsymbol{\lambda_1} \in X$} 
  \begin{align*}
   \boldsymbol{\lambda}_{\boldsymbol 1}^{(t)} = \boldsymbol{\lambda}_{\boldsymbol 1}^{(t-1)}
   + \alpha \Big[ \Bv^{(t)} - \mu \text{div}^-_S \vec{\Bg}^{(t)} \Big] \,.
  \end{align*}

\end{prop}

\begin{proof}
Given the directional $\text{G}_S$-norm \cite{AujolChambolle2005, ThaiGottschlich2016DG3PD},
we have the convex set 
\begin{align*}
 \text{G}_S(\mu) = \Bigg\{ \Bv = \text{div}^-_S \vec\Bg \in X \,, \vec{\Bg} = \big[ \Bg_s \big]_{s=0}^{S-1} \in X^S 
 ~:~ \underbrace{ \norm{\Bv}_{\text{G}_S} = \norm{ \sqrt{ \sum_{s=0}^{S-1} \Bg_s^{\cdot 2} } }_{\ell_\infty} \leq \mu
                }_{\displaystyle \Leftrightarrow \abs{g[\Bk]} = \sqrt{\sum_{s=0}^{S-1} g_s^2[\Bk]} \leq \mu \,, \forall \Bk \in \Omega }
 \Bigg\}
\end{align*}
Note that $\text{G}_S(\mu)$ is identical to $K_S(\mu)$; see Lemma \ref{lem:DTVDGnorm:dualpair}.
The objective function of (\ref{eq:L1DGnorm:1}) can be rewritten as
\begin{align*}
 \mathscr F(\Bv) &= \begin{cases}
            \norm{\Bv}_{\ell_1} + \frac{\beta}{2} \norm{\Bf - \Bv}^2_{\ell_2} \,,& \Bv \in \text{G}_S(\mu)
            \\
            +\infty \,,& \text{else}
           \end{cases}
 \hspace{3mm} = \hspace{3mm} \norm{\Bv}_{\ell_1} + \frac{\beta}{2} \norm{\Bf - \Bv}^2_{\ell_2} + J_S^* \left(\frac{\Bv}{\mu}\right)
\end{align*}
with indicator function
\begin{align*}
 J_S^* \left(\frac{\Bv}{\mu}\right) = 
 \begin{cases}
  0 \,, & \Bv \in \text{G}_S(\mu) \\ +\infty \,, & \text{else}
 \end{cases}
 \,.
\end{align*}
Thus, we can rewrite (\ref{eq:L1DGnorm:1}) as
\begin{align*}  
 &\min_{\Bv \in \text{G}_S(\mu)} \Big\{ \norm{\Bv}_{\ell_1} + \frac{\beta}{2} \norm{\Bf - \Bv}^2_{\ell_2} \Big\}
 \\
 &= \min_{\Bv \in X} \left\{ \norm{\Bv}_{\ell_1} + \frac{\beta}{2} \norm{\Bf - \Bv}^2_{\ell_2} 
 \,, \Bv = \text{div}^-_S \vec{\Bg} \,, \abs{\frac{g[\Bk]}{\mu}} = \sqrt{\sum_{s=0}^{S-1} \frac{g_s^2[\Bk]}{\mu^2}} \leq 1 \,, \forall \Bk \in \Omega
 \right\} \,.
\end{align*}
By changing variable $\vec{\Bg}'$ to $\frac{\vec{\Bg}}{\mu}$, we have
\begin{align}  
 \min_{\Bv \in X} \left\{ \norm{\Bv}_{\ell_1} + \frac{\beta}{2} \norm{\Bf - \Bv}^2_{\ell_2} 
 \,,~ \Bv = \mu \text{div}^-_S \vec{\Bg} \,,~ \sum_{s=0}^{S-1} g_s^2[\Bk] - 1 \leq 0 \,, \forall \Bk \in \Omega
 \right\} \,.
\end{align}
We then apply ALM for the equality constraint and the KKT condition (generalized Lagrange multiplier) for the inequality constraint as
\begin{align*}
 \min_{ ( \Bv, \vec{\Bg} ) \in X^{S+1} } \cL(\Bv, \vec{\Bg} ~;~ \boldsymbol{\lambda}_{\boldsymbol 1}, \boldsymbol{\lambda}_{\boldsymbol 2})
\end{align*}
with Lagrange function
\begin{align*}
 \cL(\cdot; \cdot) = \norm{\Bv}_{\ell_1} + \frac{\beta}{2} \norm{\Bf - \Bv}^2_{\ell_2}  
 + \Big\langle \boldsymbol{\lambda_1} \,, \Bv - \mu \text{div}^-_S \vec{\Bg} \Big\rangle_{\ell_2}
 + \frac{\alpha}{2} \norm{\Bv - \mu \text{div}^-_S \vec{\Bg} }^2_{\ell_2} 
 + \Big\langle \boldsymbol{\lambda_2} \,, \sum_{s=0}^{S-1} \Bg_s^{\cdot 2} - 1 \Big\rangle_{\ell_2} \,.
\end{align*}
Note that $\lambda_2[\Bk]>0 \,, \forall \Bk \in \Omega$. 
In a sequential fashion, we consider the primal, dual problems and then update the Lagrange multiplier $\boldsymbol {\lambda_1}$ as:

\noindent {\bfseries The primal $\Bv$-problem:} fix $\vec{\Bg}$ and solve
\begin{align*}
 \min_{\Bv \in X} \left\{ \cL(\Bv) =
 \norm{\Bv}_{\ell_1} + \frac{\beta}{2} \norm{\Bf - \Bv}^2_{\ell_2}  
 + \Big\langle \boldsymbol{\lambda_1} \,, \Bv - \mu \text{div}^-_S \vec{\Bg} \Big\rangle_{\ell_2}
 + \frac{\alpha}{2} \norm{\Bv - \mu \text{div}^-_S \vec{\Bg} }^2_{\ell_2} 
 \right\}
\end{align*}
The Euler-Lagrange equation is given by
\begin{align*}
 0 &= \frac{\partial \cL(\Bv)}{\partial \Bv} = 
 \frac{\Bv}{\abs{\Bv}}  + \beta(\Bv - \Bf) + \boldsymbol{\lambda_1}
 + \alpha \Big[ \Bv - \mu \text{div}^-_S \vec{\Bg} \Big]
 \\ \Leftrightarrow~
 \Bv^* &= \frac{\beta}{\alpha + \beta} \Bf + \frac{\alpha \mu}{\alpha + \beta} \text{div}^-_S \vec{\Bg} 
 - \frac{1}{\alpha + \beta} \boldsymbol{\lambda_1}
 - \frac{1}{\alpha + \beta} \frac{\Bv}{\abs{\Bv}}
 \\&:=
 \Shrink \Big( \frac{\beta}{\alpha + \beta} \Bf + \frac{\alpha \mu}{\alpha + \beta} \text{div}^-_S \vec{\Bg}
               -\frac{1}{\alpha + \beta} \boldsymbol{\lambda_1}
               \,,~ \frac{1}{\alpha + \beta}
         \Big) \,.
\end{align*}

\noindent {\bfseries The dual $\vec{\Bg} = \big[ \Bg_s \big]_{s=0}^{S-1}$-problem:} Fix $\Bv$ and solve
\begin{align*}
 \min_{\vec{\Bg} \in X^S} \left\{ \cL(\vec{\Bg}) =
 \Big\langle \boldsymbol{\lambda_1} \,, \Bv - \mu \text{div}^-_S \vec{\Bg} \Big\rangle_{\ell_2}
 + \frac{\alpha}{2} \norm{\Bv - \mu \text{div}^-_S \vec{\Bg} }^2_{\ell_2} 
 + \Big\langle \boldsymbol{\lambda_2} \,, \sum_{s=0}^{S-1} \Bg_s^{\cdot 2} - 1 \Big\rangle_{\ell_2}
 \right\}
\end{align*}
Due to its separability, we evaluate the dual $\vec{\Bg}$ problem at $s = s' \,, s' = 0, \ldots, S-1$ 
and its Euler-Lagrange equation is given by
\begin{align*}
 &0 = \frac{\partial \cL(\vec{\Bg})}{\partial \Bg_{s'}} = 
 - \mu \underbrace{ \Big\langle \boldsymbol{\lambda_1} \,, \overbrace{ \frac{\partial \big\{ \text{div}^-_S \vec{\Bg} \big\} }{\partial \Bg_{s'}} }^{ = -\partial_{s'}^+ \delta } \Big\rangle_{\ell_2} }
                 _{ = -\partial_{s'}^+ \boldsymbol{\lambda_1} }
 - \alpha \mu \underbrace{ \frac{\partial \big\{ \text{div}^-_S \vec{\Bg} \big\} }{\partial \Bg_{s'}} }_{= -\partial^+_{s'} \delta}
 \Big[ \Bv - \mu \text{div}^-_S \vec{\Bg} \Big]
 + 2 \boldsymbol{\lambda_2} \cdot^\times \Bg_{s'}
 \\ \Leftrightarrow~ &
 - \mu \partial_{s'}^+ \Big[ \alpha \mu \text{div}^-_S \vec{\Bg} - \boldsymbol{\lambda_1} - \alpha \Bv \Big]
 + 2 \boldsymbol{\lambda_2} \cdot^\times \Bg_{s'} = 0
 \,,~ s' = 0, \ldots, S-1
 \\ \Leftrightarrow~ & 
 - \mu \nabla_S^+ \Big[ \alpha \mu \text{div}^-_S \vec{\Bg} - \boldsymbol{\lambda_1} - \alpha \Bv \Big]
 + 2 \boldsymbol{\lambda_2} \cdot^\times \vec{\Bg} = 0 \,.
\end{align*}
As in (\ref{eq:KKT:1}), with $\lambda_2[\Bk] > 0\,, \Bk \in \Omega$, its element form is 
\begin{align} \label{eq:KKT:3}
 &- \mu \nabla_S^+ \Big[ \alpha \mu \text{div}^-_S \vec{g}[\Bk] - \lambda_1[\Bk] - \alpha v[\Bk] \Big]
 + 2 \lambda_2[\Bk] \vec{g}[\Bk] = 0
 \\ \Leftrightarrow~ &
 \lambda_2[\Bk] = \frac{\mu}{2} \abs{\nabla^+_S \Big[ \alpha \mu \text{div}^-_S \vec{g}[\Bk] - \lambda_1[\Bk] - \alpha v[\Bk] \Big]} \,.   
 \notag
\end{align}
Thus, we rewrite (\ref{eq:KKT:3}) with $\Bk \in \Omega$ as
\begin{align*}
 &- \nabla_S^+ \Big[ \alpha \mu \text{div}^-_S \vec{g}[\Bk] - \lambda_1[\Bk] - \alpha v[\Bk] \Big]
 + \abs{\nabla^+_S \Big[ \alpha \mu \text{div}^-_S \vec{g}[\Bk] - \lambda_1[\Bk] - \alpha v[\Bk] \Big]} \vec{g}[\Bk] = 0 \,.
\end{align*}
Applying gradient descent and the fixed point algorithm, we have
\begin{align*}
 &- \nabla_S^+ \Big[ \alpha \mu \text{div}^-_S \vec{g}^{(t-1)}[\Bk] - \lambda_1[\Bk] - \alpha v[\Bk] \Big]
 + \abs{\nabla^+_S \Big[ \alpha \mu \text{div}^-_S \vec{g}^{(t-1)}[\Bk] - \lambda_1[\Bk] - \alpha v[\Bk] \Big]} \vec{g}^{(t)}[\Bk] 
 \\
 &= - \frac{ \vec{g}^{(t)}[\Bk] - \vec{g}^{(t-1)}[\Bk] }{ \tau } \,,
\end{align*}
the solution of the $\vec{\Bg}$-problem in matrix form is
\begin{align*}
 \vec{\Bg}^{(t)}
 = \frac{ \vec{\Bg}^{(t-1)} + \tau \nabla_S^+ \Big[ \alpha \mu \text{div}^-_S \vec{\Bg}^{(t-1)} - \boldsymbol{\lambda_1} - \alpha \Bv \Big] }
        { 1 + \tau \abs{\nabla^+_S \Big[ \alpha \mu \text{div}^-_S \vec{\Bg}^{(t-1)} - \boldsymbol{\lambda_1} - \alpha \Bv \Big]} } 
 \,,~ t = 1, \ldots .
\end{align*}

\noindent {\bfseries Update the Lagrange multiplier $\boldsymbol{\lambda_1} \in X$:}
\begin{align*}
 \boldsymbol{\lambda}_{\boldsymbol 1}^{(t)} = \boldsymbol{\lambda}_{\boldsymbol 1}^{(t-1)}
 + \alpha \Big[ \Bv - \mu \text{div}^-_S \vec{\Bg} \Big] \,.
\end{align*}

\end{proof}


\begin{prop} \label{prop:combinedmodel:pproblem}
 The $\vec{\Bp}$-problem in the SHT model is
\begin{align*}
 \min_{\vec{\Bp} \in X^N} \left\{
 \mu_3 \sum_{n=1}^N \norm{\nabla^+_M \Bp_n}_{\ell_1}    
 + \frac{\mu_4}{2} \norm{ \Bu - \sum_{n=1}^N c_n \Bp_n }^2_{\ell_2}
 \text{  s.t.  } 
 \sum_{n=1}^N \Bp_n = 1,
 p_n [\Bk] > 0, n = 1, \ldots, N, \Bk \in \Omega   
 \right\}.
\end{align*}
At iteration $t$, we have

\noindent {\bfseries The primal variable:}
\begin{align*}
 \Bp_n^{(t)} &= \frac{\displaystyle \exp\left\{ -\frac{1}{\xi} \Big[ \text{div}^-_M \vec{\Bq}_n^{(t-1)} 
                + \frac{\mu_4}{2 \mu_3} \big( \Bu - c_n \big)^{\cdot 2}
                \Big] \right\} }
              {\displaystyle \sum_{i=1}^N \exp \left\{ - \frac{1}{\xi} \Big[ \text{div}^-_M \vec{\Bq}_i^{(t-1)} 
                + \frac{\mu_4}{2 \mu_3} \big( \Bu - c_i \big)^{\cdot 2} \Big] \right\} }
 \\
 &= \frac{\displaystyle \exp\left\{ -\frac{1}{\xi} \Big[ -\sum_{m=0}^{M-1} \Big[ \cos\left(\frac{\pi m}{M}\right) \Bq_{nm}^{(t-1)} \BDn + \sin\left(\frac{\pi m}{M}\right) \BDmT \Bq_{nm}^{(t-1)} \Big] 
          + \frac{\mu_4}{2 \mu_3} \big( \Bu - c_n \big)^{\cdot 2}
          \Big] \right\} }
         {\displaystyle \sum_{i=1}^N \exp \left\{ - \frac{1}{\xi} \Big[ -\sum_{m=0}^{M-1} \Big[ \cos\left(\frac{\pi m}{M}\right) \Bq_{im}^{(t-1)} \BDn + \sin\left(\frac{\pi m}{M}\right) \BDmT \Bq_{im}^{(t-1)} \Big] 
          + \frac{\mu_4}{2 \mu_3} \big( \Bu - c_i \big)^{\cdot 2} \Big] \right\} }
 \,,~ n = 1, \ldots, N     
\end{align*}

\noindent {\bfseries The dual variable:} 
\begin{align*}  
 \vec{\Bq}_n^{(t)} &= \frac{ \vec{\Bq}_n^{(t-1)} + \tau \nabla^+_M \Bp_n^{(t)} }{ 1 + \tau \abs{\nabla^+_M \Bp_n^{(t)}} }
 \,,~~ a = 1, \ldots, N \,,
 \\
 \text{or}
 \\
 \Bq_{nm}^{(t)} &= \frac{\displaystyle \Bq_{nm}^{(t-1)} + \tau \left[ \cos\left(\frac{\pi m}{M}\right) \Bp_n^{(t)} \BDnT + \sin\left(\frac{\pi m}{M}\right) \BDm \Bp_n^{(t)} \right] }
                        {\displaystyle 1 + \tau \left[ \sum_{m=0}^{M-1} \left[ \cos\left(\frac{\pi m}{M}\right) \Bp_n^{(t)} \BDnT + \sin\left(\frac{\pi m}{M}\right) \BDm \Bp_n^{(t)} \right]^{.2} \right]^{.\frac{1}{2}} }
 \,,~~ n = 1, \ldots, N \,,~ m = 0, \ldots, M-1 
\end{align*} 
\end{prop}

\begin{proof}

This subproblem is solved by a smoothed primal-dual model \cite{BaeYuanTai2010} and Chambolle's projection \cite{Chambolle2004} for 
the dual variable (gradient descent of the Euler Lagrange equation and fixed point algorithm).
In order to make the paper self-contained, the extension of \cite{BaeYuanTai2010} to the multi-directional case is provided here by
introducing a primal model, a primal-dual model, a dual model and then a smoothed dual model
and a smoothed primal-dual model.

\noindent {\bfseries The primal model:}

Given the relaxed version on the binary set, the ``$\vec{\Bp}$-problem'' can be rewritten as the convex minimization
\begin{align}  \label{eq:pproblem:primal}
 \min_{\vec{\Bp} \in \mathcal Q_+} \bigg\{ \cL^\text{P}(\vec{\Bp}) = 
 \sum_{n=1}^N \norm{\nabla^+_M \Bp_n}_{\ell_1}    
 + \frac{\mu_4}{2 \mu_3} \underbrace{ \sum_{n=1}^N \Big\langle \big( \Bu - c_n \big)^{\cdot 2} \,, \Bp_i \Big\rangle_{\ell_2} }_{ = \norm{ \Bu - \sum_{n=1}^N c_n \Bp_n }^2_{\ell_2} }
 \bigg\}
\end{align}
with the convex set
\begin{align*}
 \mathcal Q_+ = \left\{ \vec{\Bp} = \big[ \Bp_n \big]_{n=1}^N \in X^N ~:~ 
 \sum_{n=1}^N \Bp_n = 1 \,, p_n[\Bk] > 0 \,, n = 1, \ldots, N \,, \Bk \in \Omega
 \right\} \,.
\end{align*}

\noindent {\bfseries The primal-dual model:}

From Lemma \ref{lem:DTVDGnorm:dualpair}, given a convex set of the DTV-norm
\begin{equation*}
  K_M(1) = \left\{ \vec{\Bq}_n = \big[ \Bq_{nm} \big]_{m=0}^{M-1} \in X^M ~:~ 
 \abs{\vec{\Bq}_n} = \sqrt{ \sum_{m=0}^{M-1} \Bq_{nm}^{\cdot 2} } \leq 1
 \right\}
\end{equation*}
with the notation $\vec{\Bq} = \big[ \vec{\Bq}_n \big]^N_{n=1} = \big[ \Bq_{nm} \big]_{n=[1,N]}^{m =[0,M-1]} \in X^{N M}$,
the primal-dual model of (\ref{eq:pproblem:primal}) (a separable problem) can be written as
\begin{align*}  
 \cL^\text{P}(\vec{\Bp}) &= \sum_{n=1}^N \max_{\vec{\Bq}_n \in K_M(1)} \Big\langle  \Bp_n \,, \text{div}^-_M \vec{\Bq}_n \Big\rangle_{\ell_2}
 + \frac{\mu_4}{2 \mu_3} \sum_{n=1}^N \Big\langle \Bp_n \,, \big( \Bu - c_n \big)^{.2} \Big\rangle_{\ell_2}
 \\
 &= \max_{\vec{\Bq} \in \big[ K_M(1) \big]^N}
 \bigg\{ \sum_{n=1}^N \Big\langle \Bp_n \,, \text{div}^-_M \vec{\Bq}_n
 + \frac{\mu_4}{2 \mu_3} \big( \Bu - c_n \big)^{\cdot 2}
 \Big\rangle_{\ell_2} 
 \bigg\}
 \,.
\end{align*}
From (\ref{eq:pproblem:primal}) and the Min-Max theorem, the primal $(\vec{\Bp})$ - dual $(\vec{\Bq})$ model is 
\begin{align} \label{eq:pproblem:primaldual}
 &\max_{\vec{\Bq} \in \big[ K_M(1) \big]^N}
 \min_{\vec{\Bp} \in \mathcal Q_+} \bigg\{
 \cL^\text{PD}(\vec{\Bq}, \vec{\Bp}) = \sum_{n=1}^N \Big\langle \Bp_n \,, \text{div}^-_M \vec{\Bq}_n
 + \frac{\mu_4}{2 \mu_3} \big( \Bu - c_n \big)^{\cdot 2}
 \Big\rangle_{\ell_2}
 \bigg\} \,.
\end{align}

\noindent {\bfseries The dual model:}
We first recall a result from \cite{BaeYuanTai2010}: 
\begin{Lem}
 Given a convex set 
 \begin{align*}
  &\Delta_+ = \Big\{ \big( v_1(x), \ldots, v_N(x) \big) \Big\} ~:~ 
  \sum_{n=1}^N v_n(x) = 1 \,, v_n(x) \geq 0 \,, n = 1, \ldots, N \,, \Bk \in \Omega \Big\} \,,
 \end{align*} 
 we have 
 \begin{align*}
  \min_{(v_1, \ldots, v_N) \in \Delta_+} \sum_{n=1}^N v_n q_n = \min \big( q_1, \ldots, q_N \big) \,.
 \end{align*}
 
\end{Lem}
Since the pixels are assumed to be independent, the primal-dual model (\ref{eq:pproblem:primaldual}) can be rewritten as
\begin{align*} 
 &\max_{\vec{\Bq} \in \big[ K_M(1) \big]^N}
 \min_{\vec{\Bp} \in \mathcal Q_+} \bigg\{
 \cL^\text{PD}(\vec{\Bq}, \vec{\Bp}) = \sum_{\Bk \in \Omega} \sum_{n=1}^N p_n[\Bk] 
 \Big[ \text{div}^-_M \vec{q}_n[\Bk] + \frac{\mu_4}{2 \mu_3} \big( u[\Bk] - c_n \big)^2 \Big]
 \bigg\}
 \\
 &= \max_{\vec{\Bq} \in \big[ K_M(1) \big]^N}
 \sum_{\Bk \in \Omega}
 \underbrace{
 \min_{\vec{p}[\Bk] \in \mathcal Q_{+\Bk}} \bigg\{
 \cL^\text{PD}(\vec{q}[\Bk], \vec{p}[\Bk]) = \sum_{n=1}^N p_n[\Bk] 
 \Big[ \text{div}^-_M \vec{q}_n[\Bk] + \frac{\mu_4}{2 \mu_3} \big( u[\Bk] - c_n \big)^2 \Big]
 \bigg\}
 }_{ \displaystyle =
 \min_{1\leq n \leq N} \bigg\{  
 \text{div}^-_M \vec{q}_1[\Bk] + \frac{\mu_4}{2 \mu_3} \big( u[\Bk] - c_1 \big)^{2} \,,
 \ldots \,, 
 \text{div}^-_M \vec{q}_N[\Bk] + \frac{\mu_4}{2 \mu_3} \big( u[\Bk] - c_N \big)^{2}
 \bigg\}
 }
\end{align*}
with an element version at a pixel $\Bk$ of a convex set $Q_+$
\begin{align} \label{eq:convexset:Q+k}
 Q_{+\Bk} = \Big\{ \vec{p}[\Bk] = \big[ p_n[\Bk] \big]_{n=1}^N \in \mathbb R^N ~:~
 \sum_{n=1}^N p_n[\Bk] = 1 \,, p_n[\Bk] > 0 \,, n = 1, \ldots, N \,, \Bk \in \Omega
 \Big\} \,.
\end{align}
Thus, the dual problem is
\begin{align} \label{eq:pproblem:dual}
 \max_{\vec{\Bq} \in \big[ K_M(1) \big]^N} \bigg\{
 \cL^\text{D} (\vec{\Bq}) = 
 \sum_{\Bk \in \Omega} \min_{1 \leq n \leq N}
 \Big\{
 \text{div}^-_M \vec{q}_1[\Bk] 
 + \frac{\mu_4}{2 \mu_3} \big( u[\Bk] - c_1 \big)^2 \,, \ldots \,,  
 \text{div}^-_M \vec{q}_N[\Bk] + \frac{\mu_4}{2 \mu_3} \big( u[\Bk] - c_N \big)^2 
 \Big\}  
 \bigg\} \,.
\end{align}

\noindent {\bfseries The smoothed dual model:}
\begin{Lem} \label{lem:smoothfunction}
Given $\vec{u} = \big[ u_n \big]_{n=1}^N$,
the non-smooth function $\max_{1 \leq n \leq N} u_i$ is defined as 
the asympototic function of a proper convex function 
$g(\vec{u}) = \log \Big[ \sum_{n=1}^N e^{u_n} \Big]$ as
\begin{align*}
 \max_{1 \leq n \leq N} u_n &= \lim_{\xi \rightarrow 0^+} \xi g \big(\frac{\vec{u}}{\xi}\big) 
 = \lim_{\xi \rightarrow 0^+} \xi \log \Big[ \sum_{n=1}^N e^{\frac{u_n}{\xi}} \Big]
 \\
 &\approx~ \xi \log \Big[ \sum_{n=1}^N e^{\frac{u_n}{\xi}} \Big]
\end{align*}
for some small constant $\xi >0$.
\end{Lem}
From (\ref{eq:pproblem:dual}), we have
\begin{align*}
 \cL^\text{D}(\vec{\Bq}) &= - \sum_{\Bk \in \Omega} \max_{1 \leq n \leq N} \bigg\{
 - \text{div}^-_M \vec{q}_1[\Bk] 
 - \frac{\mu_4}{2 \mu_3} \big( u[\Bk] - c_1 \big)^2 \,, \ldots \,,  
 - \text{div}^-_M \vec{q}_N[\Bk] 
 - \frac{\mu_4}{2 \mu_3} \big( u[\Bk] - c_N \big)^2  
 \bigg\}
 \\
 &\approx - \xi \sum_{\Bk \in \Omega} \log \bigg[
 \sum_{n=1}^N \exp \Big\{ \frac{1}{\xi} \Big[ - \text{div}^-_M \vec{q}_n[\Bk] 
 - \frac{\mu_4}{2 \mu_3} \big( u[\Bk] - c_n \big)^2
 \Big\} \bigg]
 ~:=~ \cL^\text{D}_{\xi > 0} (\vec{\Bq}) \,.
\end{align*}
Thus, a smoothed dual model is
\begin{align} \label{eq:pproblem:smootheddual}
 &\max_{\vec{\Bq} \in \big[ K_M(1) \big]^N} \bigg\{ \cL^\text{D}_{\xi > 0} (\vec \Bq) = 
 -\xi \sum_{\Bk \in \Omega}
 \log \Big[ \sum_{n=1}^N \exp \Big\{ -\frac{1}{\xi} \Big[ \text{div}^-_M \vec{q}_n[\Bk]  
 + \frac{\mu_4}{2 \mu_3} \big( u[\Bk] - c_n \big)^2
 \Big] \Big\}  \Big]
 \bigg\} \,.
\end{align}

\noindent {\bfseries The smoothed primal-dual model:}
\begin{Lem}
For any given $\Bu \in \mathcal Q_+$ and $h \in \mathbb R^N$
 \begin{equation*}
  \log \sum_{n=1}^N \mu_n e^{h_n} = \max_{u \in \mathcal Q_+} 
  \left\{ \langle u, h \rangle_{\ell_2} - \sum_{n=1}^N u_n \log \frac{u_n}{\mu_n}
  \right\} \,.
 \end{equation*}
\end{Lem}
Given the convex set $Q_{+\Bk}$ in (\ref{eq:convexset:Q+k}) and because the pixels are assumed to be independent, 
the energy function of a smoothed dual model (\ref{eq:pproblem:smootheddual}) can be rewritten as
\begin{align*}
 \cL^\text{D}_{\xi > 0} (\vec{\Bq}) &= -\xi \sum_{\Bk \in \Omega} \max_{\vec{p}_N[\Bk] \in \mathcal Q_{+\Bk}}
 \left\{ -\frac{1}{\xi}  \Big\langle \vec{p}_N[\Bk] \,, \text{div}^-_M \vec{q}[\Bk] 
 + \frac{\mu_4}{2 \mu_3} \big( u[\Bk] - \vec{c} \big)^2 \Big\rangle_{\ell_2}
 - \sum_{n=1}^N p_n[\Bk] \log p_n[\Bk]
 \right\}
 \\
 &= \min_{\vec{\Bp} \in \mathcal Q_+} \bigg\{
 \sum_{n=1}^N 
 \underbrace{
 \sum_{\Bk \in \Omega} p_n[\Bk] \Big[ \text{div}^-_M \vec{q}_n[\Bk] 
 + \frac{\mu_4}{2 \mu_3} \big( u[\Bk] - c_n \big)^2 \Big]
 }_{ = \big\langle \Bp_n \,, \text{div}^-_M \vec{\Bq}_n + \frac{\mu_4}{2 \mu_3} ( \Bu - c_n )^{\cdot 2} \big\rangle_{\ell_2} }
 + \xi \sum_{n=1}^N \underbrace{ \sum_{\Bk \in \Omega} p_n[\Bk] \log p_n[\Bk] }_{ = \big\langle \Bp_n \,, \log \Bp_n \big\rangle_{\ell_2} }
 \bigg\} \,.
\end{align*}
From (\ref{eq:pproblem:smootheddual}), we have the smoothed primal-dual model 
\begin{align}  \label{eq:pproblem:smoothedprimaldual}
 \max_{\vec{\Bq} \in \big[ K_M(1) \big]^N} 
 \underbrace{
 \min_{\vec{\Bp} \in \mathcal Q_+} \bigg\{
 \cL^\text{PD}_{\xi>0} (\vec{\Bq}, \vec{\Bp}) = \sum_{n=1}^N \Big\langle \Bp_n \,, \text{div}^-_M \vec{\Bq}_n 
 + \frac{\mu_4}{2 \mu_3} \big( \Bu - c_n \big)^{\cdot2} \Big\rangle_{\ell_2}
 + \xi \sum_{n=1}^N \big\langle \Bp_n \,, \log \Bp_n \big\rangle_{\ell_2}
 \bigg\}
 }_{\displaystyle = \cL^\text{D}_{\xi>0} (\vec{\Bq}) := -\xi \sum_{\Bk \in \Omega} \log \left[ 
 \sum_{n=1}^N \exp \left\{ -\frac{1}{\xi} \Big[ \text{div}^-_M \vec{q}_n[\Bk] 
 + \frac{\mu_4}{2 \mu_3} \big( u[\Bk] - c_n \big)^2
 \Big] \right\} \right]
 }
 \,.
\end{align}

\noindent {\bfseries a) Primal $(\vec{\Bp})$-problem:}
The primal problem of (\ref{eq:pproblem:smoothedprimaldual}) over a convex set 
$\mathcal Q_+ = \Big\{ \vec{\Bp} = \big[ \Bp_n \big]_{n=1}^N \in X^N ~:~
 \sum_{n=1}^N \Bp_n = 1 \,, p_n[\Bk] \geq 0 \,, n = 1, \ldots, N \,, \Bk \in \Omega
 \Big\} 
$
is
\begin{align*}
 \min_{\vec{\Bp} \in \mathcal Q_+} \cL^\text{PD}_{\xi>0} (\vec{\Bp}) \,.
\end{align*}
Applying ALM for the equality constraint, the primal $(\vec{\Bp})$-problem can be rewritten as
\begin{align*}
 \min_{0<\vec{\Bp} \in X^N} \max_{\boldsymbol \lambda \in X} \Bigg\{ 
 &\cL^\text{PD}_{\xi>0} (\vec{\Bp} \,; \boldsymbol{\lambda}) =
 \sum_{n=1}^N \big\langle \Bp_n \,, \text{div}^-_M \vec{\Bq}_n 
 + \frac{\mu_4}{2 \mu_3} \big( \Bu - c_n \big)^{\cdot 2} \big\rangle_{\ell_2}
 + \xi \sum_{n=1}^N \big\langle \Bp_n \,, \log \Bp_n \big\rangle_{\ell_2} 
 \\& \qquad \qquad \qquad \qquad \quad
 + \big\langle \boldsymbol{\lambda} \,, \sum_{n=1}^N \Bp_n - 1 \big\rangle_{\ell_2}
 + \frac{\gamma}{2} \norm{\sum_{n=1}^N \Bp_n - 1}^2_{\ell_2}
 \Bigg\} \,.
\end{align*}
Due to its separability, we consider the problem at $n=a \,, a = 1, \ldots, N$.
The Euler-Lagrange equation is
\begin{align*}  
\begin{cases}
 \displaystyle
 0 = \frac{ \partial \cL^\text{PD}_{\xi>0} (\vec{\Bp} \,, \boldsymbol{\lambda}) }{ \partial \boldsymbol{\lambda} }
   = \sum_{n=1}^N \Bp_n - 1    &  \qquad \text{(a)}
 \\ 
 \displaystyle
 0 = \frac{ \partial \cL^\text{PD}_{\xi>0} (\vec{\Bp} \,, \boldsymbol{\lambda}) }{ \partial \Bp_a }
 = \text{div}^-_M \vec{\Bq}_a + \frac{\mu_4}{2 \mu_3} \big( \Bu - c_a \big)^{\cdot 2}
 + \xi (1 + \log \Bp_a) + \boldsymbol{\lambda} + \gamma \underbrace{ \Big[ \sum_{n=1}^N \Bp_n - 1 \Big] }_{ = 0 }
 & \qquad \text{(b)}
\end{cases}
\end{align*}
\begin{align*}
 \text{(b)} &~\Leftrightarrow~ \Bp_a = \frac{1}{e^{\frac{\boldsymbol \lambda}{\xi}} + 1}
 \exp \left\{ -\frac{1}{\xi} \Big[ \text{div}^-_M \vec{\Bq}_a 
 + \frac{\mu_4}{2 \mu_3} \big( \Bu - c_a \big)^{\cdot2}
 \Big] \right\} \,,~ a = 1, \ldots, N
 \\
 \text{(a)} &~\Leftrightarrow~ \sum_{n=1}^N \exp \left\{
 - \frac{1}{\xi} \Big[ \text{div}^-_M \vec{\Bq}_n 
 + \frac{\mu_4}{2 \mu_3} \big( \Bu - c_n \big)^{\cdot 2}
 \Big] \right\} = e^{\frac{\boldsymbol{\lambda}}{\xi} + 1} \,.
\end{align*}
Thus, the primal variable is
\begin{align*}
 \Bp_a &= \frac{\displaystyle \exp\Big\{ -\frac{1}{\xi} \Big[ \text{div}^-_M \vec{\Bq}_a 
                + \frac{\mu_4}{2 \mu_3} \big( \Bu - c_a \big)^{\cdot 2}
                \Big] \Big\} }
              {\displaystyle \sum_{n=1}^N \exp \Big\{ - \frac{1}{\xi} \Big[ \text{div}^-_M \vec{\Bq}_n 
                + \frac{\mu_4}{2 \mu_3} \big( \Bu - c_n \big)^{\cdot 2} \Big] \Big\} }
 \,,~ a = 1, \ldots, N \,.    
\end{align*}

\noindent {\bfseries b) The Dual $(\vec{\Bq})$-problem:}
Given the convex set
$
K_M(1) = \Big\{ \vec{\Bq}_n = \big[ \Bq_{nm} \big]_{m=0}^{M-1} \in X^M ~:~ \abs{\vec{\Bq}_n} = \sqrt{\sum_{m=0}^{M-1} \Bq_{nm}^{\cdot 2}} \leq 1\Big\} 
$, 
the dual problem of (\ref{eq:pproblem:smoothedprimaldual}) is
\begin{align*}
 \max_{\vec{\Bq} \in \big[ K_M(1) \big]^N} \bigg\{ 
 \cL^\text{D}_{\xi>0} (\vec{\Bq}) = - \xi \sum_{\Bk \in \Omega} \log \Big[ \sum_{n=1}^N \exp \Big\{ 
 -\frac{1}{\xi} \Big[ \text{div}^-_M \vec{q}_n[\Bk] 
 + \frac{\mu_4}{2\mu_3} \big( u[\Bk] - c_n \big)^2
 \Big]
 \Big\}
 \bigg\} \,.
\end{align*}
Due to the pixel independence, we consider the problem at $\Bk = \Bx \in \Omega \,, n=a, a = 1, \ldots, N$
and with the KKT condition (with the Lagrange multiplier $\lambda[\Bk] > 0$) for the inequality constraint, we have
\begin{align} \label{eq:KKTCondition:dualqproblem}
 \frac{ \partial \cL^\text{D}_{\xi>0} (\vec{\Bq}) }{\partial \vec{q}_a[\Bx]}
 + \lambda[\Bx] \underbrace{ \frac{\partial}{\partial \vec{q}_a[\Bx]} \Big\{ \abs{\vec{q}_a[\Bx]}^2 - 1 \Big\} }
                          _{ = 2 \vec{q}_a[\Bx] }
 = 0 \,,~ a = 1, \ldots, N \,, \Bx \in \Omega
\end{align}
with
\begin{align*}
 \frac{ \partial \cL^\text{D}_{\xi>0} (\vec{\Bq}) }{\partial \vec{q}_a[\Bx]} &= - \xi \frac{\partial}{\partial \vec{q}_a[\Bx]}
 \bigg\{ \log \Big[ \sum_{n=1}^N   \exp \Big\{ -\frac{1}{\xi}  \Big[ \text{div}^-_M \vec{q}_n[\Bx]
 + \frac{\mu_4}{2\mu_3} \big( u[\Bx] - c_n \big)^2
 \Big] \Big\} \Big] 
 \bigg\}
 \\
 &= -\xi (-\frac{1}{\xi}) \underbrace{ \frac{\partial}{\partial \vec{q}_a[\Bx]} \Big\{ \text{div}^-_M \vec{q}_a[\Bx] \Big\} }
                                    _{ = (\text{div}^-_M)^* \frac{\partial \vec{q}_a[\Bx]}{\partial \vec{q}_a[\Bx]} = (-\nabla^+_M) \delta[\Bx]} 
    \underbrace{
    \frac{ \exp \Big\{ -\frac{1}{\xi}
           \Big[ \text{div}^-_M \vec{q}_a[\Bx] 
           + \frac{\mu_4}{2 \mu_3} \big( u[\Bx] - c_a \big)^2 \Big]
           \Big\} }
         { \displaystyle \sum_{n=1}^N \exp \Big\{ -\frac{1}{\xi} \Big[ \text{div}^-_M \vec{q}_n[\Bx]
           + \frac{\mu_4}{2 \mu_3} \big( u[\Bx] - c_n \big)^2
           \Big] \Big\} }
    }_{ = p_a[\Bx] }
 \\
 &= -\nabla^+_M p_a[\Bx] \,.
\end{align*}
As in (\ref{eq:KKT:1}), given $a = 1, \ldots, N \,, \Bx \in \Omega$, 
we rewrite the KKT condition (\ref{eq:KKTCondition:dualqproblem}) in vector form as
\begin{align*}  
 -\nabla^+_M p_a[\Bx] + 2 \lambda[\Bx] \vec{q}_a[\Bx] = 0
 ~\Leftrightarrow~
 \lambda[\Bx] = \frac{1}{2} \abs{\nabla_M^+ p_a[\Bx]} 
\end{align*}
and thus
\begin{align*}
 -\nabla_M^+ p_a[\Bx] ~+~ \abs{\nabla_M^+ p_a[\Bx]} \vec{q}_a[\Bx] ~=~ 0 \,.
\end{align*}
Applying gradient descent and the fixed point algorithm with $t = 1, \ldots$, we have
\begin{align*}  
 &-\nabla_M^+ p_a[\Bx] ~+~ \abs{\nabla_M^+ p_a[\Bx]} \vec{q}_a^{(t)}[\Bx] ~=~ - \frac{ \vec{q}_a^{(t)}[\Bx] - \vec{q}_a^{(t-1)}[\Bx] }{ \tau } \,,
\end{align*}
and the solution in matrix form with 
$\displaystyle  \vec{\Bq}_a = \big[ \Bq_{am} \big]_{m=0}^{M-1}$ and
$\nabla^+_M \Bp_a = \big[ \partial^+_m \Bp_a \big]_{m=0}^{M-1}$ is given by
\begin{align}  \label{eq:SmoothedDualModel:qsolution:proof}
 \vec{\Bq}_a^{(t)} = \frac{ \vec{\Bq}_a^{(t-1)} + \tau \nabla^+_M \Bp_a }{ 1 + \tau \abs{\nabla^+_M \Bp_a} }
 \,,~~ a = 1, \ldots, N \,.
\end{align}

\end{proof}


\begin{prop} \label{prop:bilevel:smoothedprimaldualproblem}
 Given a primal-dual model (\ref{eq:BiLevelMinimization:2:step2:primaldual}) in the bilevel-SHT, 
 a smoothed primal-dual model is defined as
 \begin{equation} \label{eq:BiLevelMinimization:2:step2:SmoothedPrimalDual1:proof}
  \max_{ \vec{\Bq} \in \big[K_M(1)\big]^N }
  \underbrace{
  \min_{ \vec{\Bp} \in \mathcal Q_+} \bigg\{
  \mathcal L^\text{PD}_{\xi>0}(\vec{\Bp}; \vec{\Bq}) = 
  \sum_{n=1}^N \Big\langle \Bp_n \,, \frac{\mu_3}{2} \big(\Bu - c_n \big)^{.2} + \text{div}^-_M \vec{\Bq}_n \Big\rangle_{\ell_2}
  + \xi \sum_{n=1}^N \Big\langle \Bp_n \,, \log \Bp_n \Big\rangle_{\ell_2} 
  \bigg\}
  }_{\displaystyle = \cL^\text{D}_{\xi>0} (\vec{\Bq}) := -\xi \sum_{\Bk \in \Omega} \log \Big[ \sum_{n=1}^N
     \exp \Big\{ -\frac{1}{\xi} \Big[ \frac{\mu_3}{2} (u[\Bk] - c_n)^2 + \text{div}^-_M \vec{q}_n[\Bk] \Big] \Big\}
     \Big] } \,.
 \end{equation} 
At iteration $t$, the primal and dual solutions are given by
 \begin{align*}
  \Bp_n^{(t)} &= \frac{\displaystyle \exp\left[ - \frac{1}{\xi} \left[ \frac{\mu_3}{2} \big(\Bu - c_n \big)^{.2} + \text{div}^-_M \vec{\Bq}_n^{(t-1)} \right]\right] }
                  {\displaystyle \sum_{i=1}^N \exp\bigg[ - \frac{1}{\xi} \Big[ \frac{\mu_3}{2} \big(\Bu - c_i \big)^{.2} + \text{div}^-_M \vec{\Bq}_i^{(t-1)} \Big]\bigg]  }
  \\& 
  = \frac{\displaystyle \exp\left[ - \frac{1}{\xi} \left[ \frac{\mu_3}{2} \big(\Bu - c_n \big)^{.2} - \sum_{m=0}^{M-1} \left[ \cos\left(\frac{\pi m}{M}\right) \Bq_{nm}^{(t-1)} \BDn + \sin\left(\frac{\pi m}{M}\right) \BDmT \Bq_{nm}^{(t-1)} \right] \right]\right] }
         {\displaystyle \sum_{i=1}^N \exp\left[ - \frac{1}{\xi} \left[ \frac{\mu_3}{2} \big(\Bu - c_i \big)^{.2} - \sum_{m=0}^{M-1} \left[ \cos\left(\frac{\pi m}{M}\right) \Bq_{im}^{(t-1)} \BDn + \sin\left(\frac{\pi m}{M}\right) \BDmT \Bq_{im}^{(t-1)} \right] \right]\right]  }
  \,,~~ n = 1, \ldots, N \,,
 \end{align*}
 and
 \begin{align*}
  \vec{\Bq}_n^{(t)} = \frac{ \vec{\Bq}_n^{(t-1)} + \tau \nabla^+_M \Bp_n^{(t)} }{ 1 + \tau \abs{\nabla^+_M \Bp_n^{(t)}} }
  \,,~~ n = 1, \ldots, N \,,
 \end{align*}
 or 
 \begin{align*}
  \Bq_{nm}^{(t)} = \frac{\displaystyle \Bq_{nm}^{(t-1)} + \tau \left[ \cos\left(\frac{\pi m}{M}\right) \Bp_n^{(t)} \BDnT + \sin\left(\frac{\pi m}{M}\right) \BDm \Bp_n^{(t)} \right] }
                        {\displaystyle 1 + \tau \left[ \sum_{m=0}^{M-1} \left[ \cos\left(\frac{\pi m}{M}\right) \Bp_n^{(t)} \BDnT + \sin\left(\frac{\pi m}{M}\right) \BDm \Bp_n^{(t)} \right]^{.2} \right]^{.\frac{1}{2}} }
  \,,~~ n = 1, \ldots, N \,,~ m = 0, \ldots, M-1.
 \end{align*}
  
\end{prop}


\begin{proof}
As in Section \ref{sec:CombinedModel}, with $\vec{\Bp} \in \mathcal Q_+$,
 a dual-primal model (\ref{eq:BiLevelMinimization:2:step2:primaldual}) is recast as the dual model
 \begin{equation} \label{eq:BiLevelMinimization:2:step2:dual:proof}
  \max_{ \vec{\Bq} \in \big[K_M(1)\big]^N } \bigg\{ \mathcal L^\text{D}(\vec{\Bq}) =
  \sum_{\Bk \in \Omega} \min \Big\{ 
  \frac{\mu_3}{2} \big(u[\Bk] - c_1 \big)^2 + \text{div}^-_M \vec{q}_1[\Bk], \ldots,
  \frac{\mu_3}{2} \big(u[\Bk] - c_N \big)^2 + \text{div}^-_M \vec{q}_N[\Bk] \Big\}  
  \bigg\} \,.
 \end{equation}
 Due to the min operator in the energy function $\mathcal L_\text{D}(\cdot)$,
 the minimization in (\ref{eq:BiLevelMinimization:2:step2:dual:proof}) is a non-smooth 
 dual model whose smoothed version (called a smoothed dual model) is defined as 
 \begin{align}  \label{eq:BiLevelMinimization:2:step2:SmoothedDual:proof}
  \max_{ \vec{\Bq} \in \big[K_M(1)\big]^N } \left\{
  \mathcal L^\text{D}_{\xi>0}(\vec{\Bq}) =
  - \xi \sum_{\Bk \in \Omega} \log \left[ \sum_{n=1}^N \exp 
  \left\{ \frac{ -\frac{\mu_3}{2} \big(u[\Bk] - c_n \big)^2 - \text{div}^-_M \vec{q}_n[\Bk] }{\xi} \right\} \right]
  \right\} \,.
 \end{align} 
 Note that $\mathcal L^\text{D}(\vec{\Bq}) \approx \mathcal L^\text{D}_{\xi>0}(\vec{\Bq})$.
 By the smooth log-sum function in convex analysis \cite{Rockafellar1970} and the pixel independence,
 a smoothed dual energy function in (\ref{eq:BiLevelMinimization:2:step2:SmoothedDual:proof}) can be rewritten as
 \begin{align*} 
  \mathcal L^\text{D}_{\xi>0}(\vec{\Bq}) 
  &= - \xi \sum_{\Bk \in \Omega} \max_{\vec{p}[\Bk] \in \mathcal Q_{+\Bk}} \Big\{
  - \frac{1}{\xi} \Big\langle \vec{p}[\Bk] \,, \frac{\mu_3}{2} \big(u[\Bk] - \vec{c} \big)^2 + \text{div}^-_M \vec{q}[\Bk] \Big\rangle_{\ell_2}
  - \sum_{n=1}^N p_n[\Bk] \log p_n[\Bk]
  \Big\} 
  \\
  &= \min_{\vec{\Bp} \in \mathcal Q_+} \bigg\{
  \sum_{n=1}^N \underbrace{ \sum_{\Bk \in \Omega} p_n[\Bk] \Big[ \frac{\mu_3}{2} \big(u[\Bk] - c_n \big)^2 + \text{div}^-_M \vec{q}_n[\Bk] \Big] }
                         _{ = \big\langle \Bp_n \,, \frac{\mu_3}{2} (\Bu - c_n)^{.2} + \text{div}^-_M \vec{\Bq}_n \big\rangle_{\ell_2} }
  + \xi \sum_{n=1}^N \underbrace{ \sum_{\Bk \in \Omega} p_n[\Bk] \log p_n[\Bk] }
                             _{ = \big\langle \Bp_n \,, \log \Bp_n \big\rangle_{\ell_2} }
  \bigg\}   \,.
 \end{align*}
 Thus, we rewrite (\ref{eq:BiLevelMinimization:2:step2:SmoothedDual:proof}) as a smoothed primal-dual model
 \begin{equation} \label{eq:BiLevelMinimization:2:step2:SmoothedPrimalDual:proof}
  \max_{ \vec{\Bq} \in \big[K_M(1)\big]^N }
  \bigg\{
  \mathcal L^\text{D}_{\xi>0}(\vec{\Bq}) =
  \min_{ \vec{\Bp} \in \mathcal Q_+} 
  \Big\{ \mathcal L^\text{PD}_{\xi>0}(\vec{\Bp}; \vec{\Bq}) = 
  \sum_{n=1}^N \Big\langle \Bp_n \,, \frac{\mu_3}{2} \big(\Bu - c_n \big)^{.2} + \text{div}^-_M \vec{\Bq}_n \Big\rangle_{\ell_2}
  + \xi \sum_{n=1}^N \Big\langle \Bp_n \,, \log \Bp_n \Big\rangle_{\ell_2}  
  \Big\}
  \bigg\}
  \,.
 \end{equation}

\noindent {\bfseries The primal $\vec{\Bp}$-problem:}
 We apply ALM for the equality constraint $\sum_{n=1}^N \Bp_n = 1$ in a set $\mathcal Q_+$ to
 the primal $\vec{\Bp}$ problem of (\ref{eq:BiLevelMinimization:2:step2:SmoothedPrimalDual:proof}) as
 \begin{align} \label{eq:BiLevelMinimization:2:step2:SmoothedPrimalDual:pproblem}
  \min_{0 < \vec{\Bp} \in X^N} \max_{\boldsymbol{\lambda} \in X} \Bigg\{ \mathcal L^\text{PD}_{\xi>0} (\vec{\Bp}; \boldsymbol{\lambda}) = 
  &\sum_{n=1}^N \Big\langle \Bp_n \,, \frac{\mu_3}{2} \big(\Bu - c_n \big)^{.2} + \text{div}^-_M \vec{\Bq}_n \Big\rangle_{\ell_2}
  + \xi \sum_{n=1}^N \Big\langle \Bp_n \,, \log \Bp_n \Big\rangle_{\ell_2}    \notag
  \\&
  + \big\langle \boldsymbol{\lambda} \,, \sum_{n=1}^N \Bp_n - 1 \big\rangle_{\ell_2} + \frac{\gamma}{2} \norm{ \sum_{n=1}^N \Bp_n - 1 }^2_{\ell_2}
  \Bigg\}
 \end{align} 
 Due to the separability in $\vec{\Bp}$ and $\boldsymbol{\lambda}$, the Euler Lagrange equations of the minimization problem 
 (\ref{eq:BiLevelMinimization:2:step2:SmoothedPrimalDual:pproblem})
 w.r.t. a primal variable $\Bp_n \mid_{n=a} \,, a = 1, \ldots, N$  and the Lagrange multiplier $\boldsymbol \lambda$ are
 \begin{align}  \label{eq:BiLevelMinimization:2:step2:SmoothedPrimalDual:pproblem:solution}
  &\begin{cases}
   \displaystyle
   0 = \frac{\partial \mathcal L^\text{PD}_{\xi>0}( \vec{\Bp}; \boldsymbol{\lambda} )}{\partial \boldsymbol{\lambda}} = 
   \sum_{n=1}^N \Bp_n - 1   \notag
   \\ \displaystyle
   0 = \frac{\partial \mathcal L^\text{PD}_{\xi>0}( \vec{\Bp}; \boldsymbol{\lambda} )}{\partial \Bp_a}
   = \frac{\mu_3}{2} \big(\Bu - c_a \big)^{.2} + \text{div}^-_L \vec{\Bq}_a + \xi(1 + \log\Bp_a)
   + \boldsymbol{\lambda} + \gamma \Big[\underbrace{ \sum_{n=1}^N \Bp_n - 1 }_{= 0} \Big]  
  \end{cases}   \notag
  \\
  \Leftrightarrow~ &
  \Bp^*_a = \frac{\displaystyle \exp\bigg[ - \frac{1}{\xi} \Big[ \frac{\mu_3}{2} \big(\Bu - c_a \big)^{.2} + \text{div}^-_M \vec{\Bq}_a \Big]\bigg] }
                 {\displaystyle \sum_{n=1}^N \exp\bigg[ - \frac{1}{\xi} \Big[ \frac{\mu_3}{2} \big(\Bu - c_n \big)^{.2} + \text{div}^-_M \vec{\Bq}_n \Big]\bigg]  }
  \,,~ a = 1, \ldots, N.  
 \end{align}
 We observe that $\vec{\Bp} = \big[ \Bp_a \big]_{a=1}^N > 0$, i.e. $p_a[\Bk] > 0, \forall \Bk \in \Omega \,,~ a = 1, \ldots, N$.
 
\noindent {\bfseries The dual $\vec{\Bq}$-problem:}
 Due to the inequality constraint in a convex set $K_M(1)$,
 the KKT condition of the minimization in (\ref{eq:BiLevelMinimization:2:step2:SmoothedDual:proof})
 (or in a smoothed primal-dual model (\ref{eq:BiLevelMinimization:2:step2:SmoothedPrimalDual:proof})), 
 with given a Lagrange multiplier $\lambda[\Bx] > 0 \,, \forall \Bx \in \Omega$
 yields
 \begin{align}  \label{eq:KKTCondition}
  \frac{\partial \cL^D_{\xi>0} (\vec{\Bq})}{\partial \vec{q}_a[\Bx]} ~+~
  \lambda[\Bx] \underbrace{ \frac{\partial}{\partial \vec{q}_a[\Bx]} \bigg\{ \abs{\vec{q}_a[\Bx]}^2 - 1 \bigg\} }
                         _{\displaystyle = 2 \vec{q}_a[\Bx] }
  ~=~ 0 \,,~~ a = 1, \ldots, N \,, \Bx \in \Omega \,.
 \end{align}
Because of the pixel independence $\Bk \in \Omega$, in order to calculate $\frac{\partial \cL^D_{s>0} (\vec{\Bq})}{\partial \vec{\Bq}_a}$ (in matrix form),
 we consider this derivative in element form at $\Bk = \Bx \in \Omega$ and $n = a \,, a = 1, \ldots, N$:
 \begin{align*}
  \frac{\partial \cL^D_{\xi>0} (\vec{\Bq})}{\partial \vec{q}_a[\Bx]}  
  &= -\xi (-\frac{1}{\xi}) \underbrace{ \frac{\partial}{\partial \vec{q}_a[\Bx]} \Big\{ \text{div}^-_M \vec{q}_a[\Bx] \Big\} }
                                     _{ = \big( -\nabla_M^+ \big) \delta[\Bx] }
     \frac{\exp \Big\{ \frac{ - \frac{\mu_3}{2} \big(u[\Bx] - c_a \big)^2 - \text{div}^-_M \vec{q}_a[\Bx] }{\xi} \Big\} }
          {\displaystyle \sum_{n=1}^N \exp \Big\{ \frac{ -\frac{\mu_3}{2} \big(u[\Bx] - c_n \big)^2 - \text{div}^-_M \vec{q}_n[\Bx] }{\xi} \Big\} }
  \\&
  = -\nabla_M^+ \left\{
                \frac{\displaystyle \exp \bigg[ -\frac{1}{\xi} \Big[ \frac{\mu_3}{2} \big(u[\Bx] - c_a \big)^2 + \text{div}^-_M \vec{q}_a[\Bx] \Big] \bigg] 
                }
                {\displaystyle \sum_{n=1}^N \exp \bigg[ -\frac{1}{\xi} \Big[ \frac{\mu_3}{2} \big(u[\Bx] - c_n \big)^2 + \text{div}^-_M \vec{q}_n[\Bx] \Big] \bigg] }
                \right\}  
  \\&
  \stackrel{(\ref{eq:BiLevelMinimization:2:step2:SmoothedPrimalDual:pproblem:solution})}{=}
  -\nabla_M^+ p_a[\Bx] \,.
 \end{align*}
As in Section \ref{sec:CombinedModel} and \cite{Chambolle2004},
 the KKT condition (\ref{eq:KKTCondition}) (in vector form with $\vec{\Bq}_a = \big[ \Bq_{am} \big]_{m=1}^M$) 
can be rewritten as
 \begin{align*}
  -\nabla_M^+ p_a[\Bx] ~+~ \abs{\nabla_M^+ p_a[\Bx]} \vec{q}_a[\Bx] ~=~ 0 \,,~~ a = 1, \ldots, N 
 \end{align*}
for $\Bx \in \Omega$.  Applying gradient descent and the fixed point algorithm, we have
 \begin{align*}  
  -\nabla_M^+ p_a[\Bx] ~+~ \abs{\nabla_M^+ p_a[\Bx]} \vec{q}_a^{(t)}[\Bx] ~=~ - \frac{ \vec{q}_a^{(t)}[\Bx] - \vec{q}_a^{(t-1)}[\Bx] }{ \tau }   \,,
 \end{align*}
and the solution in matrix form with
 $\displaystyle  \vec{\Bq}_a = \big[ \Bq_{am} \big]_{m=0}^{M-1}$ and
 $\nabla^+_M \Bp_a = \big[ \partial^+_m \Bp_a \big]_{m=0}^{M-1}$ is given by
 \begin{align*}  
  \vec{\Bq}_a^{(t)} = \frac{ \vec{\Bq}_a^{(t-1)} + \tau \nabla^+_M \Bp_a }{ 1 + \tau \abs{\nabla^+_M \Bp_a} }
  \,,~~ a = 1, \ldots, N \,.
 \end{align*}

\end{proof}




\begin{algorithm}
\label{alg:2phasePiecewiseConst:1}
\caption{A two-phase piecewise constant \& texture segmentation}
\begin{algorithmic}
\small
\STATE
{
  {\bfseries I. Compute}
   $\Big( \big[ \Br_b^{(t)} \big]_{b=0}^{L-1} \,, \big[ \mathbf{w}_a^{(t)} \big]_{a=0}^{S-1} \,,
    \big[ \mathbf{g}_a^{(t)} \big]_{a=0}^{S-1} \,, \Bv^{(t)} \,, \Bu^{(t)} \,, \Beps^{(t)} \Big) \in X^{L+2S+3}
   $:
   \begin{align*}
    &\text{\bfseries 1.}~~ 
    c_1^{(t)} = \frac{ \sum_{\Bk \in \Omega} \big[ f[\Bk] - v^{(t-1)}[\Bk] - \epsilon^{(t-1)}[\Bk] + \frac{\lambda_4^{(t-1)}[\Bk]}{\beta_4} \big] p^{(t-1)}[\Bk] }{ \sum_{\Bk \in \Omega} p^{(t-1)}[\Bk] }   
    \\
    &\text{\bfseries 2.}~~ 
    c_2^{(t)} = \frac{ \sum_{\Bk \in \Omega} \big[ f[\Bk] - v^{(t-1)}[\Bk] - \epsilon^{(t-1)}[\Bk] + \frac{\lambda_4^{(t-1)}[\Bk]}{\beta_4} \big] \big[ 1 - p^{(t-1)}[\Bk] \big] }{ \sum_{\Bk \in \Omega} (1 - p^{(t-1)}[\Bk]) }    
    \\
    &\text{\bfseries 3.}~~ 
    \Br_a^{(t)} = \Shrink \Big( \sin\left(\frac{\pi a}{L}\right) \BDm \Bp^{(t-1)} + \cos\left(\frac{\pi a}{L}\right) \Bp^{(t-1)} \BDnT - \frac{\boldsymbol{\lambda}_{\boldsymbol 1 a}^{(t-1)}}{\beta_1} \,, \frac{1}{\beta_1} \Big)
    \,,~~ a = 0, \ldots, L-1
    \\
    &\text{\bfseries 4.}~~ 
    \Bw_a^{(t)} = \Shrink \Big( \Bg_a^{(t-1)} - \frac{\boldsymbol{\lambda}_{\boldsymbol 2 a}^{(t-1)}}{\beta_2} \,, \frac{\mu_1}{\beta_2} \Big)
    \,, a = 0, \ldots, S-1
    \\
    &\text{\bfseries 5.}~~ 
    \Bg_a^{(t)} = \RE \left[ \cF^{-1} \left\{ \frac{\mathcal B_a^{(t)}(\Bz)}{\mathcal A_a^{(t)}(\Bz)} \right\} \right] [\Bk] \Bigg |_{\Bk \in \Omega}
    \,, a = 0, \ldots, S-1
    \\& \qquad
    \mathcal A_a^{(t)}(\Bz) = \beta_2 + \beta_3 \abs{ \sin\left(\frac{\pi a}{S}\right)(z_1 - 1) + \cos\left(\frac{\pi a}{S}\right)(z_2 - 1) }^2 
    \\& \qquad 
    \mathcal B_a^{(t)}(\Bz) = \beta_2 \Big[ W_a^{(t)}(\Bz) + \frac{\Lambda_{2a}^{(t-1)}(\Bz)}{\beta_2} \Big] 
    + \beta_3 \Big[ \sin\left(\frac{\pi a}{S}\right)(z_1^{-1} - 1) + \cos\left(\frac{\pi a}{S}\right) (z_2^{-1} - 1) \Big] \times
    \\& \qquad \qquad \qquad
    \Big[ V^{(t-1)}(\Bz) - \sum_{s = [0, S-1] \backslash \{a\}} \big[ \sin\left(\frac{\pi s}{S}\right)(z_1 - 1) + \cos\left(\frac{\pi s}{S}\right)(z_2 - 1) \big] G_s^{(t-1)}(\Bz) + \frac{\Lambda_3^{(t-1)}(\Bz)}{\beta_3} \Big]
    \\
    &\text{\bfseries 6.}~~ 
    \Bv^{(t)} = \Shrink \Big( \mathcal J^{(t)} \,, \frac{\mu_2}{\beta_3 + \beta_4} \Big)
    \\& \qquad 
    \mathcal J^{(t)} = \frac{\beta_3}{\beta_3 + \beta_4} \Big[ \sum_{s=0}^{S-1} \big[ \sin\left(\frac{\pi s}{S}\right) \BDm \Bg_s^{(t)} + \cos\left(\frac{\pi s}{S}\right)\Bg_s^{(t)} \BDnT \big] - \frac{\boldsymbol{\lambda}_{\boldsymbol 3}^{(t-1)}}{\beta_3} \Big]
    \\& \qquad \qquad
    + \frac{\beta_4}{\beta_3 + \beta_4} \big[ \Bf - c_1^{(t)} - \Beps^{(t-1)} + \frac{\boldsymbol{\lambda}_{\boldsymbol 4}^{(t-1)}}{\beta_4} \big] \cdot^\times \Bp^{(t-1)} 
    + \frac{\beta_4}{\beta_3 + \beta_4} \big[ \Bf - c_2^{(t)} - \Beps^{(t-1)} + \frac{\boldsymbol{\lambda}_{\boldsymbol 4}^{(t-1)}}{\beta_4} \big] \cdot^\times \big[ 1 - \Bp^{(t-1)} \big]
   \end{align*} 
 }
\end{algorithmic}
\end{algorithm}


\begin{algorithm}
\label{alg:2phasePiecewiseConst:2}
\begin{algorithmic}
\small
\STATE
{ {\bfseries (continued Algorithm 5)}
  \begin{align*}
    &\text{\bfseries 7.}~~ 
    p^{(t)}[\Bk] = \begin{cases}
               0 \,, & \tilde p^{(t)}[\Bk] < 0 \\
               1 \,, & \tilde p^{(t)}[\Bk] > 1 \\
               \tilde p^{(t)}[\Bk] \,, & \tilde p^{(t)}[\Bk] \in [0 \,, 1]
              \end{cases}
    \,,~~ \tilde p^{(t)}[\Bk] = \RE \Big[ \cF^{-1} \Big\{ \frac{\mathcal Y^{(t)}(\Bz)}{\mathcal X^{(t)} (\Bz)} \Big\} \Big] [\Bk] \,,~ \Bk \in \Omega 
    \\& \qquad
    \mathcal X^{(t)}(\Bz) = \beta_1 \sum_{l=0}^{L-1} \abs{ \sin\left(\frac{\pi l}{L}\right) (z_1 - 1) + \cos\left(\frac{\pi l}{L}\right) (z_2 - 1) }^2 + \beta_4
    \\& \qquad
    \mathcal Y^{(t)}(\Bz) = \beta_1 \sum_{l=0}^{L-1} 
    \Big[ \sin\left(\frac{\pi l}{L}\right) (z_1^{-1} - 1) + \cos\left(\frac{\pi l}{L}\right) (z_2^{-1} - 1) \Big]
    \Big[ R_l^{(t)}(\Bz) + \frac{\Lambda_{1 l}^{(t-1)}(\Bz)}{\beta_1} \Big] 
    + \beta_4 \mathcal H^{(t)}(\Bz)
    \\& \qquad
    \mathcal H^{(t)}(\Bz) = 
    \cF \left\{ -\frac{1}{2}\left[ \Bf - c_1^{(t)} - \Bv^{(t)} - \Beps^{(t-1)} + \frac{\boldsymbol{\lambda_4}^{(t-1)}}{\beta_4} \right]^{\cdot 2} 
    + \frac{1}{2} \left[ \Bf - c_2^{(t)} - \Bv^{(t)} - \Beps^{(t-1)} + \frac{\boldsymbol{\lambda}_{\boldsymbol 4}^{(t-1)}}{\beta_4} \right]^{\cdot 2}
    \right\} (\Bz)
   \\
   &\text{\bfseries 8.}~~ 
   \Beps^{(t)} = \tilde{\Beps}^{(t)} - \CST \big( \tilde{\Beps}^{(t)} \,, \nu \big) 
   \\& \qquad
   \tilde{\Beps}^{(t)} = \left[ \Bf - c_1^{(t)} - \Bv^{(t)} + \frac{\boldsymbol{\lambda}_{\boldsymbol 4}^{(t-1)}}{\beta_4} \right] \cdot^\times \Bp^{(t)} 
   + \left[ \Bf - c_2^{(t)} - \Bv^{(t)} + \frac{\boldsymbol{\lambda}_{\boldsymbol 4}^{(t-1)}}{\beta_4} \right] \cdot^\times \left[ 1 - \Bp^{(t)} \right]
  \end{align*}
  {\bfseries II. Update}
  $\Big( \big[\boldsymbol{\lambda}_{\mathbf{1}b}^{(t)}\big]_{b=0}^{L-1} \,, \big[\boldsymbol{\lambda}_{\mathbf{2}a}^{(t)}\big]_{a=0}^{S-1}
  \,, \boldsymbol{\lambda}_{\boldsymbol 3}^{(t)} \,, \boldsymbol{\lambda}_{\boldsymbol 4}^{(t)} \Big) \in X^{L+S+2}$:
  \begin{align*}
   \boldsymbol{\lambda}_{\boldsymbol 1 a}^{(t)} &= \boldsymbol{\lambda}_{\boldsymbol 1 a}^{(t-1)}
   + \beta_1 \Big[ \Br_a^{(t)} - \sin\left(\frac{\pi a}{L}\right) \BDm \Bp^{(t)} - \cos\left(\frac{\pi a}{L}\right) \Bp^{(t)} \BDnT \Big] \,,~ a = 0, \ldots, L-1
   \\
   \boldsymbol{\lambda}_{\boldsymbol 2 a}^{(t)} &= \boldsymbol{\lambda}_{\boldsymbol 2 a}^{(t-1)} 
   + \beta_2 \Big[ \Bw_a^{(t)} - \Bg_a^{(t)} \Big] \,,~ a = 0, \ldots, S-1
   \\
   \boldsymbol{\lambda}_{\boldsymbol 3}^{(t)} &= \boldsymbol{\lambda}_{\boldsymbol 3}^{(t-1)} 
   + \beta_3 \left[ \Bv^{(t)} - \sum_{s=0}^{S-1} \left[ \sin\left(\frac{\pi s}{S}\right) \BDm \Bg_s^{(t)} + \cos\left(\frac{\pi s}{S}\right) \Bg_s^{(t)} \BDnT \right] \right]
   \\
   \boldsymbol{\lambda}_{\boldsymbol 4}^{(t)} &= \boldsymbol{\lambda}_{\boldsymbol 4}^{(t-1)} 
   + \beta_4 \Big[ \big( \Bf - c_1^{(t)} - \Bv^{(t)} - \Beps^{(t)} \big) \cdot^\times \Bp^{(t)} 
   + \big( \Bf - c_2^{(t)} - \Bv^{(t)} - \Beps^{(t)} \big) \cdot^\times (1 - \Bp^{(t)}) \Big]
  \end{align*}
 }
\end{algorithmic}
\end{algorithm}


\end{document}